\documentclass{article} 
\usepackage{iclr2027_conference,times}

\usepackage{amsmath,amsfonts,bm}

\def\eqref#1{equation~\ref{#1}}

\def\1{\bm{1}}

\DeclareMathAlphabet{\mathsfit}{\encodingdefault}{\sfdefault}{m}{sl}
\SetMathAlphabet{\mathsfit}{bold}{\encodingdefault}{\sfdefault}{bx}{n}

\usepackage{amsmath}
\usepackage{xcolor}
\newif\ifmarkrevisions
\markrevisionsfalse
\ifmarkrevisions
  \definecolor{revcolor}{rgb}{0.00,0.00,0.80}
\else
  \definecolor{revcolor}{rgb}{0.00,0.00,0.00}
\fi
\newcommand{\revised}[1]{{\color{revcolor}#1}}
\usepackage{hyperref}
\usepackage{url}
\usepackage{subcaption}
\usepackage{graphicx}
\usepackage{enumitem}
\usepackage{amsthm}
\newtheorem{remark}{Remark}
\newtheorem{definition}{Definition}
\newtheorem{theorem}{Theorem}
\newtheorem{lemma}{Lemma}
\newtheorem{assumption}{Assumption}
\newtheorem{proposition}{Proposition}

\newtheorem{thmL}{Theorem}

\newtheorem{propL}[thmL]{Proposition}
\newtheorem*{thmLrep}{Theorem}
\newtheorem{corollary}{Corollary}
\newtheorem{example}{Example}

\title{Structural Correspondence and \\ Universal Approximation in\\ Diagonal plus Low-Rank Neural Networks}

\author{Ying Chen, Aoxi Li, Jihun Kim, Javad Lavaei
\\
Department of Industrial Engineering and Operations Research\\
University of California, Berkeley\\
Berkeley, CA 94720, USA \\
\texttt{\{ying-chen, aoxi, jihun.kim, lavaei\}@berkeley.edu} \\
}

\iclrfinalcopy
\begin{document}

\maketitle

\begin{abstract}
Structured low-rank and sparse-plus-low-rank parameterizations, motivated by methods such as LoRA and low-rank factorization, impose strong structural constraints on the linear transformations within neural networks. We study the expressive consequences of these constrained structures from the perspective of universal approximation. We first characterize that purely low-rank neural networks, despite being able to exactly interpolate arbitrary finite scalar datasets, can fail to approximate continuous functions that vary along directions ignored by the low-rank hidden mapping. We show that augmenting the low-rank structure with a one-parameter diagonal component suffices to overcome this geometric bottleneck, yielding a Diagonal-plus-Low-Rank (DLoR) structure. We show that a full-rank linear transformation can be reconstructed using DLoR components either by expanding network width through additive decomposition or by extending network depth through multiplicative decomposition. We term this correspondence between additive width expansion and multiplicative depth extension the \textit{Structural Correspondence} framework. By combining these constructions with local linearization of the activation function, we prove that DLoR neural networks restore universal approximation for general activation functions. These results characterize how architectural width and depth can compensate for severe rank constraints without requiring unrestricted dense hidden transformations. \revised{Diagonal terms are required for multiplicative construction, while additive expansion works with purely low-rank terms. We further determine its exact cost: both the minimal number of rank-$r$ DLoR factors representing $W$ and the minimal network depth for that layer equal $\lceil\kappa(W)/r\rceil$, where $\kappa(W)=\min_{\gamma\neq0}\operatorname{rank}(W-\gamma I)$, achieving a lower cost than its dense representation when $W$ is identity-centered.}

\end{abstract}
\vspace{-1em}
\section{Introduction}\label{sec:intro}
\vspace{-0.5em}
The era of modern foundation models has been defined by the aggressive scaling of deep neural network architectures to billions, or even trillions of parameters \citep{lecun2015deep, openai2023gpt4}. While highly over-parameterized neural networks are appealing due to their favorable optimization landscapes and global convergence properties \citep{allenzhu2019convergence, baldi2024shallow}, their immense scale imposes prohibitive memory and computational bottlenecks \citep{rajbhandari2020zero, strubell2019energy}. This reality has driven a critical need for parameter and memory-efficient training frameworks \citep{lialin2023scaling}, as well as robust network compression techniques to alleviate the burden of deployment.

Techniques such as Low-Rank Adaptation (LoRA) \citep{hu2022lora, lialin2023relora}, sparse plus low-rank matrix factorization \citep{chandrasekaran2009sparse, ding2023sparse}, and parameter-efficient adapter modules \citep{houlsby2019parameter} have become the preferred standard for efficiency. These low-rank structures are highly favored because they drastically reduce computational overhead while leveraging the phenomenon that pre-trained neural networks typically operate within a low intrinsic dimensionality \citep{aghajanyan2021intrinsic}. By operating on constrained low-rank manifolds, these methods can yield empirical fine-tuning performance that rivals full-rank training \citep{zhang2026sparsityawarelowrankrepresentationefficient}. However, purely low-rank structures restrict parameters to a severely constrained low-dimensional subspace. Specifically, low-rank components are restricted to capturing only the dominant singular value spectrum, entirely ignoring variations outside this subspace. This prompts a question: what exact topological or geometric bottlenecks constrain the expressive power of severely rank-deficient neural networks?
\begin{figure}[t]
\vspace{-1em}
    \centering
    \begin{subfigure}[b]{0.33\textwidth}
        \centering
        \includegraphics[width=0.765\linewidth]{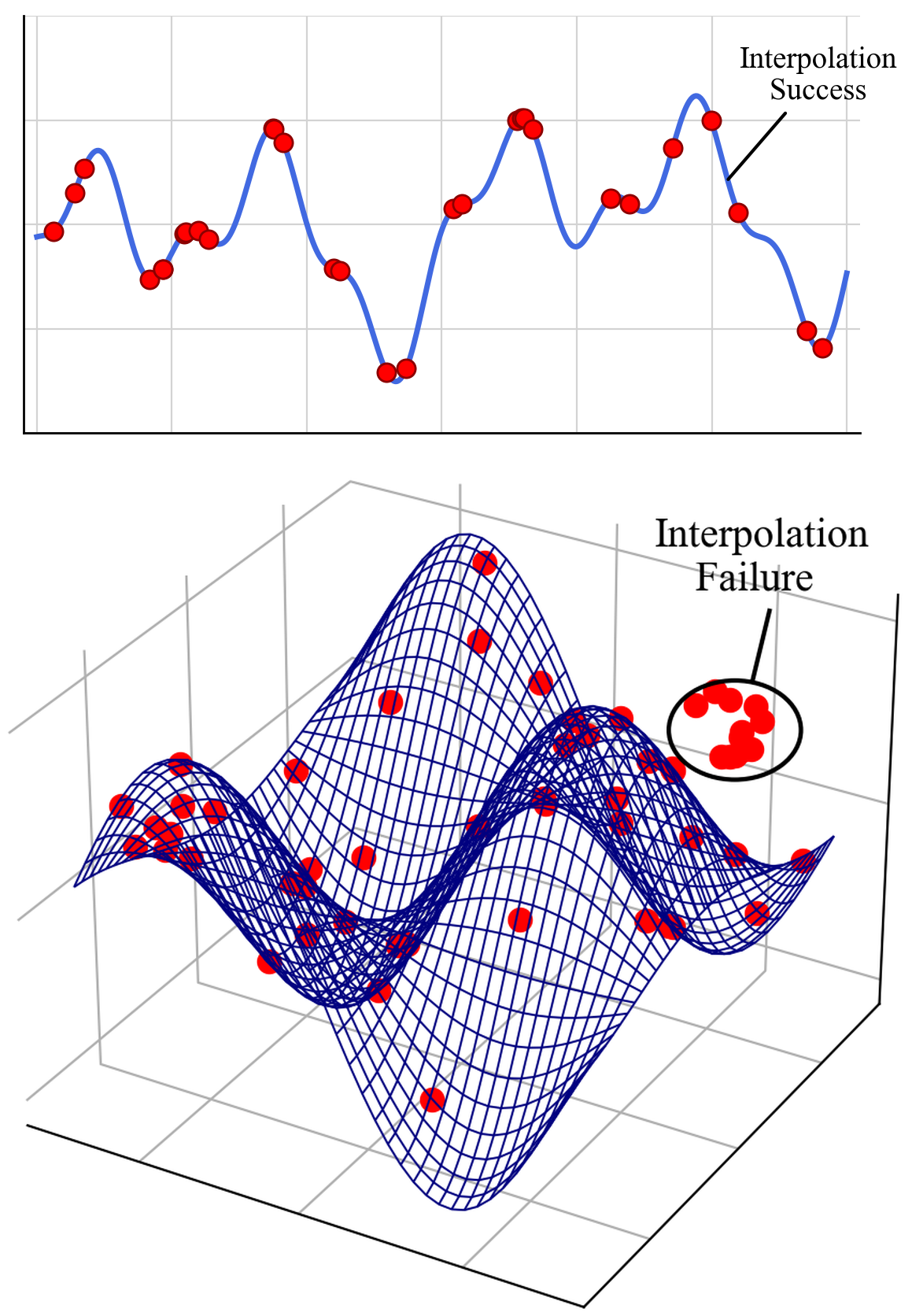}
        \caption{Interpolation by rank-1 NN}
        \label{fig:interpolation}
    \end{subfigure}
    \hfill
    \begin{subfigure}[b]{0.63\textwidth}
        \centering
        \includegraphics[width=0.765\linewidth]{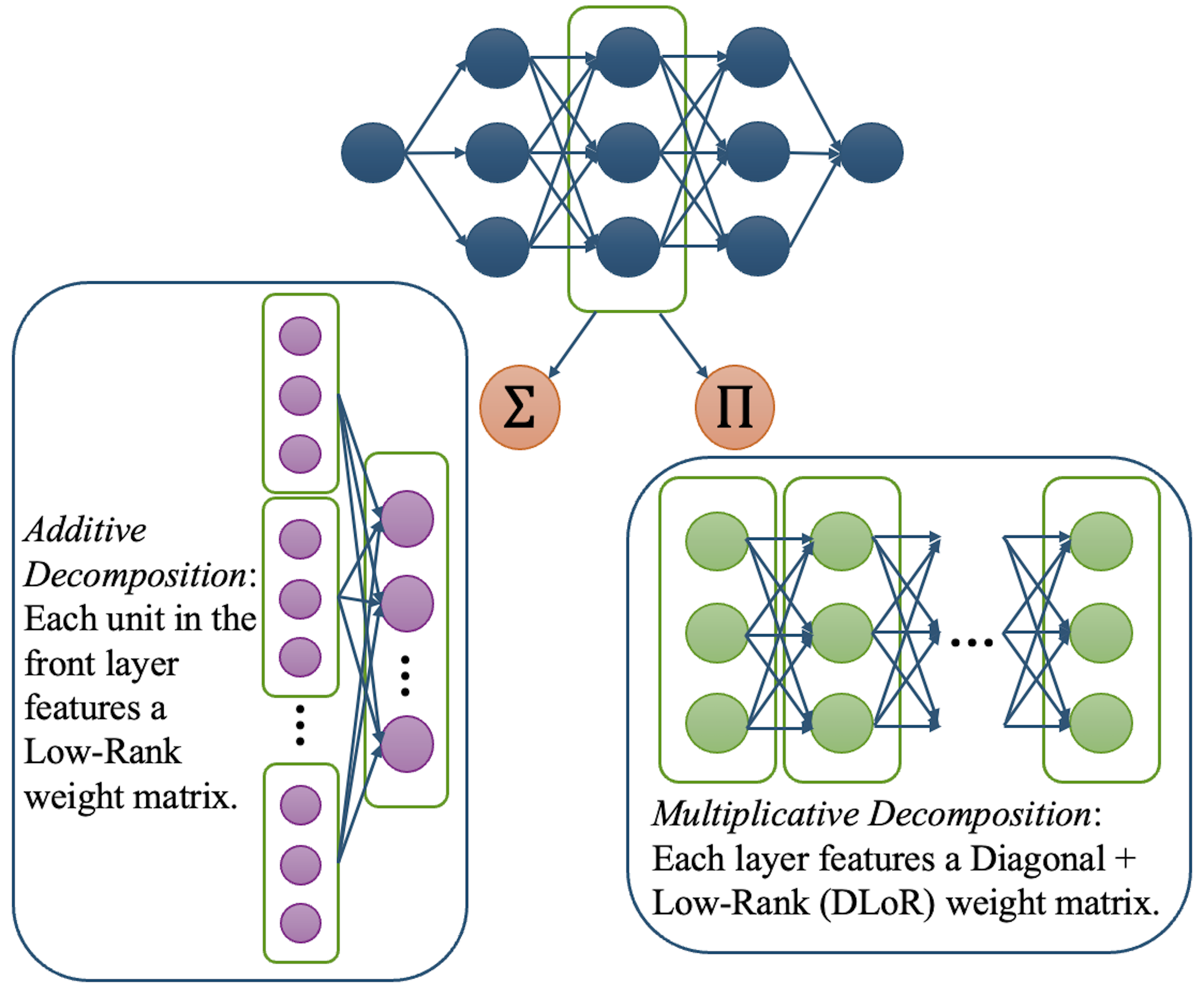}
        \caption{Additive and multiplicative decomposition}
        \label{fig:deep_wide}
    \end{subfigure}\vspace{-0.5em}
    \caption{\small Overview of our work. (a) Rank-1 neural networks (NN) can perfectly interpolate any set of scalar outputs, but inherently fail for multi-dimensional function approximations. (b) We use rank-deficient structures by either expanding width (with additive decomposition) or extending depth (with multiplicative decomposition). While the former can rely on purely low-rank weight matrices, the latter necessitates DLoR weight matrices.}
    \label{fig:mainfigure}
    \vspace{-1.4em}
\end{figure}
Recent literature has attempted to justify the expressive power of parameter-efficient models \citep{zeng2023expressive, giannou2023expressive, li2025approximation}. However, the theoretical guarantees provided by these works remain limited. Current claims rely on the presence of a pretrained dense base matrix, the specific use of the ReLU or Swish activation function, or non-verifiable singularity conditions regarding the weight matrices. In contrast, we ask a more fundamental question: can a strictly constrained neural network achieve the Universal Approximation Theorem (UAT) \citep{hornik1989multilayer, kidger2020universal}, without these additional priors and assumptions? 

We first illustrate a distinction between finite-sample interpolation and continuous function approximation using general activation functions: while a single purely rank-1 hidden layer can exactly interpolate scalar datasets, this finite-sample interpolation capacity does not extend to uniform approximation of general functions on multidimensional input domains. Strict rank-1 mappings suffer from a geometric phenomenon that we term \textit{orthogonal blindness}. Since the network evaluates inputs strictly through a 1-dimensional projection, it becomes completely invariant to variations in the data orthogonal to this projection vector. {This raises a structural question: under severe rank constraints, what architectural mechanisms can compensate for the resulting loss of expressivity?}

Crucially, we establish a \textbf{Structural Correspondence} framework showing that rank deficiency can be compensated for through two distinct architectural mechanisms. We prove that restructuring a weight matrix via an \textit{additive decomposition} naturally translates to expanding the network's width, mapping a high-rank linear operation onto a parallel summation of nonlinear layers. Conversely, a \textit{multiplicative decomposition} corresponds to extending the network's depth. These two mechanisms have different structural requirements: sums of low-rank matrices can recover high-rank matrices, whereas products of rank-deficient matrices remain rank deficient.

{This structural asymmetry motivates augmenting the low-rank factors in the multiplicative construction with a minimal full-rank pathway.} We consider a matrix structure of the form $M = \alpha I + U V^T$, a special instance of a sparse plus low-rank matrix that we term a \textbf{Diagonal plus Low-Rank (DLoR)} component. It introduces a sparse diagonal component to the low-rank structure $\operatorname{rank}(UV^T)\leq r$, where $r$ is any arbitrary rank constraint. Unlike a dense prior, this $\alpha I$ component requires minimally one additional parameter to store, yet it provides a full-rank pathway to capture the orthogonal spectrum, allowing information to bypass the low-dimensional subspace and breaking the geometric bottleneck of orthogonal blindness. {For the multiplicative construction, we further show that the resulting \(O(N/r)\) depth scaling is tight up to constant terms in the worst case.} This structure is also consistent with recent empirical frameworks such as SLTrain \citep{han2024sltrain}, which combine low-rank factors with sparse matrices.
A summary of our theoretical framework, contrasting the interpolation failure of purely rank-1 neural networks with our DLoR architectures, is provided in Figure~\ref{fig:mainfigure}. Specifically, our main contributions are as follows:
\begin{itemize}[leftmargin=0.5cm, itemsep=0pt, topsep=0pt]
    \item \textbf{The Structural Correspondence Framework:} We introduce both additive and multiplicative decomposition methods {using DLoR components.} We prove that rank deficiency can be compensated for by either expanding network width (additive) or extending network depth (multiplicative). For the multiplicative DLoR construction, we further establish a matching worst-case depth lower bound.
    \item \textbf{Universal Approximation via DLoR Structures}: We prove that introducing {only a minimal $\mathcal{O}(1)$ sparse diagonal to form a DLoR structure} \revised{removes} orthogonal blindness. We show that Deep {DLoR} Neural Networks with bounded width restore uniform approximation capabilities for continuous functions \revised{for every activation satisfying Assumption~\ref{ass:main}, that is, every non-affine continuous activation that is continuously differentiable with nonzero derivative at one point.}
    \item \revised{\textbf{The Exact Cost of a DLoR Representation:} We determine what a \emph{given} weight matrix costs: both the minimal number of DLoR factors and the minimal network depth equal $\lceil\kappa(W)/r\rceil$, which delimits the regime where the construction is parameter-efficient (Section~\ref{sec:kappa}).}
\end{itemize}
\vspace{-0.5em}
\section{Related Work}\label{sec:related_work}
\paragraph{Universal Approximation and the Depth-Width Trade-off.}
The Universal Approximation Theorem (UAT) establishes that neural networks are dense in the space of continuous functions on a compact domain, $C(\mathcal{K}, \mathbb{R}^m)$, with respect to the uniform norm. Classic results proved that a single hidden layer utilizing any non-polynomial activation function can achieve uniform $\epsilon$-approximation \citep{hornik1989multilayer, cybenko1989approximation, pinkus1999approximation}. However, these shallow architectures suffer from the curse of dimensionality: to approximate functions in a bounded domain $\mathcal{K} \subset \mathbb{R}^n$, shallow architectures require a prohibitive width scaling of $\mathcal{O}(\epsilon^{-n})$ \citep{yarotsky2017error}. Subsequent theoretical breakthroughs established that depth exponentially increases expressivity and provides fundamentally superior parameter efficiency compared to wide, shallow neural networks \citep{montufar2014number, cohen2016expressive, rolnick2018power}. \citep{telgarsky2016benefits} and \citep{eldan2016power} proved that certain highly oscillatory functions can be expressed with a deep network of $\mathcal{O}(L)$ depth and $\mathcal{O}(1)$ width, but necessitate an intractable $\mathcal{O}(2^L)$ width if restricted to shallow architectures. When approximating the class of highly smooth functions within the Sobolev space $\mathcal{W}^{s, \infty}([0,1]^n)$ (functions with $s$ bounded derivatives), deep ReLU neural networks strictly bound the required parameter complexity to $\mathcal{O}(\epsilon^{-n/s} \log(1/\epsilon))$ with a depth scaling of exactly $\mathcal{O}(\log(1/\epsilon))$ \citep{yarotsky2017error, petersen2018optimal}. In the regime of strictly bounded width, research demonstrated that deep ReLU neural networks of minimal width $n+1$ are sufficient to approximate any scalar-valued function in $C(\mathcal{K}, \mathbb{R})$ \citep{lu2017expressive, hanin2017approximating}. \citep{kidger2020universal} generalized this mapping framework, proving that for \textit{every} non-affine continuous activation function, a strict maximum width of $n+m+2$ is sufficient to approximate any continuous target $f: \mathbb{R}^n \to \mathbb{R}^m$ under a large depth. However, all these foundational proofs inherently allow the network's internal linear transformations to be unrestricted, rather than imposing explicit low-rank structural constraints. Our work bridges the gap between classic UAT and modern parameter-constrained architectures by questioning whether these universal approximation guarantees on $C(\mathcal{K}, \mathbb{R}^m)$ hold when the hypothesis space is restricted to a low-rank and sparse manifold. 

\paragraph{Expressivity of Parameter-Efficient Architectures.}
Beyond full-rank dense neural networks, a parallel line of approximation theory investigates structurally constrained architectures. \citep{bolcskei2019optimal} established a foundational framework demonstrating that sparsely connected deep neural networks can achieve information-theoretical optimal approximation rates, matching the fundamental limits of nonlinear approximation for a broad class of functions. Subsequent works \citep{elbrachter2021deep, daubechies2022nonlinear} further proved that strict sparsity enables deep neural networks to overcome the curse of dimensionality without parameter explosion. In the empirical regime of modern foundation models, architectural efficiency is typically enforced via Parameter-Efficient Fine-Tuning (PEFT) methods, such as Low-Rank Adaptation (LoRA) \citep{hu2022lora}. However, much of the approximation theory for parameter-constrained networks has focused on sparsely connected architectures \citep{bolcskei2019optimal}, whereas modern parameter-efficient methods often impose structured low-rank constraints on otherwise dense linear transformations. This leaves a gap in understanding the expressive consequences of structured rank constraints. Recent attempts to formalize PEFT expressivity assume the presence of a frozen, full-rank base matrix, specific non-singularity constraints, and the ReLU activation function \citep{zeng2023expressive, giannou2023expressive}. 
While \citep{li2025approximation} establishes universal approximation bounds for 50\% low-rank architectures, their guarantees rely on the Swish activation function and mandate relatively wide bottlenecks. In contrast, we study the function class induced directly by severely rank-constrained weight matrices without assuming a pretrained dense base and only mild assumption on activation functions. Our Structural Correspondence framework shows how dense transformations can be recovered through either additive width or multiplicative depth extension when augmented with a minimal sparse diagonal.
\paragraph{Sparse Plus Low-Rank Decomposition.}
The decomposition of complex matrices into low-rank and sparse components has deep mathematical roots in robust statistics and convex optimization, where the rank-sparsity incoherence principle proves that these two structures complement each other \citep{chandrasekaran2009sparse, yuan2009sparse, zhang2018unified, bertsimas2023sparse}. Recently, this principle has been adapted empirically for deep learning to compress large language models \citep{li2023losparse, ding2023sparse}, optimize attention mechanisms \citep{chen2021scatterbrain}, and facilitate full-rank-equivalent pretraining by capturing the lower-magnitude spectrum of the data \citep{han2024sltrain}. Simultaneously, empirical studies on PEFT have discovered that adding identity residual mappings to LoRA drastically improves optimization stability and fine-tuning performance \citep{shi2024reslora, xia2024chain}. Our DLoR parameterization is a highly constrained instance of this sparse-plus-low-rank family, in which the sparse component is reduced to a single scalar multiple of the identity.
Rather than analyzing its optimization behavior, we study the resulting function class and show that this one-parameter diagonal augmentation is sufficient for bounded-width multiplicative constructions to recover dense transformations and universal approximation.
\paragraph{\revised{Structured Factorizations and Products of Elementary Matrices.}}
\revised{Proving universal approximation by factoring each dense weight into structured factors and linearizing the activation between them is a general template, used by \citet{liu2024lu} for triangular, Toeplitz and Hankel factors; they note that it extends to any family closed under scalar multiplication, which includes DLoR, so our Theorem~\ref{thm:uat_multiplicative} follows from that remark combined with our factorization. What is specific to DLoR is the exact cost (Section~\ref{sec:kappa}), whose matrix-level form is classical \citep{dieudonne1955generateurs, huang2017absolute, erdos1967products, dawlings1980semigroups}. Appendix~\ref{app:prior_work} gives the full comparison.}
\vspace{-0.5em}
\section{Theoretical Limits of Pure Low-Rank Architectures}\label{sec:limits}\vspace{-0.5em}

Before analyzing our proposed sparse-augmented architecture DLoR, we characterize the expressive boundaries of strictly low-rank neural networks. We distinguish finite-sample
interpolation from two geometric limitations induced by low-rank constraints: output-space affine collapse for multidimensional targets, and input-space orthogonal blindness on multidimensional domains. Proofs can be found in Appendix~\ref{sec:proof3}.

{We first introduce the setup for the finite-sample interpolation. Given $M$ distinct inputs $\{x_i\}_{i=1}^M\subset\mathbb{R}^d$ and targets
$\{z_i\}_{i=1}^M\subset\mathbb{R}^k$, consider a low-rank parametrization of a standard feedforward neural network with one hidden layer as:} $W_1=U_1V_1^T$, $W_2=U_2V_2^T$, where $U_1\in\mathbb{R}^{N\times r_1}$, $V_1\in\mathbb{R}^{d\times r_1}$,
$U_2\in\mathbb{R}^{k\times r_2}$, and $V_2\in\mathbb{R}^{N\times r_2}$.
The network maps each individual input as:
\begin{equation} \label{eq:rank1_mapping}
    f(x)=U_2V_2^T\sigma(U_1V_1^Tx+b_1)+b_2.\vspace{-0.5em}
\end{equation}


We first show that rank-one matrices already suffice for finite scalar interpolation, and then characterize the geometric limitations that persist under general low-rank neural networks.


\vspace{-0.5em}
\subsection{Finite-Sample Interpolation Under Rank Constraints}\label{sec:scalar_output}
\vspace{-0.5em}
We begin with the most restrictive nontrivial setting, where both weight matrices are rank one. When the targets $z_i$ are scalars ($k=1$), the rank-one outer mapping $U_2 V_2^T$
reduces to a row vector. We show that even this highly constrained architecture can exactly interpolate an arbitrary finite dataset.


\begin{lemma}[Scalar Interpolation]\label{thm:gen_act_interp}
Given distinct inputs $\{x_i\}_{i=1}^M\subset\mathbb{R}^d$ and arbitrary scalar targets
$\{z_i\}_{i=1}^M\subset\mathbb{R}$, let $\sigma$ be any continuous, non-mean-periodic function (common activations like ReLU, Sigmoid, and SiLU are not mean-periodic). For network width $N=M$, there exist rank-one factors
$U_1\in\mathbb{R}^{N\times 1}$, $V_1\in\mathbb{R}^{d\times 1}$,
$U_2\in\mathbb{R}^{1\times 1}$, $V_2\in\mathbb{R}^{N\times 1}$,
and $b_1\in\mathbb{R}^N$ such that $U_2V_2^T\sigma(U_1V_1^Tx_i+b_1)=z_i,\ i=1,...,M$.

\end{lemma}

\vspace{-0.5em}
\begin{remark}\label{rem:constantinescu}
Our exact scalar interpolation result for general activation functions matches \citep{constantinescu2023approximation}.  
Instead of using full-rank hidden weight matrix $W_1$, we demonstrate that a strictly rank-1 weight matrix is sufficient to perfectly interpolate arbitrary scalar targets.
\revised{With free hidden slopes, any continuous non-polynomial $\sigma$ suffices (Lemma~\ref{lem:rank1_interp_nonpoly}).}
\end{remark}
However, when the output layer is also rank-constrained, this interpolation capability does not extend to mappings in $\mathbb{R}^k$ with $k \ge 2$. The low-rank output layer can restrict the output to an affine subspace of a smaller dimension than the geometry of the multidimensional targets as formalized below.


\begin{lemma}[Affine Collapse] \label{thm:impossibility}
If $\operatorname{rank}(W_2)\le r_2$, then $f(x)\in b_2+\operatorname{Im}(W_2)$ for every $x$, which is an affine subspace of dimension at most $r_2$. Consequently, any target set whose affine span has dimension greater than $r_2$
cannot be exactly interpolated.
\end{lemma}

This obstruction is caused by the output rank constraint. If $W_2$ is unrestricted, a rank-one hidden layer can again exactly interpolate arbitrary finite multidimensional targets (proof in Appendix~\ref{app:zero_error_outer_layer}).

\vspace{-0.5em}
\subsection{Geometric Limitations of Low-Rank Hidden Mappings}\label{sec:impossibility}

\vspace{-0.5em}
 Exact finite interpolation is a strictly weaker condition than universal approximation. Even with a full-rank outer layer, a low-rank hidden transformation remains invariant along a nontrivial input subspace, preventing approximation of general target functions that vary along these directions. We term this geometric limitation \textit{Orthogonal Blindness}.
 
\begin{definition}[Orthogonal Blindness]\label{def:ortho_blindness}
Let $\mathcal{N}(V_1^T) = \{x \in \mathbb{R}^d \mid V_1^T x = 0\}$ denote the at least $(d-r_1)$-dimensional null space orthogonal to $\operatorname{Col}(V_1)$. For any input $x$, we apply the orthogonal decomposition $x = x_\parallel + x_\perp$, where $x_\perp \in \mathcal{N}(V_1^T)$. A purely Rank-$r$ mapping $f(x) = g(V_1^T x)$ evaluates to: $f(x_\parallel + x_\perp) = g\big(V_1^T(x_\parallel + x_\perp)\big) = g(V_1^T x_\parallel) = f(x_\parallel)$.
\end{definition}

\vspace{-0.5em}
Consequently, if the target function varies along directions in the orthogonal space of $\operatorname{Col}(V_1)$, the strictly constrained network fails to achieve function approximation. Figure~\ref{fig:rank1_blindness} illustrates the rank-one case. These limitations motivate the architectural mechanisms introduced in Section~\ref{sec:duality}.
\vspace{-1em}
\section{The Structural Correspondence Framework}\label{sec:duality}
\vspace{-0.5em}
To overcome the affine collapse and orthogonal blindness intrinsic to purely low-rank mappings formalized in Section~\ref{sec:limits}, we propose augmenting the network with a minimal sparse diagonal component. Proofs can be found in Appendix~\ref{sec:proof4}.

\begin{definition}[Diagonal plus Low-Rank (DLoR) Component]\label{def:sparse_low_rank}
A linear transformation $M \in \mathbb{R}^{N \times N}$ is a DLoR component of rank $r$ if it is constrained to the form $M = \alpha I_N + U V^T$ where $\alpha \in \mathbb{R}$ is a scalar forming a sparse diagonal, and $U, V \in \mathbb{R}^{N \times r}$ form a low-rank matrix ($r \ll N$).
\end{definition}
\vspace{-0.5em}
This $\alpha I_N$ component requires only $\mathcal{O}(1)$ parameters to store, yet it serves as a vital residual pathway. When $\alpha \neq 0$,  the DLoR component can retain full rank, whereas purely rank-1 architectures necessarily have a nontrivial orthogonal null space. A similar effect has been observed for ResNet~\citep{lin2018resnet}. In this section, we establish a structural correspondence demonstrating how any dense, full-rank matrix $W \in \mathbb{R}^{N \times N}$ can be perfectly reconstructed using solely these Diagonal plus Low-Rank components by rigorously trading off either network width (via additive decomposition) or network depth (via multiplicative decomposition). 
\vspace{-0.5em}
\subsection{Additive Decomposition: Width Expansion}\label{sec:additive}
\vspace{-0.5em}
A fundamental approach to reconstructing a high-rank matrix is through linear addition. Consider an arbitrary (high-rank) weight matrix $W$ and decompose it into $L$ independent  low-rank components via singular value decomposition (SVD) for some value $L$, i.e., $W = \sum_{l=1}^L M_l$, where each matrix $M_l$ is a DLoR component with $\alpha=0$ (although any $\alpha\in \mathbb{R}$ is also acceptable). Inside a neural network, these components cannot simply be summed before activation; they must be mapped through the non-linear layer. We demonstrate that projecting this linear addition through the activation function structurally translates to expanding the network's width into a parallel ensemble. 

\begin{assumption}
$\rho:\mathbb R\to\mathbb R$ is continuous and non-affine. There exists an open interval containing $c$ on which $\rho$ is continuously differentiable, with $\rho'(c)\ne0$.
\label{ass:main}
\end{assumption}
\begin{theorem}[Additive Width Expansion]\label{thm:additive_width}
Let $W=\sum_{l=1}^L M_l\in\mathbb R^{N\times N}$, with $L\ge2$. Under Assumption~\ref{ass:main}, choose nonzero scalars $\beta_l$ such that $\sum_{l=1}^L\beta_l=0$. Then, for every compact input domain $\mathcal K\subset\mathbb R^N$, the parallel wide structure computes the exact linear mapping $Wx$ as $h \to 0$:
$\lim_{h\to0}\sup_{x\in\mathcal K}
\left\|
\sum_{l=1}^L\frac{\beta_l}{h\rho'(c)}
\rho\!\left(\frac{h}{\beta_l}M_lx+c\mathbf1_N\right)-Wx
\right\|_2=0$.
\end{theorem}
\vspace{-0.5em}
A decomposition with one component can be extended by a zero component.

\textbf{Width Architecture and Elimination of Affine Drift:} The limit expression in Theorem~\ref{thm:additive_width} precisely describes a width expansion. The parallel summation can be compactly written as $B\rho(W_{wide}x + b_{wide})$, where the outer linear projection is $B= [\frac{\beta_1}{h \rho'(c)} I_N, \dots, \frac{\beta_L}{h \rho'(c)}I_N]$, and the inner parameters $W_{wide}$ and $b_{wide}$ are formed by vertically stacking the blocks $\frac{hM_l}{\beta_l}$ and $c\mathbf{1}_N$, respectively. This increases the layer width by a factor of $L$. The zero-sum condition $\sum_{l=1}^L\beta_l=0$ cancels the constant zeroth-order terms introduced by the local Taylor expansions. Consequently, when the construction is used to approximate a dense layer $y=\rho(Wx+b)$, the original target bias $b$ can be retained without an additional correction induced by these constant offsets. Here $B\in\mathbb R^{N\times LN}$ is a structured summation operator with scalar identity blocks. The wide architecture consists of parallel DLoR branches and $B$ is not itself a square DLoR component.
\vspace{-0.5em}
\subsection{Multiplicative Decomposition: Depth Extension}\label{sec:multiplicative}
\vspace{-0.5em}
While the additive decomposition effectively recovers expressivity, expanding network width is often computationally undesirable for memory-constrained inference. Alternatively, to bound the network width to its original dimension $N$, we establish that a high-rank weight matrix can be factorized multiplicatively. We prove that any full-rank transformation can be decomposed into sequential DLoR layers. To achieve universal expressivity using highly constrained DLoR components, our theoretical framework relies on the property that continuous, non-linear activation functions can smoothly approximate linear mappings. We adapt a lemma from \citep{kidger2020universal} to establish this identity mapping.

\begin{lemma}[\citep{kidger2020universal}, Lemma 4.1]\label{lem:kidger_4_1}
Under Assumption~\ref{ass:main}, define the scalar affine functions: $\Phi_h(x) = hx + c, \Psi_h(y) = \frac{y - \rho(c)}{h \rho'(c)}$. Then, as the scaling parameter $h \to 0$, the composition $(\Psi_h \circ \rho \circ \Phi_h)(x)$ converges uniformly to the exact identity function $x$ on any compact domain.
\end{lemma}
\vspace{-0.5em}
\begin{theorem}[Multiplicative Depth Extension]\label{thm:multiplicative_depth}
For any invertible matrix $W \in \mathbb{R}^{N \times N}$ and any rank constraint $1 \leq r \leq N$, let $L = \lceil N/r \rceil$. Let $\rho$, $\Phi_h$ and $\Psi_h$ be as in Lemma~\ref{lem:kidger_4_1}. There exists a non-zero scalar $\alpha \in \mathbb{R}$ and a sequence of low-rank matrices $U_l, V_l \in \mathbb{R}^{N \times r_l}$ (where $r_l \le r$) defining DLoR components $M_l = \alpha I_N + U_l V_l^T$ for $l = 1, \dots, L$, such that $W=M_L\cdots M_1$. Moreover, for any compact input domain $\mathcal{K} \subset \mathbb{R}^N$ and target bias $b \in \mathbb{R}^N$, the deep sequential structure approximates the exact dense layer and the following convergence holds uniformly for $x\in\mathcal K$:
\[
\lim_{h \to 0}\rho \left( M_L \circ \Psi_h \circ \rho \circ \Phi_h \circ M_{L-1} \circ \dots \circ \Psi_h \circ \rho \circ \Phi_h \circ M_1 (x) + b \right) = \rho(W x + b).\vspace{-0.5em}
\]
\end{theorem}\vspace{-0.2em}

\begin{remark}[Role of the $h\to 0$ limit]
\label{rem:h_goes_to_0}
The limit $h\to 0$ is a standard analytical device for the approximation result rather than an assumption on practical training. Theorems~\ref{thm:additive_width} and~\ref{thm:multiplicative_depth} construct a family of DLoR networks indexed by $h$ whose realized functions converge uniformly on compact domains to the target dense layer. Thus, the results establish representational existence and density, but do not imply that gradient-based optimization will recover these constructions for arbitrarily small $h$. Finite-$h$ behavior of the constructed networks is illustrated empirically in Section~\ref{subsec: validate_uat}. Under an additional local $C^2$ condition, Appendix~\ref{app:finite_scale} provides explicit finite-$h$ error bounds for both constructions.
\end{remark}
\begin{remark}
    The $O(N/r)$ depth scaling in Theorem~\ref{thm:multiplicative_depth} is tight in the worst case: there exist linear maps for which any exact representation using the multiplicative depth extension with DLoR components requires $ L \ge \left\lceil \frac{N-1}{r} \right\rceil$. See Appendix~\ref{app:lower_bound_exact}.
\end{remark}

This theorem formally proves that depth can substitute for parameter rank. A single dense, high-rank linear transformation can be exactly decomposed into a deep sequence of $L$ non-linear layers, strictly preserving a constant width $N$ at every step. Moreover, the sequential composition of the DLoR components $M_l$ with the affine mappings $\Phi_h$ and $\Psi_h$ strictly preserves the DLoR architectural constraints. Because the mappings $\Phi_h$ and $\Psi_h$ consist entirely of uniform scalar multiplications and constant vector additions, applying them to the linear operator simply scales the weight matrix and shifts the layer bias. We also provide an alternative construction utilizing an augmented $2N$-dimensional state space in Appendix~\ref{sec:augmented_space} for clear intuition.
\begin{proposition}[Exact ReLU decomposition]
\label{prop:relu_exact}
Let $\mathcal K\subset\mathbb R^N$ be compact, $W\in\mathbb R^{N\times N}$, $b\in\mathbb R^N$, and $1\le r\le N$. The map $x\mapsto \operatorname{ReLU}(Wx+b)$ has an exact representation on $\mathcal K$ by a ReLU DLoR block of state width $N$, update rank at most $r$, and at most $\lceil 3N/r\rceil$ affine stages. If $W$ is invertible, $\lceil N/r\rceil$ stages suffice.
\end{proposition}

ReLU is exactly affine on its positive half-line. By adding finite bias shifts, every intermediate state can be kept in this region on the prescribed compact set, while the following bias correction removes the accumulated shift. Thus this construction preserves the state width and requires no limiting operation. The proof is given in Appendix~\ref{app:relu_exact}. 
\begin{remark}[Rectangular Transformations]
\label{rem:rectangular}
The square formulation extends to rectangular matrices through orthogonal completion. Let $W\in\mathbb{R}^{m\times n}$ have maximal rank and $W=U\Sigma V^T$ be its full singular value decomposition (SVD). If $m<n$, appending the remaining right singular vectors as rows produces an invertible $n\times n$ matrix whose first $m$ output coordinates equal $Wx$. If $m>n$, similarly, we append the remaining left singular vectors as columns. Hence, rectangular layers can be handled by the square construction in Theorem~\ref{thm:multiplicative_depth}. Rank-deficient matrices can be treated by an arbitrarily small maximal-rank perturbation.

\end{remark}


\begin{remark}[Hybrid Structural Combinations]\label{rem:hybrid_decomp}
While Theorems~\ref{thm:additive_width} and~\ref{thm:multiplicative_depth} establish the pure limits of width and depth respectively, the Structural Correspondence framework naturally permits hybrid topologies. Because matrix addition and matrix multiplication are mathematically composable, any general high-rank weight matrix $W$ can be decomposed into a complex combination of both summations and products of Sparse Low-Rank components. For instance, $W$ could be approximated by a parallel ensemble of several moderately deep sub-networks, taking the form $W = \sum_i (\prod_j M_{i,j})$. This allows machine learning practitioners to dynamically navigate the depth-width trade-off space, tailoring the architectural reconstruction to perfectly match the specific hardware constraints (e.g., parallel GPU thread limits versus sequential VRAM constraints) of their deployment environment.
\end{remark}
The limit $h\to0$ in Theorems~\ref{thm:additive_width} and~\ref{thm:multiplicative_depth} is an analytical device for general activations: it places each activation in a locally linear regime and yields explicit existence constructions. ReLU admits a stronger finite-scale statement. Proposition~\ref{prop:relu_exact} therefore gives an exact same-width ReLU realization with no limiting operation and zero simulation error. For general activations, an additional local $C^2$ assumption yields explicit finite-$h$ error bounds as in Proposition~\ref{prop:finite_scale}. These results concern representational existence and understanding the optimization behavior of directly trained DLoR networks is left for future work.


\vspace{-1em}
\section{Universal Approximation for DLoR}\label{sec:uat}
\vspace{-0.5em}
Having established the structural correspondence of DLoR components, we now prove that embedding these decompositions into a deep or wide network architecture completely overcomes orthogonal blindness and fully restores the Universal Approximation Theorem (UAT). We use $\rho$ in hidden layers and an affine final readout. In the deep construction, trainable square affine matrices are DLoR, with fixed input padding and output coordinate selection as in Remark~\ref{rem:rectangular}. In the wide construction, parallel DLoR branches are combined by the structured summation operator described in Section~\ref{sec:additive}.

It is critical to note the broad applicability of the activation condition in Theorems~\ref{thm:additive_width} and \ref{thm:multiplicative_depth}. It only requires that there exist \textit{some} expansion point $c \in \mathbb{R}$ where $\rho'(c) \neq 0$. This inherently covers smooth non-linearities like Sigmoid, Tanh, Swish, and GELU. ReLU also satisfies the assumption by choosing
$c>0$, where $\rho$ is locally affine with derivative $1$. In addition, ReLU admits the finite exact construction in Proposition~\ref{prop:relu_exact}.

Since our DLoR blocks can uniformly simulate any dense affine transformation $y = \rho(W x + b)$ over a compact domain, and continuous functions map compact sets to compact sets, the intermediate representations at any hidden layer remain bounded, allowing us to enforce uniform bounds on the asymptotic Taylor remainders of our approximations. Proofs can be found in Appendix~\ref{sec:proof5}.
\vspace{-0.5em}
\subsection{UAT via Multiplicative Depth and Additive Width}\label{sec:uat_multiplicative_Additive}
\vspace{-0.5em}
We first prove that universal approximation for DLoR can be achieved by extending depth and preserving the width bound $N = n+m+2$. Conversely, if network depth is restricted, universal approximation can be achieved by expanding the network width via our additive structural correspondence, mapping the high-rank transformation into a parallel ensemble. Both deep and wide DLoR architectures theoretically achieve universal approximation, aligning with classical approximation theory on the exponential benefits of depth and \revised{linear benefits of width} in expressivity.

\begin{theorem}[Deep DLoR UAT]\label{thm:uat_multiplicative}
Let $\mathcal K\subset\mathbb R^n$ be compact. Under Assumption~\ref{ass:main}, deep DLoR networks of arbitrary depth, state width $N=n+m+2$, and update rank at most $r$, for any fixed $1\le r\le N$, are dense in $C(\mathcal K,\mathbb R^m)$.
\end{theorem}
\vspace{-0.5em}
\begin{theorem}[Wide DLoR UAT]\label{thm:uat_additive}
Let $\mathcal K\subset\mathbb R^n$ be compact. Under Assumption~\ref{ass:main}, finite networks formed by replacing dense layers with parallel DLoR branches and the structured aggregation of Theorem~\ref{thm:additive_width} are dense in $C(\mathcal K,\mathbb R^m)$.
\end{theorem}
\vspace{-0.5em}
Beyond the standard Universal Approximation Theorem for continuous functions, \revised{a further consequence} of our Structural Correspondence framework is that \textit{any} existing universal approximation guarantee for a dense network can be systematically transferred to our constrained DLoR network.

\begin{theorem}[General UAT Transfer]\label{thm:general_transfer}
Under Assumption~\ref{ass:main}, let $F:\mathbb R^n\to\mathbb R^m$ be a finite dense network with $L$ affine stages, hidden activation $\rho$, and an affine final readout. Let $W$ be at least the maximum of its input, output, and hidden dimensions. For every compact $\mathcal K\subset\mathbb R^n$, every $\eta>0$, and every $1\le r\le W$, there exists a DLoR network $G$ of state width $W$ and at most $L\lceil W/r\rceil$ affine stages such that $\|G-F\|_{\infty,\mathcal K}<\eta$. Its update ranks are at most $r$, and its boundary interfaces are fixed padding and coordinate selection. Consequently, if $\|F-f\|_{\infty,\mathcal K}\le\epsilon$, then $\|G-f\|_{\infty,\mathcal K}<\epsilon+\eta$.
\end{theorem}
\vspace{-0.5em}
\subsection{Parameter Complexity for Sobolev Spaces}\label{sec:uat_complexity}
\vspace{-0.5em}


We count all allocated real scalar entries in $U,V$, the diagonal coefficient, and the bias, including entries set to zero by a construction. A width-$N$ affine DLoR layer with factor rank $r_j$ has $2Nr_j+N+1$ such slots. For $q=\lceil N/r\rceil$ and $\sum_jr_j=N$, one simulated square layer therefore uses at most $\sum_{j=1}^q(2Nr_j+N+1)=2N^2+q(N+1)=O(N^2)$.

For fixed width, this is a constant-factor overhead relative to a fully allocated dense matrix. We combine this accounting with the fixed-width approximation theorem of \citet[Theorem~1]{siegel2023optimal}. It gives uniform error bounded by $C_{d,s}L^{-2s/d}$ on the unit Sobolev ball, with ReLU hidden width $25d+31$. The width is independent of the requested accuracy, so DLoR conversion has constant overhead for fixed $d$. The following theorem uses the standard unrestricted-real-parameter model for ReLU networks. Proof and parameter accounting are in Appendix~\ref{sec:proof_sobolev_repaired}.

\begin{theorem}[DLoR Sobolev parameter bound]\label{thm:dlor_sobolev}
Fix integers $d,s\ge1$ and let $\mathcal B_{d,s}=\{f\in W^{s,\infty}([0,1]^d):\|f\|_{W^{s,\infty}}\le1\}$,
using the continuous representatives of these functions. There exists $C_{d,s}$ such that, for every $f\in\mathcal B_{d,s}$ and $0<\epsilon<1/2$, a ReLU DLoR network $G$ of width $N=25d+31$, with every trainable square affine matrix of the form $\alpha I_N+uv^T$, satisfies \[
\|G-f\|_{\infty,[0,1]^d}<\epsilon,
\qquad
P(G),D(G)\le C_{d,s}\epsilon^{-d/(2s)}.
\] counts affine stages, $P(G)$ counts all scalar slots including biases, and input padding and output coordinate projection are fixed as in Remark~\ref{rem:rectangular}.
\end{theorem}


\begingroup\color{revcolor}
\vspace{-1em}
\section{The Exact Cost of DLoR Representations}\label{sec:kappa}
\vspace{-0.5em}
Theorem~\ref{thm:multiplicative_depth} bounds the depth of the multiplicative construction by $\lceil N/r\rceil$ in the worst case, but says nothing about what a \emph{given} $W$ costs, so it is silent on parameter efficiency. That cost is governed by $\kappa(W):=\min_{\gamma\neq0}\operatorname{rank}(W-\gamma I_N)=N-g(W)$, the distance of $W$ from a nonzero multiple of the identity, where $g(W)$ is the largest geometric multiplicity of a nonzero real eigenvalue of $W$ and $0$ if there is none; $\mathcal P_r(L)$ denotes products of $L$ DLoR components of update rank at most $r$. New results are lettered, so existing numbering is unchanged.
\begin{thmL}[Exact factor count]\label{thm:A}
Let $W\in\mathbb R^{N\times N}$ and $1\le r<N$. If $\operatorname{rank}(W)\le r$ then $W\in\mathcal P_r(1)$; otherwise
\[
\min\{L\ge1:\ W\in\mathcal P_r(L)\}=\max\bigl\{1,\lceil\kappa(W)/r\rceil\bigr\},
\]
and the minimum is attained by explicit factors $M_1=\gamma(I_N+R_1)$, $M_j=I_N+R_j$. Invertibility of $W$ is not assumed.
\end{thmL}
\vspace{-0.3em}
The same count is exact at the network level (Theorem~\ref{thm:B}): $\lceil\kappa(W)/r\rceil$ DLoR stages of width $N$ realize $x\mapsto\rho(Wx+b)$ exactly for ReLU and for paired branches and approximate it to any accuracy under Assumption~\ref{ass:main}, while $L$ stages with $Lr<\kappa(W)$ leave a strictly positive error floor; a rectangular layer needs only $\lceil m/r\rceil$ stages, which cuts the deep MNIST replica of Section~\ref{subsec: validate_uat} from $37$ stages to $8$ (Corollary~\ref{cor:exact_transfer_cost}). Appendix~\ref{app:new_theory} proves these and adds four further results: Theorems~\ref{thm:C} and~\ref{thm:D} and Propositions~\ref{prop:E}--\ref{prop:G}.

\vspace{-0.3em}
\paragraph{When is DLoR parameter-efficient?} By Theorems~\ref{thm:A} and~\ref{thm:B} one simulated layer costs about $2N\kappa(W)$ parameters against $N^2+N$ for the dense layer it replaces, so DLoR saves parameters roughly when $\kappa(W)\le N/2$ ($N/3$ at $r=1$): the regime of identity-centered matrices such as residual maps and multiplicative adaptations of pretrained weights, and not that of generic dense weights, which Theorem~\ref{thm:D} excludes. We therefore do not claim that Section~\ref{sec:duality} compresses an arbitrary dense layer; the content of Theorems~\ref{thm:A}--\ref{thm:D} is that the boundary is sharp and given by $\kappa$ (measurements in Appendix~\ref{app:new_experiments}).
\endgroup

\vspace{-1em}
\section{Numerical Experiments}
\vspace{-0.5em}
\subsection{Validation of UAT}
\label{subsec: validate_uat}
\vspace{-0.5em}
We first provide empirical evidence to show that the deep and wide structures proposed in Theorems \ref{thm:uat_multiplicative} and \ref{thm:uat_additive} exhibit the universal approximation power as promised. Following the proof ideas of these theorems, we first train a fully connected multilayer perceptron (MLP) that approximates a given function. Then for each layer of the MLP, we construct the substructure of the deep sparse low-rank and the wide sparse low-rank network respectively and verify their approximation error to the dense MLP. We train the dense MLP in two settings: i) approximate a sawtooth function using synthetic data; ii) image classification on the MNIST dataset. The detailed experimental setup can be found in Appendix \ref{subsec: validate_uat_app}.

\begin{figure}[h]\vspace{-1em}
    \centering
    \begin{subfigure}[b]{0.31\textwidth}
        \centering
        \includegraphics[width=0.9\linewidth]{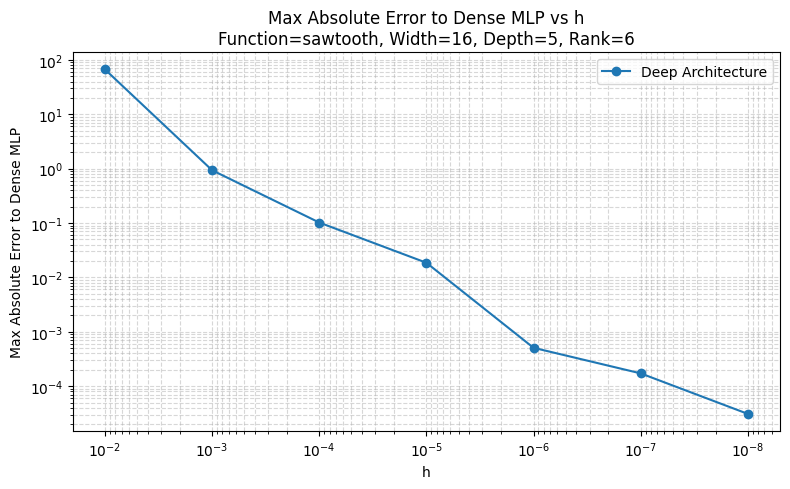}
        \caption{}
        \label{fig:deep_to_dense}
    \end{subfigure}
    \hfill
    \begin{subfigure}[b]{0.31\textwidth}
        \centering
        \includegraphics[width=0.9\linewidth]{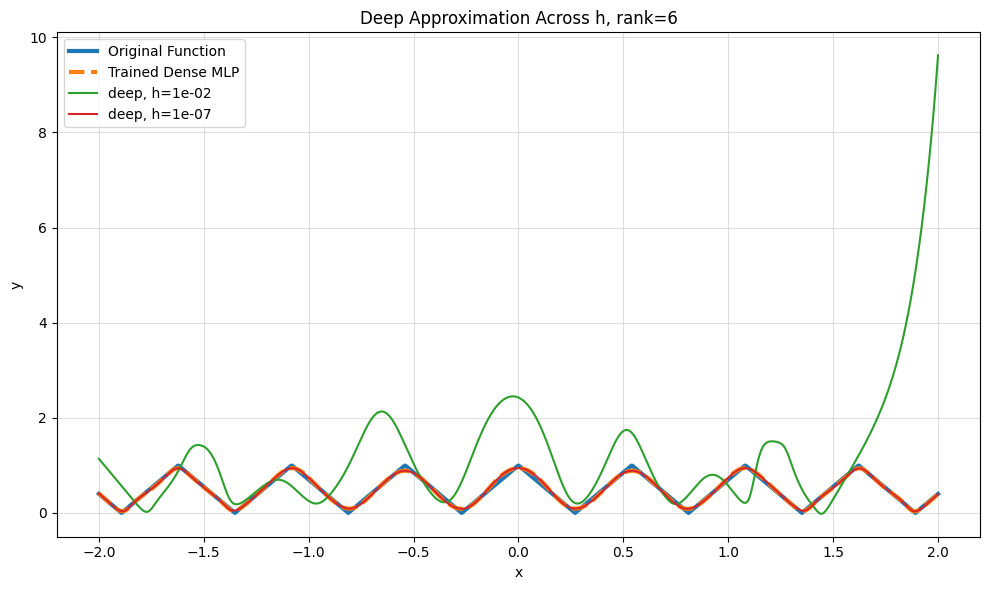}
        \caption{}
        \label{fig:deep_to_function}
    \end{subfigure}
    \hfill
    \begin{subfigure}[b]{0.31\textwidth}
        \centering
        \includegraphics[width=0.9\linewidth]{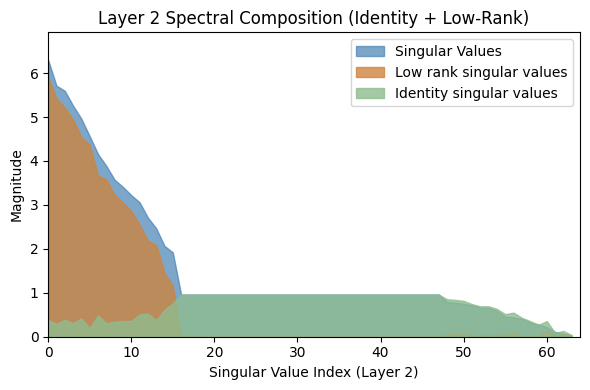}
        \caption{}
        \label{fig:deep_spectral}
    \end{subfigure}
    \hfill
    \vspace{-1em}
    \caption{\small Deep network approximation and decomposition: (\ref{fig:deep_to_dense}) approximation error to dense MLP; (\ref{fig:deep_to_function}) approximation to sawtooth function as $h \rightarrow 0$; (\ref{fig:deep_spectral}) deep network spectral decomposition}
    \label{fig:deep_construction_spectral}\vspace{-1em}
\end{figure}

\begin{figure}[h]
    \centering
    \begin{subfigure}[b]{0.31\textwidth}
        \centering
        \includegraphics[width=0.9\linewidth]{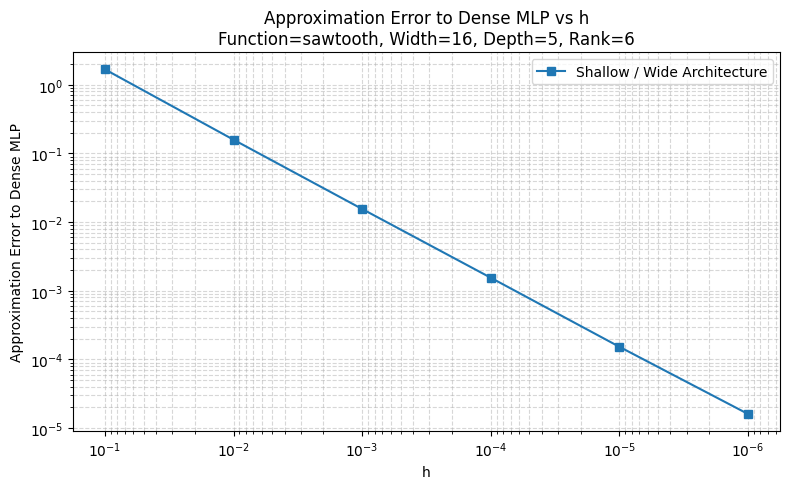}
        \caption{}
        \label{fig:wide_to_dense}
    \end{subfigure}
    \hfill
    \begin{subfigure}[b]{0.31\textwidth}
        \centering
        \includegraphics[width=0.9\linewidth]{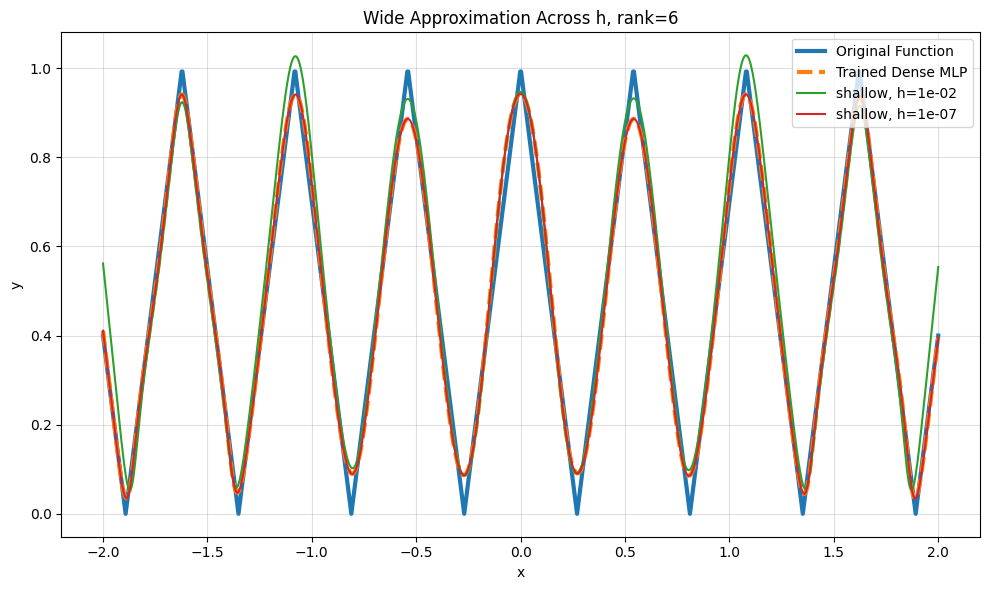}
        \caption{}
        \label{fig:wide_to_function}
    \end{subfigure}
    \hfill
    \begin{subfigure}[b]{0.31\textwidth}
        \centering
        \includegraphics[width=0.9\linewidth]{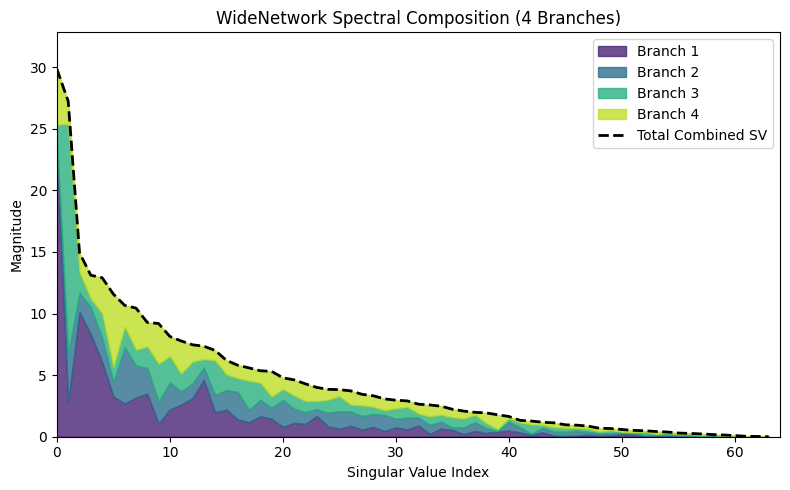}
        \caption{}
        \label{fig:wide_spectrl}
    \end{subfigure}
    \vspace{-1em}
    \caption{\small Wide network approximation and decomposition: (\ref{fig:wide_to_dense}) approximation error to dense MLP; (\ref{fig:wide_to_function}) approximation to sawtooth function as $h \rightarrow 0$; (\ref{fig:wide_spectrl}) wide network spectral decomposition.} 
    \label{fig:wide_construction_spectral}\vspace{-1em}
\end{figure}

 \paragraph{Sawtooth Results} The sawtooth experiment is intended to validate the constructive approximation results in Theorems \ref{thm:additive_width} and \ref{thm:multiplicative_depth}. Its non-smooth and oscillatory structure provides a nontrivial continuous target for demonstrating that the proposed DLoR architectures realize the approximation behavior predicted by the theory. We use the max-absolute-error to evaluate the approximation error. Figures \ref{fig:deep_to_dense} and \ref{fig:wide_to_dense} show that for both the deep network and the wide network, the approximation error to the dense MLP goes to $0$ as $h \rightarrow 0$, validating the first part of our proofs for Theorems \ref{thm:uat_multiplicative} and $\ref{thm:uat_additive}$ that these structures can universally approximate dense MLP. Figures \ref{fig:deep_to_function} and \ref{fig:wide_to_function} show that once the trained MLP universally approximates the given test function, the constructed structures will subsequently approximate the test function universally, validating the final part of the proofs.
 \vspace{-0.5em}
 \paragraph{MNIST Results} To further validate that the DLoR structures can approximate functions of higher dimensions and do not exhibit orthogonal blindness as purely low-rank structures, we construct deep and wide DLoR structures to approximate a MLP that is a classifier on the MNIST dataset. We apply the deep and wide DLoR model constructed from the baseline dense MLP to the same training and testing sets that the dense MLP is trained on and report the averaged loss and classification accuracy in table \ref{tab:mnist_train_test}. The exact construction process and hyperparameters are in Appendix \ref{subsec: validate_uat_app}. We observe that both constructed DLoR models have the same level of testing loss and testing accuracy as the baseline dense MLP, showing that the constructed models exhibit the same generalization performance on the test data as the dense MLP. The link between approximation and generalization can be explained by our Theorem \ref{thm:multiplicative_depth} and \ref{thm:additive_width}: since the constructed models approximate the given MLP, and the given MLP generalizes, the constructed DLoR models will also generalize. To further validate that the constructed DLoR models are approximating the function represented by the baseline MLP, we also report max-logit-difference, max-relative-logit-difference, and relative-logit-difference as metrics for evaluating the approximation error in table \ref{tab:mnist_approx_err}. The exact definitions of these metrics can be found in Appendix \ref{sec:app_uat_mnist_metric}. \revised{The logit deviations in Table~\ref{tab:mnist_approx_err} are the primary measure of approximation quality: a uniform logit error $\eta$ preserves every prediction of margin above $2\eta$ (Lemma~\ref{lem:logit_pert}), so the small differences in Table~\ref{tab:mnist_train_test} are consistent with them and are not evidence of better generalisation.}
\begin{table}[ht]
\centering
\begin{minipage}[t]{0.55\textwidth}
    \centering
    \small
    \begin{tabular}{c|cccc}
         & Train Loss & Train Acc & Test Loss & Test Acc \\
         \hline
        Dense MLP & 0.0172 & 0.994 & 0.0794 & 0.9805 \\
        Deep DLoR & 0.0128 & 0.996 & 0.0799 & 0.9806 \\
        Wide DLoR & 0.0125 & 0.996 & 0.0795 & 0.9805
    \end{tabular}
    \vspace{-1em}
    \caption{Loss and Accuracy}
    \label{tab:mnist_train_test}
\end{minipage}
\hfill
\begin{minipage}[t]{0.39\textwidth}
    \centering
    \small
    \begin{tabular}{c|ccc}
        & $E_\mathrm{max}$ 
        & $E^{\text{rel}}_\mathrm{max}$ 
        & $E_{\text{rel}}$ \\
        \hline
        Deep DLoR & 1.26 & 0.0224 & 0.009 \\
        Wide DLoR & 0.2918 & 0.0052 & 0.003\\
    \end{tabular}
    \vspace{-1em}
    \caption{MNIST Approximation Error}
    \label{tab:mnist_approx_err}
\end{minipage}\vspace{-2em}
\end{table}

\paragraph{Experiment 2: Spectral Decomposition} To better understand the difference between how the representational workload is divided in the substructures of the deep and wide network, we also train the DLoR deep and wide structures and visualize the singular-value decompositions of trained models inspired by \cite{han2024sltrain}. The training experiment here serves only as an illustrative demonstration of the potential mechanism underlying universal approximation of DLoR networks. A systematic study of training DLoR architectures from scratch is left to future work. We first directly train a deep DLoR structure and a wide DLoR structure from scratch to fit the synthetic data generated by the sawtooth function. After the training, we project the left and right singular vectors of the weight matrix onto each substructure. For the wide structure, the combined weight matrix is used for projection. The detailed setup can be found in Appendix \ref{subsec:spectral_decom_setup}. For the deep structure, the Figure \ref{fig:deep_spectral} shows that, after training, the learned transformations do not collapse to either purely low-rank approximations or purely identity mappings. Instead, the low-rank component contributes a small number of dominant singular directions, while the identity component preserves a broad background spectrum. The singular-value distribution of the DLoR layer resembles the sparse-plus-low-rank parametrization studied in~\cite{han2024sltrain}, while in our DLoR structure, the sparse component is restricted to only a diagonal scaled-identity term. For the wide structure, Figure~\ref{fig:wide_spectrl} shows a different pattern. Instead of components contributing separately, the combined singular-value spectrum is formed by the accumulation of many overlapping low-rank branch contributions. Each branch contributes a small amount across a broad range of singular directions. This suggests that the wide network distributes the representation across parallel branches in a less spectrally separated manner than the deep network.
\vspace{-1em}
\section{Conclusion}\label{sec:conclusion}
\vspace{-0.9em}
In this paper, we investigated the universal approximation limits of rank-deficient neural networks. While a purely rank-1 network exhibits exact interpolation capability for scalar targets, it suffers from \textit{Orthogonal Blindness} when extended to multi-dimensional continuous mappings. To resolve this without sacrificing parameter efficiency, we introduced a unified geometric theory based on sparse diagonal augmentation  ($M = \alpha I + U V^T$). By establishing a \textit{Structural Correspondence} framework of additive width expansion and multiplicative depth extension, we proved that Deep DLoR Neural Networks fully restore the Universal Approximation Theorem for general activation functions. \revised{We then determined its exact cost: both the minimal number of rank-$r$ DLoR factors and  minimal network depth equal $\lceil\kappa(W)/r\rceil$, falling below the dense cost when $W$ is identity-centered.} These findings not only uncover the limits and power of low-rank architectures but also provide intuition for the next generation of parameter-efficient pretraining and fine-tuning frameworks.

\subsubsection*{Acknowledgement}
This work was supported by the U. S. Army Research Laboratory and the U. S. Army Research Office under Grant W911NF2010219, Office of Naval Research under Grant N000142412673, and NSF.





\bibliography{iclr2027_conference}
\bibliographystyle{iclr2027_conference}
\newpage
\appendix
\section{Background}
\paragraph{Implicit Regularization vs. Explicit Rank Constraints.}
A parallel line of literature has studied the implicit regularization of gradient descent. Foundational work has demonstrated that deep, over-parameterized neural networks generalize remarkably well despite having the capacity to memorize random labels \citep{zhang2016understanding, brutzkus2017sgd}. Theoretical analysis of the optimization landscape reveals that extreme over-parameterization fundamentally smooths the loss surface, eliminating spurious local minima and ensuring global convergence \citep{choromanska2015loss, nguyen2017loss, yun2017global}. Within these over-parameterized spaces, gradient descent naturally biases deep linear neural networks towards low-rank or simple representations, effectively minimizing matrix rank rather than standard norms \citep{gunasekar2017implicit, arora2019implicit, razin2020implicit}. However, there is a fundamental topological distinction between an over-parameterized network that implicitly learns a low-rank solution \textit{after} exploring a high-dimensional landscape, and an architecture that is explicitly constrained to a strict low-rank manifold from initialization. Our paper formalizes why explicit constraints fail: while implicit regularization operates within a full-rank space capable of spanning all dimensions, explicit rank constraints \revised{remove the network's capacity to process orthogonal variations, which is the obstruction formalized in Definition~\ref{def:ortho_blindness}.}

\section{Proofs for Section~\ref{sec:limits}}
\label{sec:proof3}
\paragraph{Notation.} Throughout this paper, we establish the following mathematical conventions. Scalars are denoted by lowercase letters (e.g., $\alpha, c, h$). Vectors are denoted by lowercase letters (e.g., $x, u, v$), and matrices by uppercase letters (e.g., $W, U, V$). The $i$-th element of a vector $v$ is denoted by $v_i$, and the $(i,j)$-th entry of a matrix $W$ is denoted by $W_{i,j}$. The identity matrix of dimension $N \times N$ is denoted by $I_N$, and the $M$-dimensional column vector of all ones is denoted by $\mathbf{1}_M$. The superscript $T$ denotes the transpose of a matrix or vector. We use $\|\cdot\|$ to denote the standard Euclidean norm for vectors and the induced spectral norm for matrices. The Lebesgue measure on $\mathbb{R}^d$ is implicitly assumed when discussing probability distributions. Function composition is denoted by $\circ$, and the convolution of a function $f$ with a measure $\mu$ is denoted by $f * \mu$. We use $Z_{a:b}$ to denote the submatrix of $Z$ formed by columns $a$ through $b$. For asymptotic limits, $f(h) = o(h)$ implies $\lim_{h \to 0} \|f(h)\|/h = 0$ uniformly, and $f(h) = \mathcal{O}(h)$ implies $\|f(h)\| \le C h$ for some constant $C > 0$. For a compact set $K \subset \mathbb{R}^d$, we denote by $C(K, \mathbb{R}^m)$ the space of all continuous functions $f : K \to \mathbb{R}^m$, equipped with the sup-norm $\|f\|_{\infty} := \sup_{x \in K} \|f(x)\|_2$. For products representing forward propagation, we use the ordered convention $\prod_{j=1}^k M_j:=M_k\cdots M_1$.
\begin{figure}[h]
    \centering
   
    \includegraphics[width=0.9\linewidth]{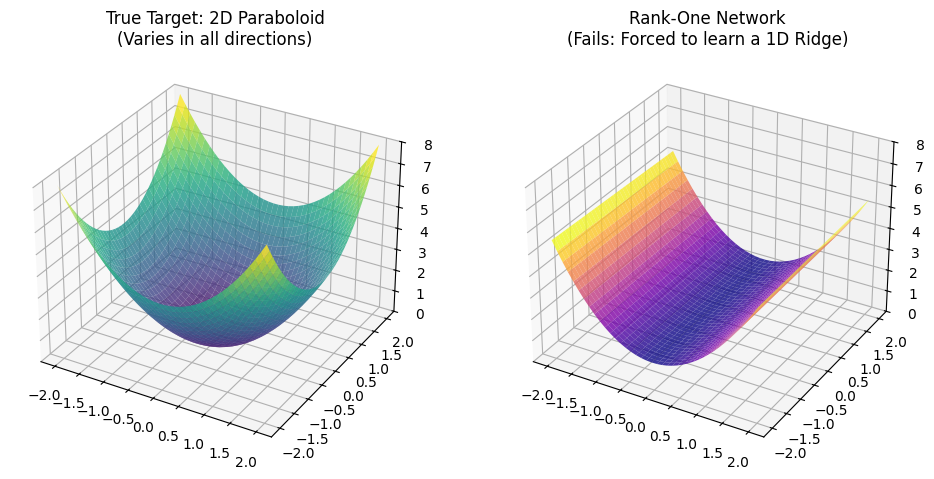}
    \caption{Illustration of approximation failure for rank-one neural networks: the target function $(x,y) \mapsto x^2 + y^2$ varies independently along two directions, while a rank-one network is restricted to a one-dimensional ridge structure. Consequently, it captures variation in only one direction and fails to approximate the full two-dimensional geometry of the function, a limitation extends to higher dimensional function approximation.}
    \label{fig:rank1_blindness}
\end{figure}

\subsection{Example: Heaviside step function}
To build intuition, we first provide a lemma regarding random projections.
\begin{lemma}[Injective Random Projection]\label{lem:random_proj}
Assume $v_1 \in \mathbb{R}^d$ is sampled from a probability distribution absolutely continuous with respect to the Lebesgue measure. Then, with probability 1, the scalar projections $y_i = v_1^T x_i$ are strictly distinct for all $i \in \{1, \dots, M\}$.
\end{lemma}

Because the projections $y_i$ are distinct reals, we define a permutation index $s(\cdot)$ that strictly sorts them: $y_{s(1)} < y_{s(2)} < \dots < y_{s(M)}$. We apply the identical permutation to the targets to obtain $z_{s(i)}$. 

If we temporarily define $\sigma(x)$ as the Heaviside step function where $\sigma(x) = 1$ if $x \geq 0$ and $0$ otherwise, we can construct an explicit thermometer encoding.
\begin{theorem}[Thermometer Encoding]\label{thm:thermometer}
Let $N=M$ and $u_1 = \mathbf{1}_M$. For the Heaviside activation, there exist a threshold bias vector $b_1 \in \mathbb{R}^M$ and a readout vector $v_2 \in \mathbb{R}^M$ such that the purely rank-1 network achieves zero interpolation error.
\end{theorem}

\subsubsection{Proof of Lemma~\ref{lem:random_proj}}
\begin{proof}
Consider the set of all pairwise difference vectors between the input samples, defined as
\begin{equation}
    \Delta_{ij} = x_i - x_j
\end{equation}
for all $i \neq j$. If the scalar projections of two distinct points are equal, such that $y_i = y_j$, this strictly implies that their difference is orthogonal to the projection vector, yielding
\begin{equation}
    v_1^T \Delta_{ij} = 0
\end{equation}
For any fixed, non-zero difference vector $\Delta_{ij}$, this orthogonality condition defines a $(d-1)$-dimensional hyperplane in $\mathbb{R}^d$ passing through the origin. Because our dataset contains a finite number of samples $M$, there are exactly $\binom{M}{2}$ such hyperplanes. The finite union of these hyperplanes has a Lebesgue measure of strictly zero in $\mathbb{R}^d$. Consequently, because $v_1$ is sampled from a probability distribution that is absolutely continuous with respect to the Lebesgue measure, the probability of drawing a vector $v_1$ that falls perfectly inside this zero-measure union is exactly zero. Therefore, with probability 1, a randomly chosen $v_1$ satisfies $v_1^T \Delta_{ij} \neq 0$ for all pairs, guaranteeing strictly distinct projections.
\end{proof}

\subsubsection{Proof of Theorem~\ref{thm:thermometer}}
\begin{proof}
We construct the threshold bias vector $b_1 = [-t_1, -t_2, \dots, -t_M]^T$ by strategically placing the thresholds $t_j$ between the sorted projection values using the midpoint rule. For the first threshold, we define a small margin $\delta > 0$ and set
\begin{equation}
    t_1 = y_{s(1)} - \delta
\end{equation}
For all subsequent thresholds where $j > 1$, we place them exactly halfway between adjacent sorted projections:
\begin{equation}
    t_j = \frac{1}{2}\left(y_{s(j-1)} + y_{s(j)}\right)
\end{equation}
The $j$-th component of the hidden layer for a given input $x_{s(i)}$ then evaluates to $h_{s(i), j} = \sigma(y_{s(i)} - t_j)$. Because the projection sequence is strictly sorted, if the index satisfies $j \leq i$, it follows that $y_{s(i)} \ge y_{s(j)}$. By the definition of our midpoints, we obtain that $y_{s(j)} > t_j$, which guarantees that the argument to the activation function is strictly positive, yielding $\sigma(y_{s(i)} - t_j) = 1$. 

Conversely, if the index satisfies $j > i$, this implies $j-1 \ge i$, meaning $y_{s(j-1)} \ge y_{s(i)}$. Our midpoint definition dictates that $t_j > y_{s(j-1)}$, resulting in a strictly negative argument and yielding $\sigma(y_{s(i)} - t_j) = 0$. Consequently, the hidden layer representation structurally collapses into a perfect thermometer encoding:
\begin{equation}
    h_{s(i)} = (\underbrace{1, \dots, 1}_{i}, \underbrace{0, \dots, 0}_{M-i})^T
\end{equation}
To finalize the exact interpolation, we construct the output weights $v_2 = [w_1, \dots, w_M]^T$ using backward differences. We anchor the first weight to the first target by setting $w_1 = z_{s(1)}$, and compute the remaining weights as
\begin{equation}
    w_j = z_{s(j)} - z_{s(j-1)}
\end{equation}
for all $j \ge 2$. Summing these weights up to index $i$ perfectly telescopes to the desired target $z_{s(i)}$, mathematically satisfying $v_2^T h_{s(i)} = z_{s(i)}$ and driving the interpolation error to zero.
\end{proof}

\subsection{Proof of Lemma~\ref{thm:gen_act_interp}}

A continuous function $\sigma$ is mean-periodic if $\sigma*\mu=0$ for some nonzero compactly supported signed Borel measure $\mu$. Non-mean-periodicity ensures the linear independence of the distinct translates used below.
\begin{proof}
Fix an input projection vector $v_1$ that yields strictly distinct scalar projections $y_i = v_1^T x_i$. Let the hidden weights be $u_1 = \mathbf{1}_M$ and the biases be $b_1 = [b_{11}, \dots, b_{1M}]^T$. We can explicitly define the $M \times M$ evaluation matrix $H$ by its entries:
\begin{equation}
    H_{i,j} = \sigma(y_j + b_{1i})
\end{equation}
By defining the data offsets as $\delta_j = y_j - y_1$ and the evaluation points as $t_i = y_1 + b_{1i}$, we can rewrite the matrix entries as $H_{i,j} = \sigma(t_i + \delta_j)$. This naturally motivates the definition of translated activation functions $g_j(t) = \sigma(t + \delta_j)$, allowing us to cleanly express the matrix entries as $H_{i,j} = g_j(t_i)$.

We first define the evaluation matrix $H(m)$ for some given scalar functions $g_1, \dots, g_m$, and evaluation points $t_1, \dots, t_m$ by its entries: $H(m)_{i,j} = g_j(t_i)$. Now we establish that if these functions $g_1, \dots, g_m$ are linearly independent on the real line, we can recursively select the evaluation points $t_i$ to ensure that the determinant of $H(m)$ is strictly non-zero by induction on $m$. Then by taking $m=M$ and setting $H=H(M)$, we prove the existence of the invertible evaluation matrix $H$. \textbf{Base case}: when $m=1$, for a single dimension, linear independence ensures that $g_1$ is not the zero function, allowing us to select a point $t_1$ where $g_1(t_1) \neq 0$. For the induction hypothesis, assuming this property holds for $m-1$ functions. At the inductive step, given linearly independent functions $g_1, \dots, g_{m}$, by the induction hypothesis, there exists $t_1, \dots, t_{m-1}$ such that $H(m-1)$ is strictly invertible. To extend $H(m-1)$ to $H(m)$, we evaluate the determinant of $H(m)$ by expansion along the final row as a function of an arbitrary evaluation point $x$:
\begin{equation}
    f(x) = \sum_{i=1}^m C_i g_i(x)
\end{equation}
The last coefficient of this expansion is exactly the determinant of $H(m-1)$, $C_m = \det(H(m-1))$, which is non-zero by our inductive hypothesis. Because the functions are linearly independent, this linear combination is strictly non-zero, ensuring there exists a valid evaluation point $t_m = x$ that renders the full matrix $H(m)$ invertible.

To formalize the required linear independence of the translated functions, we assume by contradiction that they are linearly dependent. This implies the existence of real constants $c_j$, not all zero, such that their linear combination vanishes for all real $t$:
\begin{equation}
    \sum_{j=1}^M c_j \sigma(t + \delta_j) = 0
\end{equation}
This dependency perfectly describes the convolution of the activation function $\sigma$ with a discrete measure $\mu$. We explicitly construct this measure as a linear combination of Dirac measures centered at the negative offsets, defined as
\begin{equation}
    \mu = \sum_{j=1}^M c_j \delta_{-\delta_j}
\end{equation}
Because the original projections $y_j$ are strictly distinct, the corresponding offsets $\delta_j$ are also distinct. Combined with the non-zero constants $c_j$, this makes $\mu$ a strictly non-zero, compactly supported measure. This convolutional relationship forces $\sigma$ to be a mean-periodic function, which directly contradicts our initial premise. Thus, the translated functions are linearly independent, the matrix $H$ is strictly invertible, and we can compute the exact linear readout weights as $v_2^T = Z H^{-1}$.
\end{proof}

\subsection{Proof of Lemma~\ref{thm:impossibility}}
\begin{proof}
For every hidden state $h(x)$, the output satisfies
$W_2h(x)+b_2\in b_2+\operatorname{Im}(W_2)$. This affine subspace has dimension $\operatorname{rank}(W_2)\le r_2$. Hence an interpolated target set must have affine span of dimension at most $r_2$. For $r_2=1$, all outputs lie on an affine line, so three noncollinear targets cannot be interpolated.
\end{proof}
\subsection{Proof of Zero Interpolation Error for Full Outer Layer}
\label{app:zero_error_outer_layer}

In this subsection, we provide the formal theorem and proof supporting the claim that a network with a full-capacity outer layer can achieve exactly zero interpolation error on multidimensional targets, even when the hidden layer is restricted to a strictly rank-1 structure.
\begin{theorem}[Zero Error Memorization for Multi-Dimensional Targets]
\label{thm:zero_error_outer_multi}
Let $X \in \mathbb{R}^{d \times M}$ be a dataset matrix comprising $M$ distinct input column vectors, and let $Z \in \mathbb{R}^{k \times M}$ be the corresponding arbitrary $k$-dimensional target matrix. Consider the two-layer neural network parameterized in matrix form as:
\begin{equation}
    f(X) = U_2 V_2^T \sigma(u_1 v_1^T X + b_1 \mathbf{1}_M^T) + b_2 \mathbf{1}_M^T
\end{equation}
where $u_1 \in \mathbb{R}^M$ and $v_1 \in \mathbb{R}^d$ form the rank-1 hidden weight matrix $W_1 = u_1 v_1^T$, $b_1 \in \mathbb{R}^M$ is the hidden bias, $U_2 V_2^T \in \mathbb{R}^{k \times M}$ represents the factorized outer layer weights, and $b_2 \in \mathbb{R}^k$ is the outer bias. $\mathbf{1}_M \in \mathbb{R}^M$ is the all-ones vector.

If the outer layer $U_2 V_2^T$ is fully parameterized (i.e., capable of expressing any $k \times M$ matrix) and $\sigma$ is continuous and non-mean-periodic, there exist parameters $\{u_1, v_1, b_1, U_2, V_2, b_2\}$ such that the network achieves exactly zero interpolation error:
\begin{equation}
    U_2 V_2^T \sigma(u_1 v_1^T X + b_1 \mathbf{1}_M^T) + b_2 \mathbf{1}_M^T = Z
\end{equation}
\end{theorem}

\begin{proof}
Choose $v_1$ so that $t_i=v_1^Tx_i$ are pairwise distinct, and set $u_1=\mathbf1_M$. The proof of Lemma~\ref{thm:gen_act_interp} selects biases $b_{1,j}$, conditioned on these fixed projections, for which
\[
H_{j,i}=\sigma(t_i+b_{1,j})
\]
is nonsingular. The same hidden matrix $H$ can be used for all $k$ output coordinates.

The network output is
\[
f(X) = U_2 V_2^\top H + b_2 \mathbf{1}_M^\top.
\]
Let $W_2 := U_2 V_2^\top$. We require
\[
W_2 H + b_2 \mathbf{1}_M^\top = Z.
\]
Since $H$ is invertible, for any $b_2$ we can set
\[
W_2 = (Z - b_2 \mathbf{1}_M^\top) H^{-1}.
\]

By assumption, the factorization $U_2 V_2^\top$ can represent any matrix in $\mathbb{R}^{k\times M}$ (equivalently, its factorization rank is at least $\min(k,M)$). Thus there exist $U_2, V_2$ such that $U_2 V_2^\top = W_2$.

With the above construction, we obtain $f(X) = Z$, completing the proof.
\end{proof}
\section{Proofs for Section~\ref{sec:duality}}
\label{sec:proof4}
\subsection{Proof of Lemma~\ref{lem:kidger_4_1}}
\paragraph{Remark.}
Lemma~\ref{lem:kidger_4_1} is a result in~\citep{kidger2020universal}. We include their constructive proof for completeness and to improve readability, using arguments adapted to our setting.
\begin{proof}
By the Taylor expansion of the activation function $\rho$ around the expansion point $c$, we have:
\begin{equation}\label{eq:taylor_rho_identity}
    \rho(hx + c) = \rho(c) + h \rho'(c) x + \mathcal{R}(hx)
\end{equation}
where the Taylor remainder strictly satisfies $\|\mathcal{R}(hx)\| = o(h)$ uniformly for all $x$ within any compact set. Substituting the expansion into the outer affine function $\Psi_h$, we obtain:
\begin{equation}\label{eq:psi_expansion}
    \Psi_h(\rho(\Phi_h(x))) = \frac{\rho(c) + h \rho'(c) x + o(h) - \rho(c)}{h \rho'(c)} 
\end{equation}
The constant $\rho(c)$ terms in the numerator perfectly cancel. Dividing the remaining terms by the denominator $h \rho'(c)$ simplifies the expression to:
\begin{equation}
    x + \frac{o(h)}{h \rho'(c)} = x + o(1)
\end{equation}
Because $\rho'(c)$ is a non-zero constant, as the scaling parameter $h \to 0$, the $o(1)$ residual term uniformly vanishes. Consequently, the composition uniformly converges to exactly $x$.
\end{proof}
\subsection{Proof of Theorem~\ref{thm:additive_width}}

\begin{proof}
By the Taylor expansion of the activation function $\rho$ around the expansion point $c$, for any bounded input $x \in \mathcal{K}$, we have:
\begin{equation}\label{eq:taylor_rho}
    \rho\left(\frac{h}{\beta_l} M_l x + c \mathbf{1}_N\right) = \rho(c)\mathbf{1}_N + \frac{h}{\beta_l} \rho'(c) M_l x + \mathcal{E}_l(h)
\end{equation}
Because the derivative $\rho'$ is continuous and the domain $\mathcal{K}$ is compact, the remainder term is uniformly bounded, satisfying $\|\mathcal{E}_l(h)\| = o(h)$. We substitute this expansion into the left side of our additive ensemble equation:
\begin{equation}
    \sum_{l=1}^L \frac{\beta_l}{h \rho'(c)} \left[ \rho(c)\mathbf{1}_N + \frac{h}{\beta_l} \rho'(c) M_l x + \mathcal{E}_l(h) \right] 
\end{equation}
Distributing the summation yields three distinct terms:
\begin{equation}\label{eq:additive_expansion}
    \frac{\rho(c)}{h \rho'(c)}\mathbf{1}_N \left( \sum_{l=1}^L \beta_l \right) + \sum_{l=1}^L M_l x + \sum_{l=1}^L \frac{\beta_l}{h \rho'(c)} \mathcal{E}_l(h)
\end{equation}
Crucially, because we enforce the structural constraint $\sum_{l=1}^L \beta_l = 0$, the zero-order bias term perfectly annihilates. Since we defined the exact additive decomposition as $W = \sum_{l=1}^L M_l$, the first-order term simplifies exactly to $W x$. The entire expression reduces to:
\begin{equation}
    W x + \sum_{l=1}^L \frac{\beta_l}{h \rho'(c)} \mathcal{E}_l(h)
\end{equation}
Because $\|\mathcal{E}_l(h)\| = o(h)$, the scaled remainder term strictly evaluates to $\frac{1}{h} o(h) = o(1)$. As the scaling parameter $h \to 0$, the residual error uniformly vanishes, and the ensemble converges exactly to the dense affine transformation $Wx$.
\end{proof}

\subsection{Proof of Theorem~\ref{thm:multiplicative_depth}}

\begin{proof}
Let $Z \in \mathbb{R}^{N \times N}$ be an arbitrary invertible change-of-basis matrix. We define the block matrix $\mathcal{M}_k(Z)$ formed by taking the first $k \cdot r$ columns from $WZ$ and the remaining columns from $Z$. Because $W$ is an isomorphism, the set of matrices $Z$ for which the column subspaces of $WZ_{1:kr}$ and $Z_{kr+1:N}$ act as complementary subspaces forms a Zariski dense open set in $\mathbb{R}^{N \times N}$. Since the finite intersection of dense open sets remains dense, it is guaranteed that there exists an invertible matrix $Z$ such that the determinant $\det(\mathcal{M}_k(Z)) \neq 0$ for all $k \in \{1, \dots, L-1\}$. We fix one such matrix $Z$.

To justify nonemptiness for each $k$, fix a $kr$-dimensional subspace $A$. The equal-dimensional subspaces $A$ and $WA$ have a common complement $C$. Choosing the first $kr$ columns of $Z$ as a basis of $A$ and the remaining columns as a basis of $C$ makes both $Z$ and $\mathcal M_k(Z)$ invertible. Thus the relevant determinant polynomials are nonzero, and their finite product is nonzero. One choice of $Z$ therefore satisfies all these conditions simultaneously.

Choosing an arbitrary scalar $\alpha \neq 0$, we define the global error matrix:
\begin{equation}
    E = W - \alpha^L I_N
\end{equation}
We decompose this error matrix $E$ into $L$ rank-bounded components using our chosen basis $Z$. Let $C_j \in \mathbb{R}^{N \times r_j}$ represent the $j$-th block of columns of the product $E Z$, and let $D_j^T \in \mathbb{R}^{r_j \times N}$ represent the corresponding $j$-th block of rows of the inverse basis $Z^{-1}$. By block matrix multiplication, we can express the error as:
\begin{equation}
    E = \sum_{j=1}^L C_j D_j^T
\end{equation}
Since the inner dimension is exactly $r_j \le r$, each term strictly has a rank of at most $r$. We now define a sequence of partial product matrices for $k = 0, \dots, L$:
\begin{equation}\label{eq:partial_product}
    P_k = \alpha^{k-L} \left( \alpha^L I_N + \sum_{j=1}^k C_j D_j^T \right)
\end{equation}
Notice that the boundary conditions perfectly align, yielding $P_0 = I_N$ and $P_L = W$. Consider the product $P_k Z$. Because $D_j^T$ extracts the $j$-th block of columns from any matrix multiplied by $Z$, the columns of $P_k Z$ perfectly match the columns of $\mathcal{M}_k(Z)$ up to scalar multiples of $\alpha$. Factoring out these scalars yields the proportional determinant relationship $\det(P_k Z) \propto \det(\mathcal{M}_k(Z))$. Since we explicitly chose $Z$ such that $\det(\mathcal{M}_k(Z)) \neq 0$, it strictly follows that $\det(P_k) \neq 0$, making every intermediate matrix $P_k$ globally invertible.

We explicitly construct our sparse low-rank components by setting $U_k = \alpha^{k-L} C_k$ and $V_k = (P_{k-1}^{-1})^T D_k$. We verify this decomposition by defining the iterative product $Q_k = \prod_{j=1}^k (\alpha I_N + U_j V_j^T)$ and demonstrating that $Q_k = P_k$ by mathematical induction. For the base case $k=0$, $Q_0 = P_0$. Assuming $Q_{k-1} = P_{k-1}$, we evaluate the next step:
\begin{equation}\label{eq:Q_k_step}
    Q_k = (\alpha I_N + U_k V_k^T) P_{k-1} = \alpha P_{k-1} + U_k \left(V_k^T P_{k-1}\right)
\end{equation}
By substituting our definition of $V_k$, the term $V_k^T P_{k-1}$ simplifies exactly to $D_k^T$. This reduces the equation to:
\begin{equation}
    Q_k = \alpha P_{k-1} + \alpha^{k-L} C_k D_k^T
\end{equation}
Expanding the $\alpha P_{k-1}$ expression and adding the low-rank term incorporates the $k$-th block of the decomposition, yielding exactly $P_k$. By mathematical induction, the final product evaluates to $Q_L = P_L = W$.

{Having established the algebraic decomposition $W = \prod_{l=1}^L M_l$, we now rigorously prove the asymptotic convergence of the sequential mapping. Let $\mathcal{A}_h(z) = (\Psi_h \circ \rho \circ \Phi_h)(z)$ represent the intermediate identity approximation block. By substituting the definitions of $\Phi_h$ and $\Psi_h$ and expanding the activation function via Taylor's Theorem, for any bounded input $z$, this block applies the identity mapping plus a vanishing uniform error: $\mathcal{A}_h(z) = z + \frac{\mathcal{R}(hz)}{h\rho'(c)}$, where the remainder strictly satisfies $\|\mathcal{R}(hz)\| = o(h)$.}

{Let the ideal linear state sequence be recursively defined as $Z^{(0)} = x$ and $Z^{(l)} = M_l Z^{(l-1)}$, noting that $Z^{(L-1)} = (\prod_{l=1}^{L-1} M_l)x$. We track the actual forward-propagated state of the deep network, defined recursively as $\tilde{Z}^{(0)} = x$ and $\tilde{Z}^{(l)} = \mathcal{A}_h(M_l \tilde{Z}^{(l-1)})$. We proceed by induction to show $\|\tilde{Z}^{(l)} - Z^{(l)}\| \to 0$ uniformly as $h \to 0$.}

{\textbf{Base Case ($l=1$):} The first block evaluates to $\tilde{Z}^{(1)} = \mathcal{A}_h(M_1 x) = M_1 x + \frac{\mathcal{R}(h M_1 x)}{h\rho'(c)}$. Because $x$ resides in a compact set $\mathcal{K}$, the linear projection $M_1 x$ is strictly bounded. The scaled remainder evaluates to $\frac{o(h)}{h} = o(1)$. Thus, $\tilde{Z}^{(1)} = Z^{(1)} + o(1)$.}

{\textbf{Inductive Step:} Assume $\tilde{Z}^{(l-1)} = Z^{(l-1)} + \mathcal{E}_{l-1}(h)$ where $\mathcal{E}_{l-1}(h) = o(1)$. The next block evaluates to $\tilde{Z}^{(l)} = \mathcal{A}_h(M_l Z^{(l-1)} + M_l \mathcal{E}_{l-1}(h))$. Applying the Taylor expansion to this specific pre-activation vector yields:}
\begin{equation}
    {\tilde{Z}^{(l)} = M_l Z^{(l-1)} + M_l \mathcal{E}_{l-1}(h) + \frac{\mathcal{R}\big(h M_l (Z^{(l-1)} + \mathcal{E}_{l-1}(h))\big)}{h\rho'(c)}}
\end{equation}
{Because $M_l$ is a fixed, bounded linear operator, the transformed prior error remains $M_l \mathcal{E}_{l-1}(h) = o(1)$. Furthermore, because the argument inside the remainder scales linearly with $h$ over a bounded state, the Taylor remainder strictly evaluates to $o(h)$, making the fraction $o(1)$. Therefore, the total accumulated error is the sum of two $o(1)$ terms, strictly remaining $o(1)$ uniformly across the domain.}

{\textbf{Final Readout ($l=L$):} By induction, the network reaches the penultimate state $\tilde{Z}^{(L-1)} = Z^{(L-1)} + o(1)$. The final layer applies the target bias $b$ and the activation function $\rho$ without the inverse mapping $\Psi_h$, evaluating to:}
\begin{equation}
    {H^{(L)} = \rho\left(M_L \tilde{Z}^{(L-1)} + b\right) = \rho\left(M_L Z^{(L-1)} + b + M_L o(1)\right)}
\end{equation}
{Since $M_L Z^{(L-1)} = Wx$, the inner argument converges uniformly to $Wx + b$ as $h \to 0$. Because $\rho$ is continuous on a compact neighborhood containing both the limiting affine image and its sufficiently small perturbations, it is uniformly continuous, mathematically permitting us to pass the limit inside the function:}
\begin{equation}
    {\lim_{h \to 0} H^{(L)} = \rho\left( \lim_{h \to 0} [Wx + b + o(1)] \right) = \rho(Wx + b)}
\end{equation}
{This completes the proof that the sequential composition of $h$-parameterized DLoR components uniformly recovers the exact dense layer mapping, preserving the original width $N$.}

\paragraph{Main idea.}
The construction replaces a dense linear map by a sequence of DLoR maps acting on a fixed-width state. The matrix factorization reproduces the target linear transformation, while centering and rescaling the activation near a point $c$ with $\rho'(c)\ne 0$ makes each nonlinear step converge uniformly to its intended linear update on the compact set of intermediate states. Consequently, the final pre-activation converges uniformly to $Wx+b$. Uniform continuity of $\rho$ on a compact neighborhood then gives uniform convergence of the network output to $\rho(Wx+b)$. The state width remains $N$ throughout, while the factorization determines the depth.
\end{proof}

\subsection{An exact rectangular lift}
\begin{proposition}[Rectangular shear factorization]
\label{prop:rectangular_shear}
Let $W\in\mathbb R^{m\times n}$, let $r\ge1$ be an integer, and put $D=n+m$. Define $Jx=(x,0)$ and $P(u,v)=v$. Then
\[
T_W=\begin{bmatrix}I_n&0\\W&I_m\end{bmatrix}
\quad\text{satisfies}\quad PT_WJ=W.
\]
The matrix $T_W$ is invertible and is a product of $q=\lceil\operatorname{rank}(W)/r\rceil$ matrices of the form $I_D+E_l$, with $\operatorname{rank}(E_l)\le r$. For $W=0$, the product is the identity and $q=0$.
\end{proposition}
\begin{proof}
Choose a rank decomposition $W=\sum_{l=1}^q W_l$ with $\operatorname{rank}(W_l)\le r$ and put $E_l=\begin{bmatrix}0&0\\W_l&0\end{bmatrix}$. Since $E_iE_j=0$ for all $i,j$,
\[
(I_D+E_q)\cdots(I_D+E_1)
=I_D+\sum_{l=1}^qE_l=T_W.
\]
Moreover, $T_W^{-1}=T_{-W}$. Each $E_l$ has rank at most $r$ and hence admits a factorization $U_lV_l^T$. The fixed embedding and selector give $PT_WJx=Wx$.
\end{proof}
The bias of an affine map can be placed in the bottom coordinates. This result is exact at the matrix level; activation-aware realization uses the normalization in Lemma~\ref{lem:kidger_4_1}. Its state dimension is $n+m$, whereas the width-preserving network transfer uses the separate common-width construction in Remark~\ref{rem:rectangular}.
\subsection{Finite-scale approximation bounds}
\label{app:finite_scale}
\begin{proposition}[Finite-scale normalization]
\label{prop:finite_scale}
Suppose $\rho\in C^2([c-\tau,c+\tau])$, let $a=\rho'(c)\ne0$, and assume $|\rho''|\le M$ on this interval. Define
\[
\mathcal A_h(z)=
\frac{\rho(c\mathbf1+hz)-\rho(c)\mathbf1}{ha},
\qquad \kappa=\frac{M}{2|a|}.
\]
If $\|z\|_\infty\le R$ and $0<hR\le\tau$, then
\[
\|\mathcal A_h(z)-z\|_\infty\le\kappa hR^2.
\]
For the additive construction, write $R_l=\sup_{x\in\mathcal K}\|M_lx\|_\infty$. If $hR_l/|\beta_l|\le\tau$ for all $l$, its uniform linear-output error is at most
\[
\kappa h\sum_{l=1}^L\frac{R_l^2}{|\beta_l|}.
\]
For sequential normalized factors, let $z_0(x)=\widetilde z_0(x)=x$, $z_l=M_lz_{l-1}$, and $\widetilde z_l=\mathcal A_h(M_l\widetilde z_{l-1})$. Set
\[
S_l=\sup_{x\in\mathcal K}\|z_l(x)\|_\infty,
\quad q_l=\|M_l\|_{\infty\to\infty},
\quad G_0=0,\quad G_l=1+q_lG_{l-1}.
\]
Choose $R>\max_{1\le l\le L}S_l$. If $hR\le\tau$ and $q_l\kappa hR^2G_{l-1}\le R-S_l$ for $2\le l\le L$,
then
\[
\sup_{x\in\mathcal K}
\|\widetilde z_L(x)-z_L(x)\|_\infty
\le\kappa hR^2G_L.
\]
\end{proposition}
\begin{proof}
The scalar second-order Taylor remainder is at most $M h^2|z_i|^2/2$ in coordinate $i$; division by $|ha|$ gives the first bound. Applying this coordinatewise to each additive branch and summing gives the second bound. For the sequential construction, suppose $e_{l-1}\le\kappa hR^2G_{l-1}$.
The stated condition ensures
\[
\|M_l\widetilde z_{l-1}\|_\infty
\le S_l+q_le_{l-1}\le R.
\]
Hence $e_l\le q_le_{l-1}+\kappa hR^2 \le\kappa hR^2G_l$, and induction proves the claim.
\end{proof}
For a final map $z\mapsto\rho(Mz+b)$ that is Lipschitz on the relevant compact neighborhood, the incoming error is multiplied by at most its Lipschitz constant. Under continuity alone, it is controlled by the corresponding modulus of continuity. These bounds give finite positive scales for any prescribed error. Their constants depend on the factor norms, and the normalization contains $1/(h|a|)$, so the real-arithmetic bounds do not imply uniform numerical conditioning as $h$ decreases.
\subsection{Exact realization with paired activation branches}
\label{app:paired_exact}
\begin{proposition}\label{prop:paired_exact}
Suppose $\rho(t)-\rho(-t)=a t$ for all $t\in\mathbb R$,
where $a\ne0$. For every DLoR matrix $M$,
\[
\frac{\rho(Mz)-\rho(-Mz)}a=Mz
\qquad(z\in\mathbb R^N).
\]
\end{proposition}
\begin{proof}
Apply the scalar identity to each coordinate of $Mz$.
Both $M$ and $-M$ are DLoR matrices.
\end{proof}
The identity holds with $a=1$ for ReLU, Softplus, SiLU,
and exact GELU. Thus each matrix factor can be realized
at finite scale by two tied activation branches followed
by fixed subtraction. The construction uses twice the
intermediate branch width and an explicit aggregation
interface. It removes the small-$h$ normalization from
this realization, while the original single-branch UAT
continues to follow from Assumption~\ref{ass:main}.

\subsection{A sharp factor-count bound for linear products}
\label{app:linear_factor_count}
\begin{proposition}\label{prop:linear_factor_count}
If $W=M_L\cdots M_1$, where
$M_l=\alpha_l I_N+R_l$ and $\operatorname{rank}(R_l)\le r$,
then
\[
L\ge\left\lceil
\frac{\min_{\gamma\in\mathbb R}
\operatorname{rank}(W-\gamma I_N)}r
\right\rceil.
\]
For every even $N$, there is an invertible real matrix
for which this lower bound equals $\lceil N/r\rceil$.
For the fixed shear lift $T_W$ in
Proposition~\ref{prop:rectangular_shear}, it equals
$\lceil\operatorname{rank}(W)/r\rceil$.
\end{proposition}
\begin{proof}
Expanding one factor at a time gives
\[
\operatorname{rank}\!\left(
M_L\cdots M_1-\Bigl(\prod_{l=1}^L\alpha_l\Bigr)I_N
\right)\le\sum_{l=1}^L\operatorname{rank}(R_l)\le Lr.
\]
For even $N$, choose $W$ block diagonal with $N/2$ copies
of $\begin{bmatrix}0&-1\\1&0\end{bmatrix}$.
Then $W-\gamma I_N$ is invertible for every real $\gamma$.
For the shear lift, $T_W-\gamma I_{n+m}$ is invertible
when $\gamma\ne1$, and has rank $\operatorname{rank}(W)$
when $\gamma=1$. The stated bounds follow.
\end{proof}
This result concerns exact linear products. The lower
bound for nonlinear fixed-width networks is treated
separately in Appendix~\ref{app:lower_bound_exact}.
\subsection{Exact ReLU simulation and network transfer}
\label{app:relu_exact}

\begin{lemma}[Factorization of arbitrary square matrices]\label{lem:relu_square_factors}
For $W\in\mathbb R^{N\times N}$ and $1\le r\le N$, there are $q\le\lceil 3N/r\rceil$ matrices $M_j=I_N+U_jV_j^T$ with $\operatorname{rank}(U_jV_j^T)\le r$ such that $W=M_q\cdots M_1$. Singular matrices are included.
\end{lemma}
\begin{proof}
Write a full singular value decomposition $W=Q\Sigma R^T$. Every orthogonal matrix in dimension $N$ is a product of at most $N$ Householder reflections. To see this, a reflection maps its first column to the first coordinate vector, unless they already agree; the resulting orthogonal matrix fixes that vector, and induction applies to its restriction to the remaining $N-1$ coordinates. Each reflection has the form $I_N-2vv^T$ for a unit vector $v$. Moreover,
\[
\Sigma=\prod_{i=1}^N\bigl(I_N+(\sigma_i-1)e_ie_i^T\bigr),
\]
including when some singular values $\sigma_i$ are zero. Concatenating these factorizations gives $W=A_p\cdots A_1$, where $p\le3N$ and every $A_i-I_N$ has rank at most one. An identity factor can be retained when needed so that $p\ge1$.

For any $k$ consecutive factors, the telescoping identity
\[
A_k\cdots A_1-I_N
=\sum_{i=1}^k A_k\cdots A_{i+1}(A_i-I_N)
\]
shows that their product differs from $I_N$ by a matrix of rank at most $k$; an empty product denotes $I_N$. Grouping consecutive factors in blocks of size at most $r$ therefore yields the stated $q\le\lceil p/r\rceil\le\lceil3N/r\rceil$ DLoR factors. Any update of rank at most $r$ can be written as $U_jV_j^T$ with $U_j,V_j\in\mathbb R^{N\times r}$ by padding its factors with zero columns.
\end{proof}

\begin{proof}[Proof of Proposition~\ref{prop:relu_exact}]
Use Lemma~\ref{lem:relu_square_factors} to write $W=M_q\cdots M_1$. When $W$ is invertible, the matrix factorization in Theorem~\ref{thm:multiplicative_depth} instead gives $q\le\lceil N/r\rceil$; the following construction applies to either factorization. Put $A_0=I_N$, $A_j=M_j\cdots M_1$, and $t_0=0$. If $\mathcal K$ is empty the claim is immediate. Otherwise, for each $1\le j<q$ choose
\[
t_j=\left(1+\max_{x\in\mathcal K}\|A_jx\|_\infty\right)\mathbf1_N.
\]
These vectors are finite by compactness. Define $z_0=x$ and
\begin{equation}\label{eq:relu_shift_stages}
z_j=\operatorname{ReLU}\bigl(M_jz_{j-1}+t_j-M_jt_{j-1}\bigr),
\qquad 1\le j<q.
\end{equation}
Inductively, $z_{j-1}=A_{j-1}x+t_{j-1}$ makes the argument of the activation equal to $A_jx+t_j$, whose coordinates are positive for every $x\in\mathcal K$. Hence $z_j=A_jx+t_j$ exactly. The final stage consequently satisfies
\begin{equation}\label{eq:relu_shift_output}
\operatorname{ReLU}\bigl(M_qz_{q-1}+b-M_qt_{q-1}\bigr)
=\operatorname{ReLU}(Wx+b).
\end{equation}
Omitting this last activation yields $Wx+b$. For $q=1$, there are no intermediate stages and the formula holds with $z_0=x$ and $t_0=0$. Only biases are added to the factorization, so every weight remains DLoR and every state has width $N$. The finite shifts depend on the target factors and compact domain; no bound on coefficient magnitudes is asserted.
\end{proof}

\begin{proposition}[Exact transfer of ReLU networks]\label{cor:relu_transfer}
Let $F:\mathbb R^{n_0}\to\mathbb R^{n_H}$ be a feedforward network with $H\ge1$ affine stages, ReLU after every stage except the final affine readout, and stage dimensions $n_0,\ldots,n_H\le N$. For any compact $\mathcal K\subset\mathbb R^{n_0}$ and $1\le r\le N$, there is a DLoR network $G$ with state width $N$, fixed input padding and output projection, and update rank at most $r$, such that $G=F$ on $\mathcal K$. Its number $D$ of affine stages and allocated parameter count $P$ satisfy
\begin{equation}\label{eq:relu_exact_counts}
D\le H\left\lceil\frac{3N}{r}\right\rceil,
\qquad
P\le H\left\lceil\frac{3N}{r}\right\rceil(2Nr+N+1).
\end{equation}
In particular, for rank-one updates, $D\le3NH$ and $P\le3NH(3N+1)$. \revised{These counts use the generic factorization of Lemma~\ref{lem:relu_square_factors}. Replacing it by the exact factorization of Theorem~\ref{thm:A} lowers the factor $3N$ to $\kappa(\widehat W_j)\le N$ for each simulated stage, and the rectangular construction of Theorem~\ref{thm:B}(i) lowers it further to the output dimension of that stage; see Corollary~\ref{cor:exact_transfer_cost}.}
\end{proposition}
\begin{proof}
Let $J_n:\mathbb R^n\to\mathbb R^N$ append zeros, and let $P_n:\mathbb R^N\to\mathbb R^n$ select the first $n$ coordinates. For a source stage with weight $W_j\in\mathbb R^{n_j\times n_{j-1}}$ and bias $b_j$, use the padded square weight and bias
\[
\widehat W_j=J_{n_j}W_jP_{n_{j-1}},
\qquad \widehat b_j=J_{n_j}b_j.
\]
Starting from $J_{n_0}x$, each completed hidden stage produces the zero-padded source state because $\operatorname{ReLU}(0)=0$. The input set to each stage is compact, being the image of $\mathcal K$ under a continuous finite network. Apply Proposition~\ref{prop:relu_exact} on this set to simulate each hidden stage, and use its affine-output variant for the final stage. Induction over the source stages gives exact agreement of the completed blocks with the padded source network. Applying $P_{n_H}$ then gives $G(x)=F(x)$ on $\mathcal K$. Intermediate shifted states may use all $N$ coordinates; padding is restored at each completed source stage.

Each source affine stage contributes at most $\lceil3N/r\rceil$ DLoR stages. Counting all entries in $U,V\in\mathbb R^{N\times r}$, the $N$ biases, and one diagonal scalar gives $2Nr+N+1$ allocated parameters per stage, even when the construction fixes some entries. Fixed input padding and output projection carry no trainable parameters. This proves the bounds.
\end{proof}

\paragraph{Exact additive realization.}
The parallel architecture of Section~\ref{sec:additive} also admits an exact ReLU construction with its existing branches. Let $W=\sum_{l=1}^L M_l$, $L\ge2$, and choose nonzero scalars $\beta_l$ with $\sum_l\beta_l=0$. On a nonempty compact domain $\mathcal K$, choose a finite scalar
\[
c>\max_{1\le l\le L}\max_{x\in\mathcal K}
\left\|\frac{M_lx}{\beta_l}\right\|_\infty.
\]
Every branch is then in the positive linear region, and
\begin{equation}\label{eq:relu_exact_additive}
\sum_{l=1}^L\beta_l\operatorname{ReLU}
\left(c\mathbf1_N+\frac{M_lx}{\beta_l}\right)
=c\mathbf1_N\sum_{l=1}^L\beta_l+\sum_{l=1}^L M_lx=Wx.
\end{equation}
Adding $b$ after aggregation, followed by ReLU when the target stage is hidden, exactly realizes the desired target stage. A one-term decomposition can be extended by a zero branch to permit nonzero zero-sum coefficients. The structured rectangular aggregation is $B=[\beta_1I_N,\ldots,\beta_LI_N]$, as in the width-expanded architecture; each branch retains its DLoR form under scalar rescaling.

\section{Proofs for Section~\ref{sec:uat}}
We benchmark our capacity against the bounded-width UAT established for full-rank dense neural networks.
\begin{lemma}[\citep{kidger2020universal}]\label{lem:kidger_uat}
Let $\rho: \mathbb{R} \to \mathbb{R}$ be any non-affine continuous function, continuously differentiable at at least one point with a non-zero derivative. Let $\mathcal{K} \subset \mathbb{R}^n$ be compact. The class of full-rank neural networks of arbitrary depth and strictly bounded width $N = n+m+2$, utilizing activation $\rho$, is dense in $C(\mathcal{K}, \mathbb{R}^m)$ with respect to the uniform norm.
\end{lemma}
\label{sec:proof5}
\subsection{Proof of Theorem~\ref{thm:uat_multiplicative}}
\begin{proof}
By Lemma~\ref{lem:kidger_uat}, choose a dense network $F_0$ with $\|F_0-f\|_{\infty,\mathcal K}<\epsilon/3$. Embed its affine stages at the common width $N$ as in Remark~\ref{rem:rectangular}. Perturb its finitely many square weights to invertible matrices to obtain $F$ with $\|F-F_0\|_{\infty,\mathcal K}<\epsilon/3$. It remains to simulate $F$ with error below $\epsilon/3$.

If a factorization has $L=1$, the target affine matrix is already DLoR and no intermediate identity approximation is required. For $L\ge2$, use the following construction. For any dense layer evaluating $x \mapsto \rho(Wx + b)$, we factor $W = \prod_{l=1}^L M_l$ where each component is strictly $M_l = \alpha I_N + U_l V_l^T$. Let the ideal linear state sequence be $Z^{(0)} = x$, and $Z^{(l)} = M_l Z^{(l-1)}$. We construct a sub-network of $L$ sparse low-rank layers parameterized by $h > 0$. For the initial layer, we set the weights to $W^{(1)} = h M_1$ and the bias to $B^{(1)} = c \mathbf{1}_N$. The output computes to:
\begin{equation}\label{eq:H1_proof}
    H^{(1)} = \rho(c \mathbf{1}_N + h Z^{(1)}) = \rho(c)\mathbf{1}_N + h \rho'(c) Z^{(1)} + \mathcal{E}^{(1)}(h)
\end{equation}
Because the intermediate state $Z^{(1)}$ is bounded on the compact set $\mathcal{K}$, the Taylor remainder uniformly satisfies $\|\mathcal{E}^{(1)}(h)\| = o(h)$. 

For all intermediate layers where $1 < l < L$, we maintain the linear state scaling by assigning the weights $W^{(l)} = \frac{1}{\rho'(c)} M_l$ and the compensating bias $B^{(l)} = c \mathbf{1}_N - \frac{\rho(c)}{\rho'(c)} M_l \mathbf{1}_N$. Assuming inductively that the previous output holds the form $H^{(l-1)} = \rho(c)\mathbf{1}_N + h \rho'(c) Z^{(l-1)} + \mathcal{E}^{(l-1)}(h)$ with a uniformly bounded $o(h)$ error, the pre-activation evaluates to:
\begin{equation}
    A^{(l)} = W^{(l)} H^{(l-1)} + B^{(l)} = c \mathbf{1}_N + h Z^{(l)} + \frac{1}{\rho'(c)} M_l \mathcal{E}^{(l-1)}(h)
\end{equation}
Applying the activation $\rho$ yields an output $H^{(l)} = \rho(c)\mathbf{1}_N + h \rho'(c) Z^{(l)} + \mathcal{E}^{(l)}(h)$. Indeed, the perturbation of $c\mathbf1_N$ in $A^{(l)}$ is $hZ^{(l)}+o(h)=O(h)$ uniformly. Differentiability at $c$ therefore gives a Taylor remainder $o(h)$, while the propagated previous remainder also remains $o(h)$.

For the final layer $L$, we reverse the scaling of $h$ and inject the target bias $b$. We assign $W^{(L)} = \frac{1}{h \rho'(c)} M_L$ and $B^{(L)} = b - \frac{\rho(c)}{h \rho'(c)} M_L \mathbf{1}_N$. The final pre-activation computes to:
\begin{equation}
    A^{(L)} = W^{(L)} H^{(L-1)} + B^{(L)} = Z^{(L)} + b + \frac{1}{h \rho'(c)} M_L \mathcal{E}^{(L-1)}(h)
\end{equation}
Since the complete sequence yields $Z^{(L)} = Wx$, and the previous error term is bounded by $o(h)$, the scaled error term evaluates strictly to $\frac{o(h)}{h} = o(1)$. Thus, as $h \to 0$, $A^{(L)}$ uniformly converges to $Wx + b$. Uniform continuity on a compact neighborhood gives uniform convergence to $\rho(Wx+b)$ for a hidden target layer. For the target network's final affine readout, omit this last activation
and use $A^{(L)}$ directly.

To compose the blocks, place each ideal intermediate image inside a fixed compact neighborhood. If $e_{k-1}$ is the incoming state error, $\delta_k(h)$ is the uniform block error on that neighborhood, and $\omega_k$ is a modulus of continuity of the $k$th target layer,
then
\[
e_k\le\delta_k(h)+\omega_k(e_{k-1}),\qquad e_0=0.
\]
There are finitely many target layers, $\delta_k(h)\to0$, and $\omega_k(t)\to0$ as $t\to0$. Induction therefore keeps the approximate states in these neighborhoods and makes the final simulation error smaller than $\epsilon/3$ for sufficiently small positive $h$. The three error contributions sum to less than $\epsilon$, proving the claim.
\end{proof}

\subsection{Proof of Theorem~\ref{thm:uat_additive}}
\begin{proof}
Choose a finite dense network $F$ with $\|F-f\|_{\infty,\mathcal K}<\epsilon/2$ by Lemma~\ref{lem:kidger_uat}, using the boundary embedding in Remark~\ref{rem:rectangular}.

Given a dense target layer evaluating $y = \rho(Wx + b)$, we apply the additive width decomposition $W = \sum_{l=1}^L M_l$. As rigorously established in the proof of Theorem~\ref{thm:additive_width}, the parallel ensemble computes a pre-activation response parameterized by $h$:
\begin{equation}
    \tilde{A}(x, h) = \sum_{l=1}^L \frac{\beta_l}{h \rho'(c)} \rho\left(\frac{h}{\beta_l} M_l x + c \mathbf{1}_N\right)
\end{equation}
We demonstrated that this structure evaluates to $\tilde{A}(x, h) = Wx + \mathcal{E}(h)$, where the remainder term $\mathcal{E}(h) \to 0$ uniformly over the compact domain $\mathcal{K}$ as $h \to 0$.

We then inject this parallel ensemble output into a subsequent exact readout layer, or directly into the next activation sequence, augmented with the original dense bias $b$. This evaluates to:
\begin{equation}
    \tilde{y} = \rho(\tilde{A}(x, h) + b) = \rho(Wx + b + \mathcal{E}(h))
\end{equation}
Because the chosen activation function $\rho$ is uniformly continuous on a compact neighborhood of the affine image of $\mathcal{K}$, for any arbitrary precision bound $\delta > 0$, there exists a sufficiently small scaling parameter $h > 0$ such that the magnitude of the remainder $\|\mathcal{E}(h)\|$ is strictly small enough to guarantee $\|\tilde{y} - y\|_\infty < \delta$. For the final affine readout, omit the outer activation. The finite-composition error argument in the proof of Theorem~\ref{thm:uat_multiplicative} then makes the total simulation error smaller than $\epsilon/2$. The triangle inequality gives the required approximation to $f$.
\end{proof}
\subsection{Proof of Theorem~\ref{thm:general_transfer}}

\begin{proof}
Embed the given network $F$ at width $W$ as in Remark~\ref{rem:rectangular}. For a prescribed tolerance $\eta$, perturb its square weights to obtain a network $f_{\mathrm{dense}}$ with invertible weights and $\|f_{\mathrm{dense}}-F\|_{\infty,\mathcal K}<\eta/2$. Write the perturbed network as a composition of affine stages and hidden activations:
\begin{equation}
    f_{\text{dense}}(x) = (\mathcal{L}_L \circ \dots \circ \mathcal{L}_1)(x)
\end{equation}
Here $\mathcal L_k(z)=\rho(W_kz+b_k)$ for $k<L$, whereas $\mathcal L_L(z)=W_Lz+b_L$ is affine. The fixed output selector is applied after the last stage.

Set $q=\lceil W/r\rceil$. By Theorem~\ref{thm:multiplicative_depth}, each invertible $W_k$ has a factorization $W_k=M_{k,q}\cdots M_{k,1}$ with update ranks at most $r$. The corresponding DLoR block $\widetilde{\mathcal L}_k(\cdot,h)$ has $q$ affine stages; for the final target stage its outer activation is omitted. On each fixed compact neighborhood of the ideal intermediate states, the block error satisfies
\begin{equation}
    \|\tilde{\mathcal{L}}_k(z, h) - \mathcal{L}_k(z)\|_\infty = o(1)\quad(h\to0)
\end{equation}
This ensures that as the scaling parameter $h$ approaches zero, the intermediate DLoR block uniformly converges to the exact dense layer mapping, $\lim_{h \to 0} \tilde{\mathcal{L}}_k(z, h) = \mathcal{L}_k(z)$, uniformly on these compact neighborhoods.

We construct the global DLoR network by sequentially composing these parameterized blocks, resulting in an expanded architecture of at most $Lq$ affine stages:
\begin{equation}
    \tilde{f}_{\text{DLoR}}(x, h) = (\tilde{\mathcal{L}}_L \circ \dots \circ \tilde{\mathcal{L}}_1)(x, h)
\end{equation}
Apply the compact-neighborhood induction from the proof of Theorem~\ref{thm:uat_multiplicative}. Its error recursion combines the uniform block errors with the moduli of continuity of the target layers. Since the number of stages is finite, the composed network converges uniformly to $f_{\mathrm{dense}}$:
\begin{equation}
    \lim_{h \to 0} \|\tilde{f}_{\text{DLoR}}(x, h) - f_{\text{dense}}(x)\|_\infty = 0
\end{equation}

Choose a finite $h^*>0$ so that $\|\widetilde f_{\mathrm{DLoR}}(\cdot,h^*)- f_{\mathrm{dense}}\|_{\infty,\mathcal K}<\eta/2$. Writing $G=\widetilde f_{\mathrm{DLoR}}(\cdot,h^*)$, we obtain
\[
\|G-F\|_{\infty,\mathcal K}<\eta.
\]
If $\|F-f\|_{\infty,\mathcal K}\le\epsilon$, then $\|G-f\|_{\infty,\mathcal K}<\epsilon+\eta$ by the triangle inequality. In particular, a dense approximation with error below half a prescribed tolerance transfers to a DLoR approximation within that tolerance.
\end{proof}

\subsection{Proof of Theorem~\ref{thm:dlor_sobolev}}
\label{sec:proof_sobolev_repaired}
\begin{proof}
Use \citet[Theorem~1]{siegel2023optimal} with $p=q=\infty$; its embedding condition is $0<s/d$, which holds. At width $N=25d+31$, the infimal approximation error of a ReLU network with $L$ hidden layers is bounded by $C_0L^{-2s/d}$ on $\mathcal B_{d,s}$. These networks have $L+1$ affine stages. Choose
\[
 L=\left\lceil(4\max\{C_0,1\}/\epsilon)^{d/(2s)}\right\rceil.
\]
The infimum is then at most $\epsilon/4$, so there is a particular network $F$ with $\|F-f\|_{L^\infty}<\epsilon/2$. For the continuous representative of $f$, this essential-supremum error equals the uniform error on the closed cube.

Apply Theorem~\ref{thm:general_transfer} at rank one to all $L+1$ affine stages of $F$, including its rectangular input and output stages via padding. With simulation tolerance $\epsilon/2$, the resulting $G$ has error below $\epsilon$. Each original affine stage uses at most $N$ DLoR stages, each with $2N$ factor entries, $N$ bias entries, and one diagonal scalar. Consequently,
\begin{align*}
 D(G)&\le N(L+1),\\
 P(G)&\le (3N+1)N(L+1)=O_{d,s}(\epsilon^{-d/(2s)}).
\end{align*}
All allocated slots are counted, including zeros and repeated scalar values; fixed padding and coordinate projection have no trainable entries. Because $N=25d+31$ is independent of $\epsilon$, the DLoR conversion changes the count only by a constant depending on $d$. Selecting smaller perturbations or $h$ changes coefficient values without adding slots. The real-parameter model permits the arbitrarily fine precision used by the construction.
\end{proof}

\section{Alternative Construction via Augmented State Space (Width $2N$)}
\label{sec:augmented_space}
While we previously established a purely multiplicative decomposition that strictly preserves the original network width $N$, we present here an alternative, highly intuitive additive construction. By temporarily relaxing the width constraint to $2N$, we can explicitly accumulate rank-1 updates in an augmented state space before applying the non-linear activation. 

We aim to approximate a target full-rank neural network layer, defined by the mapping $x \mapsto \rho(Wx + b)$ where $W \in \mathbb{R}^{N \times N}$ and $b \in \mathbb{R}^N$. We replace this single dense layer with a block of layers in a Deep Sparse Low Rank architecture. The ideal accumulation maps have the form $I_{2N}$ plus a
rank-one update; the activation-normalized network weights are scalar multiples of these maps and remain DLoR.

We decompose the target matrix $W$ into a sum of $k$ independent rank-1 matrices, where the number of components $k = \text{rank}(W) \le N$:
\begin{equation}
    W = \sum_{i=1}^k u_i v_i^T
\end{equation}
where $u_i, v_i \in \mathbb{R}^N$ are column vectors. To accumulate these rank-1 projections without corrupting the original input features, we construct a network operating in an augmented, concatenated state space of width $2N$. The ideal, error-free state vector at any step $l$ is denoted as $Z^{(l)} = \begin{bmatrix} X^{(l)} \\ Y^{(l)} \end{bmatrix} \in \mathbb{R}^{2N}$. The top block $X^{(l)} \in \mathbb{R}^N$ perfectly preserves the original input $x$, while the bottom block $Y^{(l)} \in \mathbb{R}^N$ acts as a memory buffer to sequentially accumulate the rank-1 transformations. We initialize this state with the zero-padded input, $Z^{(0)} = \begin{bmatrix} x \\ 0 \end{bmatrix}$.

Our strategy applies a sequence of ideal linear transformations, denoted by the matrices $M_i$. Each $M_i$ represents the application of the $i$-th rank-1 update to the augmented state. We define $M_i = I_{2N} + \tilde{u}_i \tilde{v}_i^T$, where the augmented vectors are zero-padded to match the $2N$ dimensional space:
\begin{equation}
    \tilde{u}_i = \begin{bmatrix} 0 \\ u_i \end{bmatrix} \in \mathbb{R}^{2N}, \quad \tilde{v}_i = \begin{bmatrix} v_i \\ 0 \end{bmatrix} \in \mathbb{R}^{2N}
\end{equation}
Multiplying the state vector by the operator $M_i$ perfectly preserves the identity mapping for the input block $X$ while cleanly accumulating the rank-1 projection into the buffer block $Y$:
\begin{equation}
    M_i \begin{bmatrix} X \\ Y \end{bmatrix} = \begin{bmatrix} I_N & 0 \\ 0 & I_N \end{bmatrix}\begin{bmatrix} X \\ Y \end{bmatrix} + \begin{bmatrix} 0 \\ u_i \end{bmatrix} \begin{bmatrix} v_i \\ 0 \end{bmatrix}^T \begin{bmatrix} X \\ Y \end{bmatrix} = \begin{bmatrix} X \\ Y + u_i v_i^T X \end{bmatrix}
\end{equation}

We construct a sub-network of exactly $k+1$ layers to implement these ideal $M_i$ transformations while navigating the non-linear activation function $\rho$. Throughout this construction, we denote the $l$-th layer's weight matrix and bias vector as $W^{(l)}$ and $B^{(l)}$. The pre-activation linear state is defined as $A^{(l)} = W^{(l)} H^{(l-1)} + B^{(l)}$, and the post-activation output is $H^{(l)} = \rho(A^{(l)})$. We define an expansion point $c$ where $\rho'(c) \neq 0$, and track the cumulative asymptotic error $E^{(l)}(h)$ between the actual network state $H^{(l)}$ and the ideal linear state $Z^{(l)}$. For brevity in our bias formulations, let $s_l = \tilde{v}_l^T \mathbf{1}_{2N}$ represent the sum of the elements in the augmented vector $\tilde{v}_l$.

To ensure rigorous uniform approximation, let $K \subset \mathbb{R}^N$ be the compact set of inputs $x$. The initial padded states reside in a compact set $\mathcal{K}^{(0)} \subset \mathbb{R}^{2N}$. Because each ideal operator $M_i$ is a continuous linear transformation, the ideal forward-propagated states $Z^{(l)} = M_l \dots M_1 Z^{(0)}$ also reside in strictly bounded compact sets $\mathcal{K}^{(l)}$. By Taylor's Theorem, for any compact domain, the element-wise activation function can be expanded as:
\begin{equation}
    \rho(c \mathbf{1}_{2N} + \Delta) = \rho(c)\mathbf{1}_{2N} + \rho'(c)\Delta + \mathcal{R}(\Delta)
\end{equation}
where the remainder satisfies $\|\mathcal{R}(\Delta)\| = o(\|\Delta\|)$ uniformly for any bounded perturbation $\Delta$. We proceed to construct the network layer by layer, proving inductively that the error is uniformly bounded by $o(h)$.

\textbf{Layer 1 ($l=1$):} We encode the initial state into the localized linear activation regime around the expansion point $c$. We define the constructed weights and biases as $W^{(1)} = h I_{2N} + h \tilde{u}_1 \tilde{v}_1^T = h M_1$ and $B^{(1)} = c \mathbf{1}_{2N}$. The actual pre-activation vector is $A^{(1)} = c \mathbf{1}_{2N} + h M_1 Z^{(0)}$. Applying the activation function yields the output:
\begin{equation}
    H^{(1)} = \rho(c \mathbf{1}_{2N} + h Z^{(1)}) = \rho(c)\mathbf{1}_{2N} + h \rho'(c) Z^{(1)} + E^{(1)}(h)
\end{equation}
where the ideal state updates to $Z^{(1)} = M_1 Z^{(0)}$. The error term is exactly the Taylor remainder $E^{(1)}(h) = \mathcal{R}(h Z^{(1)})$. Since $Z^{(1)} \in \mathcal{K}^{(1)}$ is bounded, the perturbation magnitude is strictly $\mathcal{O}(h)$, bounding the remainder by $\|E^{(1)}(h)\| = o(h)$ uniformly.

\textbf{Intermediate Layers ($1 < l \le k$):} Assume inductively that the previous output satisfies $H^{(l-1)} = \rho(c)\mathbf{1}_{2N} + h \rho'(c) Z^{(l-1)} + E^{(l-1)}(h)$, with a uniformly bounded error $\|E^{(l-1)}(h)\| = o(h)$. We define the parameters for layer $l$ to unscale the previous output, apply the new projection $M_l$, and rescale back into the linear regime:
\begin{align}
    W^{(l)} &= \frac{1}{\rho'(c)} I_{2N} + \frac{1}{\rho'(c)} \tilde{u}_l \tilde{v}_l^T = \frac{1}{\rho'(c)} M_l \\
    B^{(l)} &= \left( c - \frac{\rho(c)}{\rho'(c)} \right) \mathbf{1}_{2N} - \frac{\rho(c) s_l}{\rho'(c)} \tilde{u}_l
\end{align}
Notice that the action of the linear operator $M_l$ on the vector of ones is exactly $M_l \mathbf{1}_{2N} = \mathbf{1}_{2N} + s_l \tilde{u}_l$. Therefore, we can rewrite the constructed bias compactly as $B^{(l)} = c \mathbf{1}_{2N} - \frac{\rho(c)}{\rho'(c)} M_l \mathbf{1}_{2N}$. Computing the pre-activation vector $A^{(l)}$ yields:
\begin{equation}
    A^{(l)} = \frac{\rho(c)}{\rho'(c)} M_l \mathbf{1}_{2N} + h M_l Z^{(l-1)} + \frac{1}{\rho'(c)} M_l E^{(l-1)}(h) + c \mathbf{1}_{2N} - \frac{\rho(c)}{\rho'(c)} M_l \mathbf{1}_{2N}
\end{equation}
The massive constant terms involving $\rho(c)$ perfectly cancel. Defining the new ideal state as $Z^{(l)} = M_l Z^{(l-1)}$, we obtain the simplified pre-activation:
\begin{equation}
    A^{(l)} = c \mathbf{1}_{2N} + \underbrace{h Z^{(l)} + \frac{1}{\rho'(c)} M_l E^{(l-1)}(h)}_{\Delta^{(l)}(h)}
\end{equation}
Passing this through the activation function gives $H^{(l)} = \rho(c)\mathbf{1}_{2N} + \rho'(c)\Delta^{(l)}(h) + \mathcal{R}(\Delta^{(l)}(h))$. Substituting the perturbation back into the linear term isolates the total error:
\begin{equation}
    H^{(l)} = \rho(c)\mathbf{1}_{2N} + h \rho'(c) Z^{(l)} + \underbrace{M_l E^{(l-1)}(h) + \mathcal{R}(\Delta^{(l)}(h))}_{E^{(l)}(h)}
\end{equation}
By the inductive hypothesis, the prior error is $o(h)$. Because $M_l$ is a bounded linear operator, the transformed prior error remains $o(h)$. The perturbation satisfies $\|\Delta^{(l)}(h)\| \le h \|Z^{(l)}\| + o(h) = \mathcal{O}(h)$, guaranteeing the Taylor remainder is $\mathcal{R}(\Delta^{(l)}(h)) = o(h)$. The sum of these terms is strictly $o(h)$, proving that the error remains uniformly bounded across all iterations. At step $k$, the ideal state successfully accumulates the full weight matrix: $Z^{(k)} = M_k \dots M_1 Z^{(0)} = \begin{bmatrix} x \\ Wx \end{bmatrix}$.

\textbf{Final Layer ($l=k+1$):} We construct the final weights and biases to apply the inverse scaling $1/h$ and inject the zero-padded target bias $\tilde{b} = \begin{bmatrix} 0 \\ b \end{bmatrix}$. We assign $W^{(k+1)} = \frac{1}{h \rho'(c)} I_{2N}$ and $B^{(k+1)} = \tilde{b} - \frac{\rho(c)}{h \rho'(c)} \mathbf{1}_{2N}$. The final pre-activation computes to:
\begin{equation}
    A^{(k+1)} = Z^{(k)} + \tilde{b} + \frac{E^{(k)}(h)}{h \rho'(c)}
\end{equation}
Since the tracked error is bounded by $o(h)$, the scaled error term strictly vanishes as $h \to 0$, giving $\frac{o(h)}{h} = o(1)$. Noting that $Z^{(k)} + \tilde{b} = \begin{bmatrix} x \\ Wx + b \end{bmatrix}$, the final network output vector evaluates to:
\begin{equation}
    H^{(k+1)} = \rho\left( \begin{bmatrix} x \\ Wx + b \end{bmatrix} + o(1) \right)
\end{equation}
Because the continuous function $\rho$ acts on a compact domain (the affine transformation of the bounded input set $K$), it is uniformly continuous on this bounded image. We are mathematically permitted to pass the limit inside the function, guaranteeing that the constructed network's output converges uniformly to the exact target vector $\begin{bmatrix} \rho(x) \\ \rho(Wx + b) \end{bmatrix}$ as $h \to 0$.
\vspace{0.5em}

Between two macroscopic layers, the ideal state
$[\rho(x)^T,y^T]^T$ must be mapped to $[y^T,0^T]^T$.
Let $S=\begin{bmatrix}0&I_N\\0&0\end{bmatrix}$ and
$S_\delta=S+\delta I_{2N}$, where $\delta>0$.
Factor the invertible matrix $S_\delta$ into at most $2N$
rank-one DLoR factors $T_j$. Use the normalized blocks
$\mathcal A_\tau=\Psi_\tau\circ\rho\circ\Phi_\tau$
after every factor, including the last:
\[
R_{\tau,\delta}
=\mathcal A_\tau\circ T_q\circ\cdots\circ
\mathcal A_\tau\circ T_1.
\]
For fixed $\delta$, this converges uniformly to $S_\delta$ on compact sets as $\tau\to0$. Its terminal $\Psi_\tau$ is affine, with scalar linear part, and is absorbed into the first affine stage of the next accumulation block. This preserves the DLoR constraint and introduces no additional activation of the reset state. Thus the new top block approximates $y$, and the bottom block approximates zero.
\begin{theorem}[Universal Approximation via Augmented State Space]
Under Assumption~\ref{ass:main}, let $K\subset\mathbb R^n$ be compact. Deep DLoR networks of arbitrary depth, state width $2(n+m+2)$, update rank at most one, fixed input padding and output coordinate selection, and an affine final readout are dense in $C(K,\mathbb R^m)$ with respect to the uniform norm.
\end{theorem}

\begin{proof}
Choose a finite dense approximant of width $N=n+m+2$ with
error below $\epsilon/3$, using the fixed boundary interfaces of Remark~\ref{rem:rectangular}. Each of its $H$ affine stages uses an accumulation block of at most $N+1$ stages; each of the $H-1$ transitions uses at most $2N$ additional reset stages. Thus the construction has at most $H(N+1)+2N(H-1)$ stages and state width $2N$. For the final target affine stage, omit the terminal activation and select the bottom output coordinates. First choose the finitely many reset perturbations small enough that their total output effect is below $\epsilon/3$. Then choose the activation scales small enough that the total simulation error is below $\epsilon/3$, using the same compact-neighborhood induction as in Theorem~\ref{thm:uat_multiplicative}. The resulting error is below $\epsilon$.
\end{proof}

\section{Experimental Settings}
\label{sec:app-exp}
\subsection{Illustration of UAT}
\label{subsec: validate_uat_app}
\paragraph{Synthetic Data}We choose the sawtooth function $f(x) = |((x\cdot \lambda) \mod 2)-1|$ to illustrate our results, where $x \in \mathbb{R}$, and $\lambda$ is a given frequency constant set to $\lambda=3.7$. The synthetic dataset consists of $400$ points in the interval $[-2, 2]$ generated by the sawtooth function and is split evenly into the training and test dataset.
\paragraph{Baseline MLP Training.} Then we implement a five-layer MLP, where the first four layers are hidden layers consist of weight matrices of dimensions $16 \times1$ for the first layer, and $16 \times 16$ for the next three layers, biases vectors of the corresponding dimension composed with the softplus activation function at each layer. The final output layer is of dimension $1\times 16$, and no activation function is applied as the goal is function approximation. We train the MLP using the training dataset by supervised learning with mean-squared-error loss (MSE) using the Adam optimizer with learning rate $0.01$ and $2000$ epochs. 

\paragraph{Construct Deep/Wide Neural Networks.} To meaningfully validate the universal approximation ability of the deep and wide neural networks, we need to guarantee the numerical stability in construction. For both the deep and wide network, we fix rank $r=6$ to induce a moderate number of substructures to avoid overly accumulating numerical imprecision, and $c=0.5$ where the softplus function and its first derivative are evaluated stably. For the deep network, we use $h=10^{-2}$ to $10^{-8}$ and $\alpha=0.8$ to avoid degenerative $\alpha^L$. For the wide network, we use $h=10^{-1}$ to $10^{-6}$, $\beta_l = 1$ for $l=1,\dots L-1$, and $\beta_L = -L+1$ to enforce the affine elimination. 

\subsection{UAT with MNIST dataset}
\label{subsec:app_uat_mnist}
\paragraph{Baseline dense MLP} The dense MLP to be approximated consists of a flattening layer and three linear layers with dimensions: $784 \times 256$, $256 \times 128$, $128 \times 10$. The softplus activation function is applied after the first two linear layers. The dense MLP is trained with the cross entropy loss with the Adam optimizer for $10$ epochs with learning rate $0.001$ and batch size $128$.
\paragraph{Construction of the Deep DLoR Model} The construction of the deep DLoR follows from the proof of \revised{Theorem~\ref{thm:uat_multiplicative}} using $h=0.0001$, $\alpha = 1$, $c=0.5$. For the three linear layers of the dense MLP, $r=128, 64, 5$ are used respectively to factorize.

\paragraph{Construction of the Wide DLoR Model} The construction follows from the proof of \revised{Theorem~\ref{thm:uat_additive}} using $h=0.001$, $c=0.5$, and $r=5$ is used to factorize all three linear layers.

\subsection{Metrics}
\label{sec:app_uat_mnist_metric}
\paragraph{Loss and Accuracy}
\[
\mathcal{L}
=
-\frac{1}{N}
\sum_{i=1}^{N}
\log
\frac{
\exp(f(x_i)_{y_i})
}{
\sum_{j=1}^{C}\exp(f(x_i)_j)
},
\qquad
\mathrm{Acc}
=
\frac{1}{N}
\sum_{i=1}^{N}
\mathbf{1}
\left[
\arg\max_j f(x_i)_j = y_i
\right].
\]
\paragraph{Max Logits Difference}
We use $1\le i \le N$ to iterate over all the data, and $1\le j \le 10$ denotes the $10$ classes.
\[
E_{\max}
=
\max_{1\le i\le N}
\max_{1\le j\le 10}
\left|
f_{\mathrm{DLoR}}(x_i)_j
-
f_{\mathrm{dense}}(x_i)_j
\right|.
\]
\paragraph{Max Relative Logits Difference}
We use $1\le i \le N$ to iterate over all the data, and $1\le j \le 10$ denotes the $10$ classes.
\[
E_{\max}^{\mathrm{rel}}
=
\frac{
\max_{i,j}
\left|
f_{\mathrm{DLoR}}(x_i)_j
-
f_{\mathrm{dense}}(x_i)_j
\right|
}{
\max_{i,j}
\left|
f_{\mathrm{dense}}(x_i)_j
\right|
}.
\]
\paragraph{Relative Logits Difference}
The logits difference is computed in Frobenius norm:
\[
E_{\mathrm{rel}}
=
\frac{
\left\|
F_{\mathrm{struct}}
-
F_{\mathrm{dense}}
\right\|_F
}{
\left\|
F_{\mathrm{dense}}
\right\|_F
},
\]
where

\[
F_{\mathrm{dense}}
=
\begin{bmatrix}
f_{\mathrm{dense}}(x_1)^\top\\
\vdots\\
f_{\mathrm{dense}}(x_N)^\top
\end{bmatrix},
\qquad
F_{\mathrm{struct}}
=
\begin{bmatrix}
f_{\mathrm{DLoR}}(x_1)^\top\\
\vdots\\
f_{\mathrm{DLoR}}(x_N)^\top
\end{bmatrix}.
\]

\subsection{Spectral Decomposition}
\label{subsec:spectral_decom_setup}
We use a four-layer MLP where each hidden layer is of dimension $64 \times64$ as the baseline model for constructing the deep and wide neural networks.
\paragraph{Notation and Setup.} For the deep network, define the effective pre-activation weight matrix $W = \alpha I + UV^\top$, where $UV^\top$ is the learned low-rank component of rank $r$ and $\alpha I$ is a learned scaled identity matrix. We perform a Singular Value Decomposition (SVD) on $W$ such that $W = P\Sigma Q^\top$. To isolate the contributions of the individual structural components to the overall network capacity, we project both $UV^\top$ and $\alpha I$ onto the shared left and right singular vectors of $W$. We define and plot the following vectors: (i) total singular values $\text{diag}(\Sigma)$; (ii) low-rank contributions $\text{diag}(P^\top UV^\top Q)$; (iii) identity contributions $\text{diag}(P^\top (\alpha I) Q)$. For the fixed rank $r=4$, the spectral decomposition of the second layer in the deep network is visualized in Figure \ref{fig:deep_spectral}.

For the wide network, the total output of a wide layer is $\sum_{i=1}^k \sigma(x W_i + b_i)$. While the activation function $\sigma$ prevents the exact algebraic collapse of these branches into a single matrix, we analyze the pre-activation linear span to assess structural redundancy. We define $W_{total} = \sum_{i=1}^k W_i$, where $W_i = \alpha_i V_i U_i^\top$. We compute the SVD of the combined matrix, $W_{total} = P \Sigma Q^\top$. To determine how much each individual branch contributes to the global pre-activation feature space, we project each $W_i$ onto the global singular vectors. For each branch $i \in \{1, \dots, k\}$, we compute the branch-wise contribution $|\text{diag}(P^\top W_i Q)|$. For the fixed rank $r=4$, the spectral decomposition of the wide network layer is visualized in Figure \ref{fig:wide_spectrl}.

\subsection{Lower Bound for Exact Representation of a Linear Layer}
\label{app:lower_bound_exact}

We next show that the depth required by the multiplicative construction is unavoidable in the worst case. Importantly, the lower bound continues to hold
for nonlinear networks satisfying the fixed-width and Jacobian conditions stated below between consecutive diagonal-plus-low-rank transformations.

Let
\[
\mathcal{D}_d
:=
\{D\in\mathbb{R}^{d\times d}: D \text{ is diagonal}\},
\]
and define the diagonal-rank complexity of a matrix \(W\in\mathbb{R}^{d\times d}\) by
\[
\delta(W)
:=
\min_{D\in\mathcal{D}_d}
\operatorname{rank}(W-D).
\]
Thus, \(\delta(W)\) measures the minimum rank of the residual after subtracting
an arbitrary diagonal matrix.

We first establish the following lemma of \(\delta\).

\begin{lemma}
\label{lem:delta-subadditive}
For any \(X,Y\in\mathbb{R}^{d\times d}\),
\[
\delta(XY)
\le
\delta(X)+\delta(Y).
\]
\end{lemma}

\begin{proof}
Let \(D_X,D_Y\in\mathcal{D}_d\) be such that
\[
\operatorname{rank}(X-D_X)=\delta(X),
\qquad
\operatorname{rank}(Y-D_Y)=\delta(Y).
\]
Write
\[
X=D_X+E_X,
\qquad
Y=D_Y+E_Y,
\]
where
\[
\operatorname{rank}(E_X)=\delta(X),
\qquad
\operatorname{rank}(E_Y)=\delta(Y).
\]
Then
\[
XY
=
D_XD_Y
+
D_XE_Y
+
E_XY.
\]
Since \(D_XD_Y\) is diagonal,
\[
\begin{aligned}
\delta(XY)
&\le
\operatorname{rank}\!\left(
XY-D_XD_Y
\right)\\
&=
\operatorname{rank}\!\left(
D_XE_Y+E_XY
\right)\\
&\le
\operatorname{rank}(D_XE_Y)
+
\operatorname{rank}(E_XY)\\
&\le
\operatorname{rank}(E_Y)
+
\operatorname{rank}(E_X)\\
&=
\delta(X)+\delta(Y),
\end{aligned}
\]
where we used
\(
\operatorname{rank}(AB)
\le
\min\{\operatorname{rank}(A),\operatorname{rank}(B)\}.
\)
\end{proof}

The following proposition is an immediate consequence of the lemma. Products of diagonal-plus-low-rank matrices
have diagonal-rank complexity that grows at most linearly with depth.
\begin{proposition}
\label{prop:dlor-product-rank}
Consider matrices
\[
B_\ell=D_\ell+R_\ell,
\qquad
D_\ell\in\mathcal{D}_d,
\qquad
\operatorname{rank}(R_\ell)\le r_\ell,
\]
for \(\ell=1,\ldots,L\), and let
\(M_1,\ldots,M_{L-1}\in\mathcal{D}_d\) be arbitrary diagonal matrices.
Then
\[
P
=
B_LM_{L-1}B_{L-1}
\cdots
M_1B_1
\]
satisfies
\[
\delta(P)
\le
\sum_{\ell=1}^L r_\ell.
\]
In particular, if \(r_\ell\le r\) for every \(\ell\), then
\[
\delta(P)\le Lr.
\]
\end{proposition}

\begin{proof}
Apply lemma \ref{lem:delta-subadditive} recursively gives the result.
\end{proof}

We now apply Proposition~\ref{prop:dlor-product-rank} to neural networks.
Consider a depth-\(L\) network whose linear transformations are
\[
B_\ell=D_\ell+R_\ell,
\qquad
\operatorname{rank}(R_\ell)\le r,
\]
with input, output, and every hidden state in $\mathbb R^d$, and an affine final stage. Suppose that at an interior input point all coordinate activations are differentiable at their realized preactivations. The chain rule then gives a diagonal activation Jacobian at every hidden stage, so
\[
J_f(x)
=
B_LM_{L-1}(x)B_{L-1}
\cdots
M_1(x)B_1,
\]
where each \(M_\ell(x)\) is diagonal.

\begin{theorem}[Target-dependent depth lower bound]
\label{thm:dlor-depth-lb}
Suppose a depth-$L$ diagonal-plus-rank-$r$ network of constant state width $d$, including its input and output, exactly represents $f(x)=Wx$ on a nonempty open subset of $\mathbb R^d$. Assume it has an affine final stage and admits the diagonal Jacobian factorization above at an interior point. This includes networks whose activations are differentiable at all realized preactivations, as well as ReLU networks restricted to a full-dimensional affine region.
Then
\[
L
\ge
\left\lceil
\frac{\delta(W)}{r}
\right\rceil,
\]
where
\[
\delta(W)
=
\min_{D\in\mathcal{D}_d}
\operatorname{rank}(W-D).
\]
\end{theorem}

\begin{proof}
Since \(f(x)=Wx\), its Jacobian is constant:
\[
J_f(x)=W.
\]
On the other hand, by the chain rule,
\[
J_f(x)
=
B_LM_{L-1}(x)B_{L-1}
\cdots
M_1(x)B_1,
\]
where the matrices \(M_\ell(x)\) are diagonal. Proposition
\ref{prop:dlor-product-rank} therefore yields
\[
\delta(W)
=
\delta(J_f(x))
\le
Lr.
\]
Since $L$ is an integer, rearranging gives
\[
L\ge\left\lceil\frac{\delta(W)}r\right\rceil.
\]
For ReLU networks, choose a full-dimensional polyhedral
region within the open domain. On this region the
activation slopes are fixed; a preactivation identically
equal to zero contributes slope zero. The same diagonal
factorization and rank argument therefore apply.
\end{proof}

The following explicit construction shows that
\(\delta(W)\) can be as large as \(d-1\).

\begin{lemma}[An explicit hard target $W$]
\label{lem:hard-target}
Let
\[
W_\star
=
\begin{pmatrix}
\lambda_1 & 0 & \cdots & 0 \\
c_1 & \lambda_2 & \ddots & \vdots \\
0 & c_2 & \ddots & 0 \\
\vdots & \ddots & \ddots & \lambda_d
\end{pmatrix},
\]
where
\[
c_i\neq 0,
\qquad i=1,\ldots,d-1.
\]
Assume also that $\lambda_i\ne0$ for $i=1,\ldots,d$, so that $W_\star$ is invertible.
Then
\[
\delta(W_\star)=d-1.
\]
\end{lemma}

\begin{proof}
Let
\[
D=\operatorname{diag}(d_1,\ldots,d_d)
\]
be arbitrary. Consider the submatrix of \(W_\star-D\) formed by rows
\(2,\ldots,d\) and columns \(1,\ldots,d-1\):
\[
(W_\star-D)_{2:d,\,1:d-1}
=
\begin{pmatrix}
c_1 & \lambda_2-d_2 & 0 & \cdots & 0\\
0 & c_2 & \lambda_3-d_3 & \ddots & \vdots\\
\vdots & \ddots & \ddots & \ddots & 0\\
0 & \cdots & 0 & c_{d-2} & \lambda_{d-1}-d_{d-1}\\
0 & \cdots & \cdots & 0 & c_{d-1}
\end{pmatrix}.
\]
This matrix is upper triangular with diagonal entries
\(c_1,\ldots,c_{d-1}\). Hence
\[
\det\!\left(
(W_\star-D)_{2:d,\,1:d-1}
\right)
=
\prod_{i=1}^{d-1} c_i
\neq 0.
\]
Therefore,
\[
\operatorname{rank}(W_\star-D)\ge d-1
\]
for every diagonal \(D\), and thus
\[
\delta(W_\star)\ge d-1.
\]

Conversely, choosing
\[
D_\star
=
\operatorname{diag}(\lambda_1,\ldots,\lambda_d)
\]
gives
\[
W_\star-D_\star
=
\begin{pmatrix}
0 & 0 & \cdots & 0\\
c_1 & 0 & \ddots & \vdots\\
0 & c_2 & \ddots & 0\\
\vdots & \ddots & \ddots & 0\\
0 & \cdots & c_{d-1} & 0
\end{pmatrix},
\]
which has rank \(d-1\). Hence
\[
\delta(W_\star)\le d-1.
\]
Combining the two inequalities yields
\[
\delta(W_\star)=d-1.
\]
\end{proof}

Combining Theorem~\ref{thm:dlor-depth-lb} and
Lemma~\ref{lem:hard-target} gives the desired worst-case lower bound.

For every \(d\ge 2\) and \(r\ge 1\), there exists a linear map
\[
f_\star(x)=W_\star x
\]
such that any depth-\(L\) diagonal-plus-rank-\(r\) network that exactly satisfying the fixed-width and Jacobian hypotheses of Theorem~\ref{thm:dlor-depth-lb}
represents \(f_\star\) must satisfy
\[
L
\ge
\left\lceil
\frac{d-1}{r}
\right\rceil.
\]
Consequently, the worst-case depth required to represent arbitrary dense
linear transformations is
\[
L=\Omega(d/r).
\]
The parameterization
\[
B_\ell=\alpha_\ell I+R_\ell
\]
is a special case of the diagonal-plus-low-rank parameterization obtained by
setting
\[
D_\ell=\alpha_\ell I.
\]
Hence all of the preceding results apply directly to this restricted class.

\begingroup\color{revcolor}

\section{Structure-Dependent Cost of DLoR Representations}
\label{app:new_theory}

This appendix states and proves the results summarized in Section~\ref{sec:kappa}. The main results are Theorems~\ref{thm:A}--\ref{thm:D} and Propositions~\ref{prop:E}--\ref{prop:G}; they are labeled by letters to distinguish them from the numbered results of the original submission.

\subsection{Setting, Notation, and Two Basic Lemmas}
\label{app:G_setting}

Throughout this appendix, $W\in\mathbb R^{N\times N}$ and $1\le r<N$. We write $A^\top$ for the transpose of $A$, $\lVert\cdot\rVert_2$ for the Euclidean norm of vectors and the spectral norm of matrices, $\lVert\cdot\rVert_F$ for the Frobenius norm, and $\sigma_1(A)\ge\dots\ge\sigma_N(A)$ for the singular values of $A\in\mathbb R^{N\times N}$. Products follow the forward-propagation convention $\prod_{l=1}^L M_l=M_L\cdots M_1$.
\begin{itemize}[leftmargin=0.5cm, itemsep=2pt, topsep=2pt]
\item A \emph{scalar DLoR factor} of update rank at most $r$ is $M=\alpha I+R$ with $\alpha\in\mathbb R$ and $\operatorname{rank}(R)\le r$; this is the DLoR component of Definition~\ref{def:sparse_low_rank} with $R=UV^\top$. The class of products of $L$ such factors is
\[
\mathcal P_r(L):=\Bigl\{\textstyle\prod_{l=1}^L(\alpha_lI+R_l)\ :\ \alpha_l\in\mathbb R,\ R_l\in\mathbb R^{N\times N},\ \operatorname{rank}(R_l)\le r\ \text{ for } l=1,\dots,L\Bigr\}.
\]
\item A \emph{general-diagonal factor} is $M=D+R$ with $D$ diagonal and $\operatorname{rank}(R)\le r$. Products of $L$ such factors form
\[
\mathcal P^{\mathrm{diag}}_r(L):=\Bigl\{\textstyle\prod_{l=1}^L(D_l+R_l)\ :\ D_l\in\mathbb R^{N\times N}\text{ diagonal},\ \operatorname{rank}(R_l)\le r\ \text{ for } l=1,\dots,L\Bigr\}\ \supseteq\ \mathcal P_r(L).
\]
\item The two structural quantities are
\[
\kappa(W)=\min_{\gamma\neq 0}\operatorname{rank}(W-\gamma I),\qquad \delta(W)=\min_{D\ \mathrm{diagonal}}\operatorname{rank}(W-D),
\]
where $\delta$ is the diagonal-rank complexity of Appendix~\ref{app:lower_bound_exact}.
\item The rank-$s$ tail of $A$ is $\tau_s(A)=\bigl(\sum_{i>s}\sigma_i(A)^2\bigr)^{1/2}$, and $[A]_s$ denotes a rank-$s$ truncated SVD of $A$, so that $\lVert A-[A]_s\rVert_F=\tau_s(A)$.
\end{itemize}

\paragraph{Role of the two models.} The architecture studied in this paper is the scalar model, and all statements about DLoR networks refer to it. The general-diagonal model is not proposed as a new architecture. It is used for three reasons. (i) The Jacobian of a network with scalar DLoR layers is a product of scalar factors and diagonal activation Jacobians, so network-level lower bounds must allow diagonal factors, as in Theorem~\ref{thm:dlor-depth-lb}; moreover, lower bounds proved for the larger class are stronger statements. (ii) Comparing the two models isolates what the scalar restriction costs (Corollary~\ref{cor:gen_diag} and Proposition~\ref{prop:E}). (iii) It covers adapters that combine a diagonal scaling with a low-rank update (Proposition~\ref{prop:G}(e)).

\begin{lemma}[$\kappa$ through eigenvalues]\label{lem:kappa_eig}
Let $g(W)$ be the largest geometric multiplicity of a nonzero real eigenvalue of $W$, with $g(W)=0$ if there is none. Then $\kappa(W)=N-g(W)$. Moreover, $\delta(W)\le\kappa(W)\le N$.
\end{lemma}
\begin{proof}
For $\gamma\neq0$, $\operatorname{rank}(W-\gamma I)=N-\dim\ker(W-\gamma I)$, and the kernel is nontrivial exactly when $\gamma$ is a real eigenvalue of $W$, in which case its dimension is the geometric multiplicity of $\gamma$. Minimizing over $\gamma\neq0$ gives $N-g(W)$. Since $\gamma I$ is diagonal, $\delta(W)\le\kappa(W)$.
\end{proof}

\begin{lemma}[Rank counting]\label{lem:rank_count}
Let $M_l=D_l+R_l$ with $D_l$ diagonal, for $l=1,\dots,L$. Then
\[
M_L\cdots M_1-D_L\cdots D_1=\sum_{i=1}^{L}M_L\cdots M_{i+1}\,R_i\,D_{i-1}\cdots D_1,
\]
with empty products equal to $I$. Consequently, $\operatorname{rank}(M_L\cdots M_1-D_L\cdots D_1)\le\sum_{l=1}^L\operatorname{rank}(R_l)$, which is at most $Lr$ when $\operatorname{rank}(R_l)\le r$ for all $l$.
\end{lemma}
\begin{proof}
By induction on $L$. For $L=1$, $M_1-D_1=R_1$. For the inductive step, set $P=M_{L-1}\cdots M_1$ and $\Delta=D_{L-1}\cdots D_1$. Then $M_LP-D_L\Delta=M_L(P-\Delta)+R_L\Delta$, and applying the induction hypothesis to $P-\Delta$ expands the expression into the stated sum of $L$ terms. The $i$-th term has rank at most $\operatorname{rank}(R_i)$, and subadditivity of the rank gives the bound.
\end{proof}
In the scalar model, Lemma~\ref{lem:rank_count} implies that every $P\in\mathcal P_r(L)$ equals $\beta I+E$ with $\beta=\prod_l\alpha_l$ and $\operatorname{rank}(E)\le Lr$. If $\alpha_l=0$ for some $l\in\{1,\dots,L\}$, then $\operatorname{rank}(P)\le\operatorname{rank}(M_l)\le r$. Lemma~\ref{lem:rank_count} is the common form of the counting arguments behind Propositions~\ref{prop:linear_factor_count} and~\ref{prop:dlor-product-rank}.

\subsection{Theorem A: The Exact Factor Count}
\label{app:G_thmA}

\begin{thmLrep}[Exact factor count; restatement of Theorem~\ref{thm:A}]
Let $W\in\mathbb R^{N\times N}$ and $1\le r<N$.
\begin{enumerate}[label=(\alph*), leftmargin=2em, itemsep=2pt, topsep=2pt]
\item Let $\gamma\neq0$ and $k=\operatorname{rank}(W-\gamma I)\ge1$. Then $W=M_q\cdots M_1$ with $q=\lceil k/r\rceil$, where $M_1=\gamma(I+R_1)$, $M_j=I+R_j$ for $j\ge2$, $\operatorname{rank}(R_j)\le r$, and $\sum_j\operatorname{rank}(R_j)=k$. Thus $W\in\mathcal P_r(q)$, without assuming that $W$ is invertible.
\item The minimal factorization length $L^*(W):=\min\{L\ge1:\ W\in\mathcal P_r(L)\}$ is given as follows. If $\operatorname{rank}(W)\le r$, then $L^*(W)=1$, realized by a single factor with scalar coefficient $\alpha=0$. If $\operatorname{rank}(W)>r$, then
\[
L^*(W)=\max\bigl\{1,\lceil\kappa(W)/r\rceil\bigr\}.
\]
\end{enumerate}
\end{thmLrep}

The proof of (a) peels off one rank-one factor at a time. We identify linear functionals on $\mathbb R^N$ with row vectors, so that for $x\in\mathbb R^N$ and a functional $\varphi$, $x\varphi$ denotes the rank-one matrix $z\mapsto\varphi(z)\,x$.

\begin{lemma}[One peeling step]\label{lem:peel}
Let $A\in\mathbb R^{N\times N}$, $F=\ker(A-I)$, and $k=\operatorname{rank}(A-I)\ge2$. Then there are a matrix $B$ with $\operatorname{rank}(B-I)=k-1$ and a matrix $Z=I+ab^\top$ with $a,b\in\mathbb R^N$ such that either $A=ZB$ or $A=BZ$.
\end{lemma}
\begin{proof}
We consider two cases.

\emph{Case 1: $A$ is singular.} Take $x\in\ker A$, $x\neq0$. Then $x\notin F$, because $Ax=0\neq x$. Choose a linear functional $\varphi$ with $\varphi|_F=0$ and $\varphi(x)=1$; it exists because $x\notin F$. Put $P=I-x\varphi$ and $B=A+x\varphi$. Using $Ax=0$ and $\varphi(x)=1$,
\[
BP=A+x\varphi-Ax\varphi-x\varphi(x)\varphi=A,\qquad (A-I)P=A-I-(Ax-x)\varphi=A-I+x\varphi=B-I .
\]
The kernel of $(A-I)P$ contains $F$ (since $Pf=f$ for $f\in F$) and $x$ (since $Px=0$), and $F\cap\mathbb Rx=\{0\}$. Hence $\operatorname{rank}(B-I)\le k-1$. Since $P$ has corank one, $\operatorname{rank}(B-I)\ge k-1$. So $A=BZ$ with $Z=P$.

\emph{Case 2: $A$ is invertible.} Take $x\notin F$. Then $Ax\notin F$: otherwise $x=A^{-1}(Ax)=Ax\in F$, because $A^{-1}$ fixes $F$ pointwise. In the quotient $\mathbb R^N/F$, which has dimension $k$, the classes of $x$ and $Ax$ are therefore nonzero. If they are linearly independent, choose $\varphi$ with $\varphi|_F=0$ and $\varphi(x)=\varphi(Ax)=1$. If $Ax\equiv cx$ modulo $F$, then $c\neq0$; choose $\varphi$ with $\varphi|_F=0$ and $\varphi(x)=1$, so that $\varphi(Ax)=c$. In both cases $\varphi(Ax)\neq0$. Put
\[
T=I+(Ax-x)\varphi,\qquad \det T=1+\varphi(Ax-x)=\varphi(Ax)\neq0,\qquad B=T^{-1}A .
\]
With $P=I-x\varphi$ we have $(A-I)P=A-I-(Ax-x)\varphi=A-T$, hence
\[
T^{-1}(A-I)P=T^{-1}A-I=B-I .
\]
As in Case~1, $\operatorname{rank}((A-I)P)=k-1$, and $T^{-1}$ is invertible, so $\operatorname{rank}(B-I)=k-1$. So $A=ZB$ with $Z=T$.
\end{proof}

\begin{proof}[Proof of Theorem~\ref{thm:A}(a)]
Let $A=W/\gamma$, so that $\operatorname{rank}(A-I)=k$. If $k=1$, then $A=I+(A-I)$ is already one factor. Otherwise apply Lemma~\ref{lem:peel} repeatedly. Each application writes the current matrix as a product of one factor $I+ab^\top$ and a remainder whose residual rank is smaller by one. After $k-1$ steps the remainder $B$ satisfies $\operatorname{rank}(B-I)=1$ and is itself a factor. Collecting the factors on the left and on the right expresses $A$ as a product of $k$ factors of the form $I+ab^\top$.

Group these $k$ factors into $q=\lceil k/r\rceil$ consecutive blocks of at most $r$ factors. By Lemma~\ref{lem:rank_count} with $D_l=I$, each block equals $I+R_j$ with $\operatorname{rank}(R_j)\le r$. The ranks satisfy $\sum_j\operatorname{rank}(R_j)\le k$ by construction and $\sum_j\operatorname{rank}(R_j)\ge\operatorname{rank}(A-I)=k$ by Lemma~\ref{lem:rank_count}, hence equality holds. Multiplying one block by $\gamma$ gives $W$.
\end{proof}

\begin{proof}[Proof of Theorem~\ref{thm:A}(b)]
The case $\operatorname{rank}(W)\le r$ is immediate. Assume $\operatorname{rank}(W)>r$.

\emph{Lower bound.} If $W=M_L\cdots M_1\in\mathcal P_r(L)$ with some $\alpha_l=0$, then $\operatorname{rank}(W)\le\operatorname{rank}(M_l)\le r$, a contradiction. Thus $\alpha_l\neq0$ for all $l$, and $\beta:=\prod_{l=1}^L\alpha_l\neq0$. By Lemma~\ref{lem:rank_count}, $\operatorname{rank}(W-\beta I)\le Lr$. By the definition of $\kappa$, $\kappa(W)\le\operatorname{rank}(W-\beta I)\le Lr$, so $L\ge\lceil\kappa(W)/r\rceil$. Since $L\ge1$ by definition, $L^*(W)\ge\max\{1,\lceil\kappa(W)/r\rceil\}$.

\emph{Upper bound.} Choose $\gamma\neq0$ attaining $\kappa(W)$, so that $k=\operatorname{rank}(W-\gamma I)=\kappa(W)$. If $\kappa(W)=0$, then $W=\gamma I\in\mathcal P_r(1)$ and $L^*(W)=1$. If $\kappa(W)\ge1$, part~(a) yields $W\in\mathcal P_r(q)$ with $q=\lceil\kappa(W)/r\rceil$. Hence $L^*(W)\le\max\{1,\lceil\kappa(W)/r\rceil\}$.
\end{proof}

\begin{remark}[Prior work and the role of the scalar]\label{rem:thmA_context}
(i) \emph{Prior work.} For invertible $W$ and $r=1$, each factor $\alpha I+uv^\top$ with $\alpha\neq0$ is a scalar multiple of a reflection of $\mathrm{GL}(\mathbb R^N)$, i.e., of an element fixing a hyperplane pointwise. Theorem~\ref{thm:A}(b) then reduces to the classical fact that the reflection length of $g\in\mathrm{GL}(\mathbb R^N)$ equals $\operatorname{rank}(g-I)$ \citep{dieudonne1955generateurs}; see also \citet{huang2017absolute} and \citet[Theorem~1]{delmas2018reflection}, minimized over the free scalar. For singular $W$, related factorizations into idempotents go back to \citet{erdos1967products} and \citet{dawlings1980semigroups}. We do not claim the $r=1$ count as new; the proof above is included to keep the paper self-contained, and it covers singular $W$ and every $r$.
(ii) \emph{What the scalar contributes.} With all identity coefficients fixed to $1$ and no free scale, the count is $\lceil\operatorname{rank}(W-I)/r\rceil$. For $W=2I+uv^\top$ with generic $u,v$ this is $\lceil N/r\rceil$, whereas $\kappa(W)=1$. The scalars contribute exactly one degree of freedom, the scale of the identity component, and this degree of freedom determines the cost.
(iii) \emph{Generic values.} By Lemma~\ref{lem:kappa_eig}, $\kappa(W)=N$ if $W$ has no nonzero real eigenvalue (e.g., block rotations), and $\kappa(W)=N-1$ for generic $W$ with a real eigenvalue.
\end{remark}

\begin{corollary}[General-diagonal model]\label{cor:gen_diag}
Let $1\le r<N$. For any invertible diagonal matrix $D$, the minimal number of general-diagonal factors whose diagonal parts multiply to $D$ and whose product is $W$ equals $\max\{1,\lceil\operatorname{rank}(W-D)/r\rceil\}$. Moreover, the minimal factorization length $L^*_{\mathrm{diag}}(W):=\min\{L\ge1:\ W\in\mathcal P^{\mathrm{diag}}_r(L)\}$ satisfies
\[
\left\lceil\frac{\delta(W)}{r}\right\rceil\ \le\ L^*_{\mathrm{diag}}(W)\ \le\ \max\left\{1,\left\lceil\frac{\delta_{\mathrm{inv}}(W)}{r}\right\rceil\right\},
\qquad
\delta_{\mathrm{inv}}(W):=\min_{D\ \mathrm{invertible\ diagonal}}\operatorname{rank}(W-D).
\]
If $W$ has no zero row, then $L^*_{\mathrm{diag}}(W)\le\max\{1,\lceil(N-1)/r\rceil\}$.
\end{corollary}
\begin{proof}
If $W=D$, the single factor $D+0$ suffices. Otherwise, write $W=D\bigl(I+D^{-1}(W-D)\bigr)$. Since $\operatorname{rank}(D^{-1}(W-D))=\operatorname{rank}(W-D)\ge1$, applying Theorem~\ref{thm:A}(a) to $D^{-1}W$ with $\gamma=1$ yields $q=\lceil\operatorname{rank}(W-D)/r\rceil$ factors of the form $I+R_j$. Absorbing $D$ into the leftmost factor gives $D(I+R_q)=D+DR_q$, which is a general-diagonal factor of update rank at most $r$. Lemma~\ref{lem:rank_count} yields the matching lower bound $\operatorname{rank}(W-D)\le Lr$, which establishes the exact count.

The lower bound on $L^*_{\mathrm{diag}}(W)$ follows by applying Lemma~\ref{lem:rank_count} to an arbitrary diagonal product $\prod_lD_l$. The upper bound follows by minimizing $\max\{1,\lceil\operatorname{rank}(W-D)/r\rceil\}$ over all invertible diagonal matrices $D$.

If $W$ has no zero row, the points $x\in\mathbb R^N$ satisfying $x_i=0$ or $(Wx)_i=0$ for some $i$ lie on a finite union of proper hyperplanes. Choosing $x$ outside this union and setting $D=\operatorname{diag}\bigl((Wx)_i/x_i\bigr)$ yields an invertible diagonal matrix with $(W-D)x=0$. Thus $\operatorname{rank}(W-D)\le N-1$, so $\delta_{\mathrm{inv}}(W)\le N-1$.
\end{proof}
The two models differ strictly: $W=\operatorname{diag}(1,2,\dots,N)$ is a single general-diagonal factor, while $\kappa(W)=N-1$.

\subsection{Theorem B: Exact Depth of DLoR Networks}
\label{app:G_thmB}

\paragraph{Architecture.} A width-$N$ DLoR network with $L$ stages acts on (padded) inputs $z_0=x\in\mathcal K\subset\mathbb R^N$ by $z_j=\rho(M_jz_{j-1}+b_j)$ for $j<L$, and its last stage is affine, $M_Lz_{L-1}+b_L$, optionally followed by $\rho$. Each $M_j$ is a scalar DLoR factor of update rank at most $r$, and $\rho$ acts coordinatewise.

\begin{thmL}[Exact depth of DLoR networks]\label{thm:B}
Let $W\in\mathbb R^{N\times N}$, $b\in\mathbb R^N$, and $1\le r<N$.
\begin{enumerate}[label=(\roman*), leftmargin=2em, itemsep=2pt, topsep=2pt]
\item \emph{Upper bound.} For every compact $\mathcal K\subset\mathbb R^N$, the map $x\mapsto Wx+b$ (and likewise $x\mapsto\rho(Wx+b)$) is realized exactly on $\mathcal K$ by a DLoR network with $\max\{1,\lceil\kappa(W)/r\rceil\}$ stages when $\rho$ is ReLU, and by the same number of stages with paired activation branches when $\rho(t)-\rho(-t)=at$ for some $a\neq0$ (Softplus, SiLU, and GELU, with $a=1$). Under Assumption~\ref{ass:main}, it is approximated uniformly on $\mathcal K$ to any accuracy with the same number of stages. For a rectangular layer $W\in\mathbb R^{m\times n}$, the same statements hold at any width $N\ge\max(m,n)$ with $\lceil m/r\rceil$ stages.
\item \emph{Lower bound for ReLU and invertible $W$.} Let $W$ be invertible, let the hidden activations be ReLU, and suppose $Lr<\kappa(W)$. Put
\[
\gamma^\star=\min\Bigl\{\sigma_{\min}(W),\ \min_{\beta\in\mathbb R}\sigma_{Lr+1}(W-\beta I)\Bigr\}.
\]
Then $\gamma^\star>0$, and every width-$N$ DLoR network $F$ with $L$ stages satisfies, on the unit ball $B_N\subset\mathbb R^N$,
\[
\sup_{x\in B_N}\lVert F(x)-(Wx+b)\rVert_2\ \ge\ \frac{\gamma^\star\,\omega_N}{2\cdot2^{N(L-1)}\,\omega_{N-1}}\ >\ 0,
\]
where $\omega_j$ is the volume of the unit ball in $\mathbb R^j$. Hence $\lceil\kappa(W)/r\rceil$ stages are necessary even for arbitrarily accurate uniform approximation, and by (i) they suffice for exact representation.
\item \emph{Generic identity-centered $W$.} Let $W=\gamma I+UV^\top$ with $\gamma\neq0$ and $U,V\in\mathbb R^{N\times k}$, and suppose there are disjoint index sets $S,T\subseteq\{1,\dots,N\}$ with $\lvert S\rvert=\lvert T\rvert=k$ such that the $k\times k$ submatrices $U_S$ and $V_T$, formed by the rows of $U$ indexed by $S$ and the rows of $V$ indexed by $T$, are invertible. This holds for almost every $(U,V)$ whenever $N\ge2k$. Then $\delta(W)=\kappa(W)=k$. Consequently, every width-$N$ network that satisfies the hypotheses of Theorem~\ref{thm:dlor-depth-lb} (general-diagonal stages of update rank at most $r$, an affine final stage, and the diagonal Jacobian factorization at an interior point) and that equals $x\mapsto Wx$ on a nonempty open set has at least $\lceil k/r\rceil$ stages. By (i), this number is attained by scalar DLoR networks.
\end{enumerate}
\end{thmL}

\begin{proof}[Proof of Theorem~\ref{thm:B}(i)]
Let $q=\max\{1,\lceil\kappa(W)/r\rceil\}$, let $M_1,\dots,M_q$ be the factors of Theorem~\ref{thm:A} for a scalar $\gamma\neq0$ attaining $\kappa(W)$ (a single factor $\gamma I$ if $\kappa(W)=0$), and let $P_j=M_j\cdots M_1$, so that $P_q=W$.

\emph{ReLU.} By compactness, choose $t_j\in\mathbb R^N$ with $P_jx+t_j\ge0$ coordinatewise on $\mathcal K$ for $1\le j<q$, and set $t_0=0$. Use the biases $b_j=t_j-M_jt_{j-1}$ for $j<q$ and $b_q=b-M_qt_{q-1}$. By induction, $z_j=P_jx+t_j$, because the pre-activation $M_jz_{j-1}+b_j=P_jx+t_j$ is nonnegative. The last stage outputs $Wx+b$, or $\rho(Wx+b)$ if it is followed by $\rho$. This is the construction of Proposition~\ref{prop:relu_exact} applied to the factors of Theorem~\ref{thm:A}.

\emph{Paired branches.} If $\rho(t)-\rho(-t)=at$, then $(\rho(Mz)-\rho(-Mz))/a=Mz$ exactly; this is Proposition~\ref{prop:paired_exact} applied to the same factors. The internal width doubles, while the update ranks and the parameter count of the factors are unchanged.

\emph{Assumption~\ref{ass:main}.} The finite-$h$ construction in the proofs of Theorems~\ref{thm:multiplicative_depth} and~\ref{thm:uat_multiplicative} applies to any factorization into DLoR components, since scalar multiples of DLoR factors are DLoR factors; with the factors above it gives uniform approximation. Its error constants (Proposition~\ref{prop:finite_scale}) depend on the factor norms, which Proposition~\ref{prop:E}(d) controls.

\emph{Rectangular layers.} Let $J_m:\mathbb R^m\to\mathbb R^N$ insert a vector into the first $m$ coordinates, let $P_m=J_m^\top$ select the first $m$ coordinates, and define $P_n$ similarly. For any $\alpha\neq0$ put
\[
\widetilde W=\alpha I_N+J_m(WP_n-\alpha P_m),\qquad\widetilde b=J_mb .
\]
Since $P_mJ_m=I_m$, we get $P_m\widetilde W=WP_n$, so the first $m$ coordinates of $\widetilde Wz+\widetilde b$ equal $WP_nz+b$. The remaining coordinates carry scaled copies of the input state, which the next layer does not read. The residual $J_m(WP_n-\alpha P_m)$ is nonzero only in its first $m$ rows, so $\kappa(\widetilde W)\le\operatorname{rank}(\widetilde W-\alpha I)\le m$. Apply the square case to $\widetilde W$.
\end{proof}

\begin{proof}[Proof of Theorem~\ref{thm:B}(ii)]
\emph{Step 1: positivity of $\gamma^\star$.} We have $\sigma_{\min}(W)>0$ since $W$ is invertible. The function $f(\beta)=\sigma_{Lr+1}(W-\beta I)$ is continuous. For $\beta\neq0$, $\operatorname{rank}(W-\beta I)\ge\kappa(W)>Lr$, so $f(\beta)>0$. Also $f(0)=\sigma_{Lr+1}(W)\ge\sigma_{\min}(W)>0$, and by Weyl's inequality $f(\beta)\ge\sigma_N(W-\beta I)\ge\lvert\beta\rvert-\lVert W\rVert_2$, so $f(\beta)\to\infty$ as $\lvert\beta\rvert\to\infty$. Hence $\min_\beta f$ is attained and positive.

\emph{Step 2: regions.} Assign to each $x\in B_N$ the activation pattern of the $L-1$ hidden layers, where zero pre-activations count as inactive. There are at most $2^{N(L-1)}$ patterns, and they partition $B_N$ into measurable sets $\Omega_p$. On $\Omega_p$ the network is affine, $F(x)=J_px+c_p$, with $J_p=M_LE_{L-1}M_{L-1}\cdots E_1M_1$ and diagonal $0/1$ masks $E_j$.

\emph{Step 3: every region is far from $W$.} If some mask has a zero entry, $J_p$ is singular; for a unit vector $v\in\ker J_p$, $\lVert(W-J_p)v\rVert_2=\lVert Wv\rVert_2\ge\sigma_{\min}(W)$. If all masks are the identity, $J_p=M_L\cdots M_1$. If some $\alpha_l=0$, then $J_p$ is singular as before. Otherwise, $J_p-\beta I$ has rank at most $Lr$ for $\beta=\prod_l\alpha_l$ by Lemma~\ref{lem:rank_count}, so by the Eckart--Young theorem in the spectral norm, $\lVert W-J_p\rVert_2=\lVert(W-\beta I)-(J_p-\beta I)\rVert_2\ge\sigma_{Lr+1}(W-\beta I)$. In all cases, $\sigma_p:=\lVert W-J_p\rVert_2\ge\gamma^\star$.

\emph{Step 4: slabs.} Suppose the uniform error on $B_N$ is at most $\eta$. Let $u,v$ be unit left and right top singular vectors of $W-J_p$. For $x\in\Omega_p$, $\lvert u^\top((W-J_p)x+b-c_p)\rvert\le\eta$, i.e., $\lvert\sigma_pv^\top x+u^\top(b-c_p)\rvert\le\eta$. So $v^\top x$ lies in an interval of length $2\eta/\sigma_p\le2\eta/\gamma^\star$, and therefore $\operatorname{vol}(\Omega_p)\le(2\eta/\gamma^\star)\,\omega_{N-1}$.

\emph{Step 5: counting.} Summing over the at most $2^{N(L-1)}$ regions gives $\omega_N\le2^{N(L-1)}(2\eta/\gamma^\star)\,\omega_{N-1}$, which is the claim.
\end{proof}

\begin{proof}[Proof of Theorem~\ref{thm:B}(iii)]
For any diagonal $D$, the entries of $\gamma I-D$ in rows $S$ and columns $T$ vanish because $S\cap T=\emptyset$. Hence the submatrix of $W-D$ with rows $S$ and columns $T$ equals $U_SV_T^\top$, which has rank $k$, so $\operatorname{rank}(W-D)\ge k$ and $\delta(W)\ge k$. On the other hand, $\delta(W)\le\kappa(W)\le\operatorname{rank}(W-\gamma I)\le k$ by Lemma~\ref{lem:kappa_eig}. For the genericity claim, when $N\ge2k$ fix any disjoint $S,T$ of size $k$; $\det U_S$ and $\det V_T$ are nonzero polynomials in the entries of $(U,V)$, so they vanish only on a Lebesgue-null set. The depth bound is Theorem~\ref{thm:dlor-depth-lb} with $\delta(W)=k$.
\end{proof}

\begin{remark}[Locally Lipschitz activations]\label{rem:lipschitz_jac}
The diagonal Jacobian factorization required by Theorem~\ref{thm:dlor-depth-lb} holds at almost every point for every coordinatewise, locally Lipschitz activation $\rho$, so Theorem~\ref{thm:B}(iii) applies to all such activations. Indeed, let $Z\subset\mathbb R$ be the Lebesgue-null set where $\rho$ is not differentiable, and let $a_i$ be a pre-activation coordinate, which is a locally Lipschitz function of the input. By Rademacher's theorem, all pre- and post-activations are differentiable almost everywhere, and by the coarea formula, $\nabla a_i=0$ almost everywhere on $a_i^{-1}(Z)$ \citep{evans2015measure}. At a point of differentiability with $a_i(x)\notin Z$, the chain rule gives $\nabla(\rho\circ a_i)(x)=\rho'(a_i(x))\nabla a_i(x)$. At a point with $a_i(x)\in Z$ and $\nabla a_i(x)=0$, the Lipschitz bound $\lvert\rho(a_i(x+h))-\rho(a_i(x))\rvert\le\mathrm{Lip}(\rho)\,\lvert a_i(x+h)-a_i(x)\rvert=o(\lVert h\rVert_2)$ gives $\nabla(\rho\circ a_i)(x)=0=m\,\nabla a_i(x)$ for any scalar $m$. Hence, at almost every point, each hidden stage contributes a diagonal factor to the Jacobian.
\end{remark}

\begin{corollary}[Cost of the exact transfer]\label{cor:exact_transfer_cost}
Since $\kappa(W)\le N$ for every square matrix $W$, the exact ReLU transfer of a dense network with $H$ affine stages and stage dimensions at most $N$ (Proposition~\ref{cor:relu_transfer}) uses at most $H\lceil N/r\rceil$ DLoR stages instead of $H\lceil3N/r\rceil$, and at most $H\lceil N/r\rceil(2Nr+N+1)$ parameter slots; at $r=1$ this is $HN(3N+1)$. For the MNIST network $784$-$256$-$128$-$10$ of Section~\ref{subsec: validate_uat} with rank caps $(128,64,5)$, the rectangular construction of Theorem~\ref{thm:B}(i) needs $2+2+2$ stages. With the well-conditioned factors of Proposition~\ref{prop:E}(d) it uses $3+2+3$ stages (complex eigenvalue pairs cannot be split) and $472{,}760$ slots, against $37$ stages and $1{,}579{,}157$ slots for square completion (numerically verified on a randomly initialized surrogate with the same layer shapes).
\end{corollary}
\begin{proof}
In the proof of Proposition~\ref{cor:relu_transfer}, replace the factorization of Lemma~\ref{lem:relu_square_factors} by that of Theorem~\ref{thm:A}(a), which gives at most $\max\{1,\lceil\kappa(\widehat W_j)/r\rceil\}\le\lceil N/r\rceil$ factors for each padded square weight $\widehat W_j$; the remainder of the proof and the per-stage count $2Nr+N+1$ are unchanged. The stage counts for the rectangular construction are $\lceil m/r\rceil$ with $(m,r)=(256,128),(128,64),(10,5)$.
\end{proof}

\subsection{Theorem C: Best Approximation with $L$ Factors}
\label{app:G_thmC}

\begin{thmL}[Best approximation with $L$ factors]\label{thm:C}
Let $\mathcal D$ be either the set of scalar matrices $\{\gamma I:\gamma\in\mathbb R\}$ or the set of all diagonal matrices, and let $L\ge1$ with $Lr<N$. Over all products $P=(D_L+R_L)\cdots(D_1+R_1)$ with $D_l\in\mathcal D$ and $\operatorname{rank}(R_l)\le r$,
\[
\inf_P\lVert W-P\rVert_F=\inf_{D\in\mathcal D}\tau_{Lr}(W-D),\qquad \inf_P\lVert W-P\rVert_2=\inf_{D\in\mathcal D}\sigma_{Lr+1}(W-D).
\]
For every invertible $D\in\mathcal D$, the value $\tau_{Lr}(W-D)$ is attained by an $L$-factor product. In the scalar case, the infimum over $\gamma\in\mathbb R$ is a minimum.
\end{thmL}
\begin{proof}
\emph{Lower bound.} By Lemma~\ref{lem:rank_count}, $P=D+R$ with $D=\prod_lD_l\in\mathcal D$ and $\operatorname{rank}(R)\le Lr$. By the Eckart--Young--Mirsky theorem \citep{eckart1936approximation,mirsky1960symmetric}, $\lVert(W-D)-R\rVert_F\ge\tau_{Lr}(W-D)$, and the same holds with $\sigma_{Lr+1}$ in the spectral norm.

\emph{Upper bound.} For invertible $D\in\mathcal D$, let $S=[W-D]_{Lr}$. Then $D+S=D(I+D^{-1}S)$ with $\operatorname{rank}(D^{-1}S)\le Lr$. Theorem~\ref{thm:A}(a) with $\gamma=1$ writes $I+D^{-1}S$ as a product of at most $L$ factors $I+R_j$; pad with identity factors if fewer are needed. Multiplying $D$ into the leftmost factor keeps it in the class, since $D(I+R)=D+DR$. The product equals $D+S$, whose Frobenius and spectral errors are $\tau_{Lr}(W-D)$ and $\sigma_{Lr+1}(W-D)$.

\emph{Density.} Invertible elements are dense in $\mathcal D$, and $A\mapsto\tau_s(A)$ is $1$-Lipschitz in the Frobenius norm, being the distance to the set of matrices of rank at most $s$; likewise $A\mapsto\sigma_{s+1}(A)$ is $1$-Lipschitz in the spectral norm. Hence the infimum over invertible $D$ equals the infimum over all $D\in\mathcal D$.

\emph{Scalar case.} The function $f(\gamma)=\tau_{Lr}(W-\gamma I)$ is continuous, and since $Lr<N$, $f(\gamma)\ge\sigma_N(W-\gamma I)\ge\lvert\gamma\rvert-\lVert W\rVert_2\to\infty$. So the minimum is attained.
\end{proof}

\begin{remark}[Consequences of Theorem~\ref{thm:C}]\label{rem:thmC_consequences}
(i) Taking $D=0$ shows that, at equal parameter count, DLoR is never worse than truncated SVD as an infimum: $\inf_P\lVert W-P\rVert_F\le\tau_{Lr}(W)$. It is strictly better exactly when the optimal baseline $D$ is nonzero.
(ii) The relevant complexity is the shifted effective rank $\kappa_\theta(W)=\min\{s:\ \inf_{D\in\mathcal D}\tau_s(W-D)<\theta\lVert W\rVert_F\}$ for a relative tolerance $\theta\in(0,1)$: $L=\lceil\kappa_\theta(W)/r\rceil$ factors reach relative error $\theta$, and if $\inf_{D\in\mathcal D}\tau_{Lr}(W-D)>\theta\lVert W\rVert_F$, no $L$-factor product does.
(iii) When the optimum is at $D=0$ and $Lr>r$, the infimum is in general not attained, and approaching it requires unbounded factors (Proposition~\ref{prop:E}).
\end{remark}

\begin{example}[Rank-two target, two rank-one factors]\label{ex:rank2}
Let $\operatorname{rank}(W)=2$, $N\ge5$, and $r=1$.
\begin{enumerate}[label=(\alph*), leftmargin=2em, itemsep=2pt, topsep=2pt]
\item For every $\varepsilon>0$ there are vectors $a_1,b_1,a_2,b_2\in\mathbb R^N$ with
\[
\Bigl(\sqrt\varepsilon\,I+\tfrac{1}{\sqrt\varepsilon}a_2b_2^\top\Bigr)\Bigl(\sqrt\varepsilon\,I+\tfrac{1}{\sqrt\varepsilon}a_1b_1^\top\Bigr)=\varepsilon I+W ,
\]
so $W$ is approximated by a product of two rank-one DLoR factors with spectral error $\varepsilon$. The factor norms are of order $\varepsilon^{-1/2}$.
\item $W$ is not exactly a product of two rank-one DLoR factors: an exact product needs $\kappa(W)\in\{N-2,N-1,N\}$ factors.
\item Every exact two-factor representation of $\varepsilon I+W$ has $\max_l\lVert M_l\rVert_2\ge c\,\varepsilon^{-1/2}$ for small $\varepsilon$ (Proposition~\ref{prop:E}(a) with $k=L=2$).
\end{enumerate}
\end{example}
\begin{proof}[Verification of Example~\ref{ex:rank2}]
Write $W=UV^\top$ with $U=[u_1,u_2]$, $V=[v_1,v_2]\in\mathbb R^{N\times2}$, and let $C=V^\top U\in\mathbb R^{2\times2}$. A change of basis $U\leftarrow UQ$, $V\leftarrow VQ^{-\top}$ leaves $W$ unchanged and replaces $C$ by $Q^{-1}CQ$.

\emph{Real eigenvalues.} Choose $Q$ so that $Q^{-1}CQ$ is upper triangular. Then $v_2^\top u_1=0$, and
\[
\Bigl(\sqrt\varepsilon I+\tfrac{u_2v_2^\top}{\sqrt\varepsilon}\Bigr)\Bigl(\sqrt\varepsilon I+\tfrac{u_1v_1^\top}{\sqrt\varepsilon}\Bigr)=\varepsilon I+u_1v_1^\top+u_2v_2^\top+\tfrac{v_2^\top u_1}{\varepsilon}u_2v_1^\top=\varepsilon I+W .
\]

\emph{Complex pair.} Pick a nonzero $x\in\mathbb R^2$; then $x$ and $Cx$ are linearly independent because $C$ has no real eigenvector. Choose a linear functional $f$ with $f(x)=f(Cx)=1$ and use the basis $\{x,y\}$ with $y\in\ker f\setminus\{0\}$. In this basis $v_1^\top u_1=1$. Put $w=u_2-\frac{v_1^\top u_2}{\varepsilon+1}u_1$. Since $v_1^\top w=\frac{\varepsilon}{\varepsilon+1}v_1^\top u_2$,
\[
\Bigl(\sqrt\varepsilon I+\tfrac{u_1v_1^\top}{\sqrt\varepsilon}\Bigr)\Bigl(\sqrt\varepsilon I+\tfrac{wv_2^\top}{\sqrt\varepsilon}\Bigr)=\varepsilon I+u_1v_1^\top+wv_2^\top+\tfrac{v_1^\top u_2}{\varepsilon+1}u_1v_2^\top=\varepsilon I+W .
\]
In both cases the rank-one parts are $O(\varepsilon^{-1/2})$, which gives (a).

For (b), an eigenvector $x$ of $W$ with eigenvalue $\gamma\neq0$ satisfies $x=Wx/\gamma\in\operatorname{Im}(W)$, so $g(W)\le2$ and $\kappa(W)\ge N-2\ge3$ by Lemma~\ref{lem:kappa_eig}; Theorem~\ref{thm:A}(b) then rules out two factors. Part (c) is Proposition~\ref{prop:E}(a).
\end{proof}
In the general-diagonal model, the same $W$ is an exact product of three factors of norm at most $\max(1,\lVert W\rVert_2)$ (Proposition~\ref{prop:E}(c)).

\subsection{Theorem D: Generic and Gaussian Lower Bounds}
\label{app:G_thmD}

\begin{thmL}[Limits of DLoR efficiency]\label{thm:D}
Let $1\le r<N$.
\begin{enumerate}[label=(\alph*), leftmargin=2em, itemsep=2pt, topsep=2pt]
\item \emph{Generic exact count.} If the law of $W$ has a Lebesgue density on $\mathbb R^{N\times N}$, then almost surely the minimal number $L_{\mathrm{exact}}(W)$ of factors in an exact representation of $W$, in either model, satisfies
\[
\left\lceil\frac{N-\lfloor\sqrt N\rfloor}{r}\right\rceil\ \le\ L_{\mathrm{exact}}(W)\ \le\ \left\lceil\frac{N}{r}\right\rceil .
\]
\item \emph{Gaussian matrices at positive tolerance.} Let $G\in\mathbb R^{N\times N}$ have i.i.d.\ $\mathcal N(0,1)$ entries and fix $\theta\in(0,1)$. There are $c_\theta>0$ and $N_\theta$ such that for all $N\ge N_\theta$, with probability at least $1-e^{-c_\theta N}$, every product $P$ of $L$ factors, in either model, with $\lVert G-P\rVert_F\le\theta\lVert G\rVert_F$ has $Lr\ge c_\theta N$; the event does not depend on $r$. Consequently, the expected number of factors needed for relative error $\theta$ is $\Theta_\theta(N/r)$.
\item \emph{Singular values do not determine the cost.} The identity $I_N$ and the cyclic shift $Q$, defined by $Qe_j=e_{j+1}$ for $j<N$ and $Qe_N=e_1$, have identical singular values. $I_N$ is a single factor, while $Q$ needs exactly $\lceil(N-1)/r\rceil$ scalar factors and at least that many general-diagonal factors.
\end{enumerate}
\end{thmL}

\begin{proof}[Proof of Theorem~\ref{thm:D}(a)]
By Lemma~\ref{lem:rank_count}, every $L$-factor product lies in $\mathcal M_q=\{D+R:\ D\text{ diagonal},\ \operatorname{rank}(R)\le q\}$ with $q=Lr$. The matrices of rank $j$ form a smooth manifold of dimension $j(2N-j)$, which is increasing in $j\le N$. Hence $\mathcal M_q$ is a finite union of images, under the smooth map $(D,R)\mapsto D+R$, of manifolds of dimension at most $N+q(2N-q)$. If $N+q(2N-q)<N^2$, i.e., $(N-q)^2>N$, each such image is Lebesgue-null (cover the domain by countably many compact pieces on which the map is Lipschitz). So almost surely $W\notin\mathcal M_q$ for every $q$ with $N-q>\sqrt N$, which forces $Lr\ge N-\lfloor\sqrt N\rfloor$. The upper bound is Theorem~\ref{thm:A}(a) with $\gamma=1$ (or a single factor if $W=I$).
\end{proof}

\begin{proof}[Proof of Theorem~\ref{thm:D}(b)]
\emph{Step 1: setup.} Fix an integer $\nu\ge2$ with $1-1/\nu>\theta^2$ and $\zeta\in(0,1)$ with $(1-\zeta)(1-1/\nu)>\theta^2(1+\zeta)$. Let $m=\lfloor N/\nu\rfloor$, split $m\nu$ of the indices into $\nu$ groups of size $m$, and consider the $\nu(\nu-1)$ off-diagonal blocks $H_{st}$, $s\neq t$, where $H_{st}$ is the submatrix of $G$ with rows in group $s$ and columns in group $t$.

\emph{Step 2: diagonals do not help off the diagonal blocks.} For any $P=D+R$ with $D$ diagonal and $\operatorname{rank}(R)\le Lr$, the block of $P$ in position $(s,t)$, $s\neq t$, equals the corresponding block $R_{st}$ of $R$, which has rank at most $Lr$. By the Eckart--Young theorem,
\[
\lVert H_{st}-R_{st}\rVert_F^2\ \ge\ \sum_{i>Lr}\sigma_i(H_{st})^2\ \ge\ \lVert H_{st}\rVert_F^2-Lr\,\lVert H_{st}\rVert_2^2 .
\]
Summing over the disjoint blocks, $\lVert G-P\rVert_F^2\ge\sum_{s\neq t}\bigl(\lVert H_{st}\rVert_F^2-Lr\lVert H_{st}\rVert_2^2\bigr)$.

\emph{Step 3: concentration.} $\lVert H_{st}\rVert_F^2$ and $\lVert G\rVert_F^2$ are $\chi^2$ variables with $m^2$ and $N^2$ degrees of freedom, so by the bounds of \citet{laurent2000adaptive}, $\lVert H_{st}\rVert_F^2\ge(1-\zeta)m^2$ for all blocks and $\lVert G\rVert_F^2\le(1+\zeta)N^2$, except with probability $e^{-\Omega(N^2)}$. By the standard operator-norm bound for Gaussian matrices \citep[Theorem~4.4.5]{vershynin2018high}, $\lVert H_{st}\rVert_2\le C_0\sqrt m$ for all blocks, except with probability $2\nu^2e^{-m}$, where $C_0$ is an absolute constant.

\emph{Step 4: conclusion.} On the intersection of these events, $\lVert G-P\rVert_F\le\theta\lVert G\rVert_F$ implies
\[
\theta^2(1+\zeta)N^2\ \ge\ \nu(\nu-1)\bigl((1-\zeta)m^2-C_0^2\,Lr\,m\bigr).
\]
Since $\nu(\nu-1)m^2/N^2\to1-1/\nu$, the choice of $\zeta$ gives $\nu(\nu-1)(1-\zeta)m^2-\theta^2(1+\zeta)N^2\ge c'N^2$ for large $N$ and some $c'>0$, hence $Lr\ge c'N^2/(\nu(\nu-1)C_0^2m)\ge c_\theta N$. The event has probability at least $1-e^{-c_\theta N}$ after decreasing $c_\theta$ if necessary, and it does not depend on $r$. The expectation bound follows, with the matching upper bound $\lceil N/r\rceil$ from part~(a).
\end{proof}

\begin{proof}[Proof of Theorem~\ref{thm:D}(c)]
$Q$ is a permutation matrix, so all its singular values equal $1$, as do those of $I_N$. For any diagonal $D$, consider the submatrix of $Q-D$ formed by rows $2,\dots,N$ and columns $1,\dots,N-1$. Its diagonal consists of the entries $(j+1,j)$ of $Q$, which equal $1$, and the entries of $D$ occupy the positions $(i,i)$, which lie strictly above that diagonal. So the submatrix is upper triangular with unit diagonal, $\operatorname{rank}(Q-D)\ge N-1$, and $\delta(Q)\ge N-1$; Lemma~\ref{lem:rank_count} then gives the lower bound in both models. Since $Q\mathbf 1_N=\mathbf 1_N$, $\operatorname{rank}(Q-I)\le N-1$, so $\kappa(Q)=N-1$ and Theorem~\ref{thm:A} gives exactly $\lceil(N-1)/r\rceil$ scalar factors.
\end{proof}

\subsection{Proposition E: Conditioning}
\label{app:G_propE}

Let $R\in\mathbb R^{N\times N}$ have rank $k\ge2$, and let $W_\varepsilon=\varepsilon I+R$ with $\varepsilon>0$.

\begin{propL}[Conditioning of DLoR factorizations]\label{prop:E}
\begin{enumerate}[label=(\alph*), leftmargin=2em, itemsep=2pt, topsep=2pt]
\item[]
\item \emph{Scalar model, lower bound.} Let $r<k$ and $Lr+k<N$. Every representation $W_\varepsilon=M_L\cdots M_1$ with $M_l=\alpha_lI+R_l$, $\operatorname{rank}(R_l)\le r$, satisfies $\prod_l\alpha_l=\varepsilon$ and
\[
\prod_{j=1}^{k}\sigma_j(W_\varepsilon)\ \le\ \varepsilon^{\,k-r}\Bigl(\prod_{l=1}^{L}\lVert M_l\rVert_2\Bigr)^{r}.
\]
Hence $\max_l\lVert M_l\rVert_2\ge c\,\varepsilon^{-(k-r)/(rL)}$ for small $\varepsilon$, with $c>0$ depending only on $R$, $r$, and $L$. For $r=1$ this is $\varepsilon^{-(k-1)/L}$.
\item \emph{Scalar model, matching upper bound for $r=1$.} For every $L\ge k$ there are $L$-factor representations of $W_\varepsilon$ with $\max_l\lVert M_l\rVert_2=O(\varepsilon^{-(k-1)/L})$.
\item \emph{General-diagonal model, bounded factors.} (1) $R$ is a product of at most $2k-1$ factors $D+uv^\top$ of spectral norm at most $\max(1,\lVert R\rVert_2)$. (2) $W_\varepsilon$ is a product of at most $8k$ such factors of norm at most $\max(1,\varepsilon+\lVert R\rVert_2)$, uniformly in $\varepsilon$.
\item \emph{Schur construction, scalar model.} Let $W\in\mathbb R^{N\times N}$, $\gamma\neq0$, and $g=\dim\ker(W-\gamma I)$. There are factors $M_j$ with $\operatorname{rank}(M_j-I)\le r$ and
\[
W=\gamma\,M_q\cdots M_1,\qquad \lVert M_j\rVert_2\le1+\frac{\lVert W\rVert_2}{\lvert\gamma\rvert},
\]
computed from one real Schur decomposition without matrix inversion. Here $q=\lceil(N-g)/r\rceil$ if $W$ has real spectrum, and $q\le\lceil(N-g)/(r-1)\rceil$ in general for $r\ge2$. With $\gamma$ attaining $\kappa(W)$ we have $N-g=\kappa(W)$, so the count is optimal for real spectra; with $\gamma=\lVert W\rVert_2$ the factors have norm at most $2$.
\end{enumerate}
\end{propL}

\begin{proof}[Proof of Proposition~\ref{prop:E}(a)]
\emph{Step 1: the scalars multiply to $\varepsilon$.} Let $\beta=\prod_l\alpha_l$. Then $(\varepsilon-\beta)I=(W_\varepsilon-\beta I)-R$. By Lemma~\ref{lem:rank_count}, the right-hand side has rank at most $Lr+k<N$, so $\varepsilon=\beta$.

\emph{Step 2: each factor has only $r$ large singular values.} $M_l$ acts as $\alpha_lI$ on $\ker R_l$, which has codimension at most $r$; by the min--max characterization of singular values, $\sigma_j(M_l)\le\lvert\alpha_l\rvert$ for $j>r$. Hence $\prod_{j\le k}\sigma_j(M_l)\le\lVert M_l\rVert_2^{\,r}\lvert\alpha_l\rvert^{k-r}$.

\emph{Step 3: exterior power.} The operator norm of the $k$-th exterior power $\Lambda^kA$ is the product of the top $k$ singular values of $A$, and exterior powers are multiplicative. Therefore
\[
\prod_{j\le k}\sigma_j(W_\varepsilon)=\lVert\Lambda^kW_\varepsilon\rVert\le\prod_l\lVert\Lambda^kM_l\rVert\le\prod_l\lVert M_l\rVert_2^{\,r}\lvert\alpha_l\rvert^{k-r}=\varepsilon^{k-r}\Bigl(\prod_l\lVert M_l\rVert_2\Bigr)^{r}.
\]

\emph{Step 4.} By Weyl's inequality, the left-hand side is at least $\prod_{j\le k}(\sigma_j(R)-\varepsilon)$, which tends to the positive constant $\prod_{j\le k}\sigma_j(R)$ as $\varepsilon\to0$.
\end{proof}

\begin{proof}[Proof of Proposition~\ref{prop:E}(b)]
Write $R=UV^\top$ with $U,V\in\mathbb R^{N\times k}$, and bring $C=V^\top U$ to real Schur form by a change of basis as in Example~\ref{ex:rank2}. Then $C$ is block upper triangular with diagonal blocks of size $1$ or $2$; let $U_c,V_c$ be the columns of $U,V$ belonging to block $c=1,\dots,p$. Since $V_c^\top U_{c'}=0$ for $c>c'$, the product with the largest index on the left,
\[
\Bigl(I+\tfrac{1}{\varepsilon}U_pV_p^\top\Bigr)\cdots\Bigl(I+\tfrac{1}{\varepsilon}U_1V_1^\top\Bigr),
\]
has no cross terms and equals $I+R/\varepsilon$. A $1\times1$ block is already a rank-one update; a $2\times2$ block splits into two rank-one factors exactly as in the complex case of Example~\ref{ex:rank2}, after dividing by $\varepsilon$. This gives $k$ factors of the form $I+ab^\top$ with norms $O(1/\varepsilon)$. For $L=k$, multiply each factor by $\varepsilon^{1/k}$: the scalars multiply to $\varepsilon$, and the norms are $O(\varepsilon^{-(k-1)/k})$. For $L>k$, append $L-k$ identity factors, multiply the $k$ original factors by $\varepsilon^{1-(k-1)/L}$ and the identity factors by $\varepsilon^{-(k-1)/L}$. The scalars multiply to $\varepsilon^{\,k-(k-1)k/L-(k-1)(L-k)/L}=\varepsilon$, and every factor has norm $O(\varepsilon^{-(k-1)/L})$.
\end{proof}

\begin{proof}[Proof of Proposition~\ref{prop:E}(c)]
\emph{Part (1).} Take a thin SVD $R=\sum_{j=1}^k\sigma_ju_jv_j^\top$. Build an orthogonal matrix $Q_U$ with $Q_Ue_j=u_j$ for $j<k$ as a product of at most $k-1$ Householder reflections: at step $j$, the current image $c$ of $e_j$ and the target $u_j$ are both unit vectors orthogonal to $u_1,\dots,u_{j-1}$, so the reflection across the hyperplane orthogonal to $c-u_j$ maps $c$ to $u_j$ and fixes $u_1,\dots,u_{j-1}$ (no reflection is needed if $c=u_j$). Build $Q_V$ in the same way from $v_1,\dots,v_{k-1}$. Then
\[
R=Q_U\Bigl[\operatorname{diag}(\sigma_1,\dots,\sigma_{k-1},0,\dots,0)+\sigma_k(Q_U^\top u_k)(Q_V^\top v_k)^\top\Bigr]Q_V^\top .
\]
The middle factor equals $Q_U^\top RQ_V$, so it has norm $\lVert R\rVert_2$, and it is diagonal plus rank one. Each reflection $I-2ww^\top$ is a general-diagonal rank-one factor of norm $1$. The total is at most $(k-1)+1+(k-1)=2k-1$ factors.

\emph{Part (2).} Let $\mathcal H=\operatorname{Im}(R)+\operatorname{Im}(R^\top)$, of dimension $d\le2k$. Then $R(\mathcal H)\subseteq\operatorname{Im}(R)\subseteq\mathcal H$, and $R$ vanishes on $\mathcal H^\perp\subseteq\operatorname{Im}(R^\top)^\perp=\ker R$. Choose an orthogonal $Q$, a product of at most $d$ reflections, mapping the span of the first $d$ coordinate vectors onto $\mathcal H$. Then $Q^\top W_\varepsilon Q=(\varepsilon I_d+R_{\mathcal H})\oplus\varepsilon I_{N-d}$, where $R_{\mathcal H}$ is the matrix of $R|_{\mathcal H}$. Take an SVD $U_{\mathcal H}\Sigma V_{\mathcal H}^\top$ of the $d\times d$ block; each orthogonal factor is a product of at most $d$ reflections acting on the first $d$ coordinates. So
\[
W_\varepsilon=Q\,(U_{\mathcal H}\oplus I)\,(\Sigma\oplus\varepsilon I)\,(V_{\mathcal H}\oplus I)^\top\,Q^\top,
\]
which uses at most $4d\le8k$ reflections. The diagonal middle factor is absorbed into an adjacent reflection, since $D(I-2ww^\top)=D-2(Dw)w^\top$ is diagonal plus rank one with norm at most $\lVert D\rVert_2\le\varepsilon+\lVert R\rVert_2$; if there is no reflection, $W_\varepsilon$ is itself diagonal.
\end{proof}

\begin{proof}[Proof of Proposition~\ref{prop:E}(d)]
\emph{Step 1: Schur form adapted to $\gamma$.} Let the first $g$ columns of an orthogonal matrix form an orthonormal basis of $\ker(W-\gamma I)$, and take a real Schur form of the compression of $W$ to the orthogonal complement. This gives an orthogonal $Q$ with
\[
T=Q^\top WQ=\begin{pmatrix}\gamma I_g & T_{12}\\ 0 & T_{22}\end{pmatrix},
\]
where $T_{22}$ is quasi-upper-triangular with diagonal blocks of size $1$ or $2$. A sorted Schur form alone is not sufficient when $\gamma$ is a defective eigenvalue.

\emph{Step 2: active columns.} Let $C=T/\gamma-I$. Its first $g$ columns vanish. No other column vanishes: $Ce_i=0$ with $i>g$ would make $Qe_i$ an eigenvector of $W$ for $\gamma$ orthogonal to $\ker(W-\gamma I)$. So there are exactly $N-g$ active (nonzero) columns, namely the last $N-g$.

\emph{Step 3: factors.} Partition the active columns, in increasing order, into consecutive blocks $B_1,\dots,B_q$ of at most $r$ columns that do not split a $2\times2$ diagonal block of $T_{22}$. With $D_B$ the $0/1$ diagonal selector of an index set $B$, set $\widetilde M_j=I+CD_{B_j}$ and $M_j=Q\widetilde M_jQ^\top$. Clearly $\operatorname{rank}(M_j-I)\le\lvert B_j\rvert\le r$.

\emph{Step 4: telescoping.} Let $P_j=I+C(D_{B_1}+\dots+D_{B_j})$. Then $\widetilde M_jP_{j-1}=P_j+CD_{B_j}C(D_{B_1}+\dots+D_{B_{j-1}})$. The product $D_{B_j}C(D_{B_1}+\dots+D_{B_{j-1}})$ involves only entries of $C$ with row index in $B_j$ and column index in an earlier block, i.e., strictly below the diagonal; these vanish because $C$ is quasi-upper-triangular and no $2\times2$ block is split. Hence $\widetilde M_jP_{j-1}=P_j$, and $P_q=I+C=T/\gamma$ because the inactive columns of $C$ are zero. So $T=\gamma\widetilde M_q\cdots\widetilde M_1$ and $W=QTQ^\top=\gamma M_q\cdots M_1$.

\emph{Step 5: norms.} $\widetilde M_j=(I-D_{B_j})+(T/\gamma)D_{B_j}$, so $\lVert\widetilde M_jx\rVert_2\le\lVert(I-D_{B_j})x\rVert_2+(\lVert T\rVert_2/\lvert\gamma\rvert)\lVert D_{B_j}x\rVert_2\le(1+\lVert W\rVert_2/\lvert\gamma\rvert)\lVert x\rVert_2$.

\emph{Step 6: count.} With real spectrum, all diagonal blocks are $1\times1$ and $q=\lceil(N-g)/r\rceil$. Otherwise, greedy packing closes each block with at least $r-1$ columns, so $q\le\lceil(N-g)/(r-1)\rceil$ for $r\ge2$.
\end{proof}

\subsection{Proposition F: Separation from Bottlenecked Networks}
\label{app:G_propF}

\begin{propL}[Separation from bottlenecked networks]\label{prop:F}
Let $W=\gamma I+E\in\mathbb R^{N\times N}$ be invertible with $\operatorname{rank}(E)=k$, let $b\in\mathbb R^N$, and let $\rho$ act coordinatewise with $\rho'\ge m>0$ on an interval containing every coordinate of $Wx+b$ for $x\in[0,1]^N$. Let $F(x)=\rho(Wx+b)$.
\begin{enumerate}[label=(\alph*), leftmargin=2em, itemsep=2pt, topsep=2pt]
\item \emph{DLoR realizes $F$ cheaply.} Let $q=\max\{1,\lceil k/r\rceil\}$. A DLoR network of width $N$ (width $2N$ inside paired-branch stages) with $q$ stages and at most $2Nk+q(N+1)$ parameters realizes $F$ on $[0,1]^N$: exactly when $\rho$ is ReLU or satisfies $\rho(t)-\rho(-t)=at$ with $a\neq0$, and to arbitrary uniform accuracy under Assumption~\ref{ass:main}.
\item \emph{Topological lower bound.} Let $G:[0,1]^N\to\mathbb R^N$ have the form $G=h\circ g$, where $g:[0,1]^N\to\mathbb R^w$ is continuous, $w<N$, and $h$ is arbitrary. Then
\[
\sup_{x\in[0,1]^N}\lVert G(x)-F(x)\rVert_2\ \ge\ \tfrac12\,m\,\sigma_{\min}(W).
\]
\item \emph{Consequence for feedforward networks.} Let $G$ be a feedforward network without skip connections, with continuous activations, hidden widths $w_1,\dots,w_H$, weight matrices $A_1\in\mathbb R^{w_1\times N},A_2,\dots,A_H$, and readout $A_{H+1}\in\mathbb R^{N\times w_H}$. If some $w_i<N$ or some $A_i$ has rank less than $N$, then $G$ has uniform error at least $\frac12m\,\sigma_{\min}(W)$ on $[0,1]^N$. In particular, every pure low-rank network (all weight matrices of rank less than $N$) has this error floor at every depth and width; and a network of this kind with smaller error has all hidden widths at least $N$ and all weight matrices of rank $N$, so with dense layers it has at least $(H+1)N^2$ weights, whereas the DLoR network of (a) has at most $2Nk+q(N+1)$ parameters.
\end{enumerate}
\end{propL}

\begin{proof}[Proof of Proposition~\ref{prop:F}(a)]
First, $W$ is a product of $q$ scalar DLoR factors whose update ranks $r_j\le r$ satisfy $\sum_jr_j\le k$. If $k=0$, then $W=\gamma I$ is a single factor with no update. If $\gamma\neq0$ and $k\ge1$, Theorem~\ref{thm:A}(a) gives $W=M_q\cdots M_1$ with $q=\lceil k/r\rceil$ and $\sum_jr_j=k$. If $\gamma=0$, then $W=E$ is invertible, so $k=N$, and Theorem~\ref{thm:A}(a) applied with any $\gamma'\neq0$ that is not an eigenvalue of $W$ gives $\lceil N/r\rceil=q$ factors with $\sum_jr_j=N=k$.

Second, apply Theorem~\ref{thm:B}(i) to these factors. For ReLU, bias shifts keep every hidden pre-activation nonnegative on $[0,1]^N$, so the hidden stages compute the partial products exactly, and the last stage outputs $\rho(M_qz_{q-1}+b_q)=\rho(Wx+b)=F(x)$. If $\rho(t)-\rho(-t)=at$, the paired-branch construction (Proposition~\ref{prop:paired_exact}) computes each partial product exactly, and the last stage again applies $\rho$ to $Wx+b$. Under Assumption~\ref{ass:main}, the finite-$h$ construction gives uniform approximation to any accuracy.

Third, count parameters. Stage $j$ stores $U_j,V_j\in\mathbb R^{N\times r_j}$, one scalar $\alpha_j$, and a bias $b_j\in\mathbb R^N$, i.e., $2Nr_j+N+1$ numbers; paired branches reuse the same parameters. Summing over the $q$ stages gives $2N\sum_jr_j+q(N+1)\le2Nk+q(N+1)$.
\end{proof}

\begin{proof}[Proof of Proposition~\ref{prop:F}(b)]
Let $c=\frac12\mathbf 1_N$ and let $S$ be the sphere of radius $\frac12$ centered at $c$, which lies in $[0,1]^N$. Since $w\le N-1$, the Borsuk--Ulam theorem \citep{matousek2003using} applied to $g|_S$ (composed with an inclusion $\mathbb R^w\subseteq\mathbb R^{N-1}$) gives $x\in S$ with $g(x)=g(x')$ for the antipodal point $x'=2c-x$. Hence $G(x)=G(x')$. By the mean value theorem in each coordinate, $\lVert F(x)-F(x')\rVert_2\ge m\lVert W(x-x')\rVert_2\ge m\,\sigma_{\min}(W)\lVert x-x'\rVert_2=m\,\sigma_{\min}(W)$, since $\lVert x-x'\rVert_2=1$. By the triangle inequality, $G$ errs by at least $\frac12m\,\sigma_{\min}(W)$ at $x$ or at $x'$.
\end{proof}

\begin{proof}[Proof of Proposition~\ref{prop:F}(c)]
Write $z_0(x)=x$ and let $z_i(x)\in\mathbb R^{w_i}$ be the activations of hidden layer $i$; each $z_i$ is continuous because the activations are. Without skip connections, the output depends on $x$ only through $z_i(x)$, so $G=h_i\circ z_i$ for some $h_i$. If $w_i<N$, part (b) applies with $g=z_i$. If $\operatorname{rank}(A_i)=s<N$, write $A_i=B_iC_i$ with $C_i\in\mathbb R^{s\times w_{i-1}}$ (where $w_0=N$) and $B_i$ of full column rank. The output depends on $x$ only through $A_iz_{i-1}(x)$, hence only through $g(x)=C_iz_{i-1}(x)\in\mathbb R^s$, which is continuous; part (b) applies with $w=s$. This covers the readout $A_{H+1}$ as well. For the counting statement, a dense matrix of rank $N$ has at least $N$ rows and $N$ columns, hence at least $N^2$ entries, and there are $H+1$ weight matrices.
\end{proof}

\paragraph{Scope.} Part (c) concerns dense or low-rank layers. It does not claim that every structured parameterization of a rank-$N$ matrix needs $N^2$ parameters; DLoR itself is the counterexample. Residual networks with low-rank branches, $x\mapsto x+UV^\top x$, also achieve $O(Nk)$ parameters on these targets; they have skip connections, so (c) does not apply to them. The essential ingredient is an identity pathway, inside the activation (DLoR) or around it (residual).

\subsection{Proposition G: Adaptation of a Frozen Matrix}
\label{app:G_propG}

Let $W_0,\Delta\in\mathbb R^{m\times n}$ and $s=\operatorname{rank}(\Delta)$. \emph{Identity-normalized} factors are $I+R_j$ with $\operatorname{rank}(R_j)\le r$.

\begin{propL}[Adaptation of a frozen matrix]\label{prop:G}
\begin{enumerate}[label=(\alph*), leftmargin=2em, itemsep=2pt, topsep=2pt]
\item[]
\item \emph{Input side.} $W_0+\Delta=W_0(I+R_1)\cdots(I+R_L)$ for some $L$ if and only if $\operatorname{Im}(\Delta)\subseteq\operatorname{Im}(W_0)$. In that case the minimal $L$ is $\lceil s/r\rceil$ ($L=0$ if $s=0$).
\item \emph{Output side.} $W_0+\Delta=(I+R_L)\cdots(I+R_1)W_0$ for some $L$ if and only if $\ker W_0\subseteq\ker\Delta$, and then the minimal $L$ is again $\lceil s/r\rceil$.
\item \emph{LoRA equivalence.} For square invertible $W_0$ both conditions hold, and with $L$ fixed the representable updates are exactly $\{\Delta:\ \operatorname{rank}(\Delta)\le Lr\}$, the class of a LoRA update of rank $Lr$ \citep{hu2022lora}, at the same $2N$ parameters per unit of rank.
\item \emph{Free scalar.} Let $W_0$ be square and invertible and $\operatorname{rank}(W^\star)>r$. Then $W^\star=(M_L\cdots M_1)W_0$ with scalar DLoR factors $M_l$ of update rank at most $r$ if and only if
\[
L\ \ge\ \max\Bigl\{1,\ \Bigl\lceil\frac{1}{r}\min_{\gamma\neq0}\operatorname{rank}(W^\star-\gamma W_0)\Bigr\rceil\Bigr\},
\]
and the same holds for $W^\star=W_0(M_L\cdots M_1)$. Taking $\gamma=1$, the right-hand side never exceeds $\max\{1,\lceil\operatorname{rank}(W^\star-W_0)/r\rceil\}$. For $W^\star=\lambda W_0+\Delta W$ with $\lambda\notin\{0,1\}$ and $\Delta W\neq0$, it is at most $\lceil\operatorname{rank}(\Delta W)/r\rceil$, while $\operatorname{rank}(W^\star-W_0)\ge N-\operatorname{rank}(\Delta W)$.
\item \emph{General diagonals.} If $W^\star=PW_0$ with $P\in\mathcal P^{\mathrm{diag}}_r(L)$, then $L\ge\lceil\min_{D\ \mathrm{diagonal}}\operatorname{rank}(W^\star-DW_0)/r\rceil$. This covers a diagonal rescaling in the style of IA$^3$ \citep{liu2022fewshot} combined with a low-rank update.
\end{enumerate}
\end{propL}

\begin{proof}
\emph{Part (a).} Necessity: $\Delta=W_0(P-I)$ with $P$ the product of the factors, so $\operatorname{Im}(\Delta)\subseteq\operatorname{Im}(W_0)$. Sufficiency: let $S:\operatorname{Im}(W_0)\to\mathbb R^n$ be a linear right inverse of $W_0$ on its image, and let $C=S\Delta$. Then $W_0C=\Delta$, and $\operatorname{rank}(C)=s$ because $S$ is injective. Theorem~\ref{thm:A}(a) with $\gamma=1$ writes $I+C$ as a product of $\lceil s/r\rceil$ factors $I+R_j$. Minimality: $s=\operatorname{rank}(W_0(P-I))\le\operatorname{rank}(P-I)\le Lr$ by Lemma~\ref{lem:rank_count}.

\emph{Part (b).} Necessity: $\Delta=(P-I)W_0$ vanishes on $\ker W_0$. Sufficiency: define $C(W_0x)=\Delta x$ on $\operatorname{Im}(W_0)$, which is well defined because $\ker W_0\subseteq\ker\Delta$, and extend $C$ by zero on a complement. Then $\operatorname{Im}(C)=\operatorname{Im}(\Delta)$, so $\operatorname{rank}(C)=s$ and $(I+C)W_0=W_0+\Delta$. Minimality follows as in (a).

\emph{Part (c).} Immediate from (a) and (b), padding with identity factors.

\emph{Part (d).} $W^\star=PW_0$ if and only if $P=W^\star W_0^{-1}$, and $\operatorname{rank}(P)=\operatorname{rank}(W^\star)>r$. Apply Theorem~\ref{thm:A}(b) to $P$ and use $\operatorname{rank}(W^\star W_0^{-1}-\gamma I)=\operatorname{rank}((W^\star-\gamma W_0)W_0^{-1})=\operatorname{rank}(W^\star-\gamma W_0)$. The right-multiplied case uses $W_0^{-1}W^\star-\gamma I=W_0^{-1}(W^\star-\gamma W_0)$. For the example take $\gamma=\lambda$. Finally, $W^\star-W_0=(\lambda-1)W_0+\Delta W$ with $(\lambda-1)W_0$ invertible, so its rank is at least $N-\operatorname{rank}(\Delta W)$.

\emph{Part (e).} $P=D+R$ with $D$ diagonal and $\operatorname{rank}(R)\le Lr$ by Lemma~\ref{lem:rank_count}, so $W^\star-DW_0=RW_0$ has rank at most $Lr$.
\end{proof}

\subsection{Supplementary Lemmas}
\label{app:G_supp}

\begin{lemma}[Finite interpolation with a rank-one first layer]\label{lem:rank1_interp_nonpoly}
Let $x_1,\dots,x_M\in\mathbb R^d$ be distinct and let $\sigma$ be continuous and not a polynomial. Then a single hidden layer of width $M$ whose first weight matrix has rank at most one, followed by an affine readout, interpolates any labels $y\in\mathbb R^M$.
\end{lemma}
\begin{proof}
Choose $v\in\mathbb R^d$ such that the projections $t_i=v^\top x_i$ are pairwise distinct; this excludes finitely many hyperplanes $(x_i-x_j)^\top v=0$. The vectors $(\sigma(at_i+c))_{i\le M}$, $a,c\in\mathbb R$, span $\mathbb R^M$. Otherwise, a nonzero $u\in\mathbb R^M$ is orthogonal to all of them, hence $\sum_iu_if(t_i)=0$ for every $f$ in the span of $\{t\mapsto\sigma(at+c)\}$, which is dense in $C([-\tau,\tau])$ for every $\tau>0$ when $\sigma$ is not a polynomial \citep{leshno1993multilayer,pinkus1999approximation}. Approximating a continuous $f$ with $f(t_i)=u_i$ gives $\sum_iu_i^2=0$, a contradiction. Take $M$ linearly independent such vectors as hidden units $\sigma(a_jv^\top x+c_j)$. The first weight matrix has rows $a_jv^\top$, so its rank is at most one, and the feature matrix is invertible, so a linear readout fits $y$.
\end{proof}
Lemma~\ref{lem:rank1_interp_nonpoly} relaxes the activation assumption of Lemma~\ref{thm:gen_act_interp} by using free slopes $a_j$; the non-mean-periodic condition corresponds to the special case in which only translates of $\sigma$ are used.

\begin{lemma}[Quantitative orthogonal blindness on $\lVert x\rVert_2^2$]\label{lem:blindness_quant}
If a network receives its input only through $x\mapsto Ax+b$ with $\operatorname{rank}(A)<d$, then on the unit ball of $\mathbb R^d$ its uniform error for $f(x)=\lVert x\rVert_2^2$ is at least $1/2$. Moreover, on $[-1,1]^d$ with the uniform distribution, a predictor that depends on $x$ only through one coordinate has normalized mean squared error at least $(d-1)/d$; in particular, $1/2$ for $x_1^2+x_2^2$ and $7/8$ for $\lVert x\rVert_2^2$ in $d=8$.
\end{lemma}
\begin{proof}
Take a unit vector $v\in\ker A$. The network takes the same value at $0$ and at $v$, while $f(0)=0$ and $f(v)=1$. For the second claim, given $x_1$ the residual of the best predictor is $\sum_{i\ge2}(x_i^2-\mathbb E[x_i^2])$, whose variance is $(d-1)\operatorname{Var}(x_1^2)=(d-1)\cdot\frac4{45}$, while $\operatorname{Var}(f)=d\cdot\frac4{45}$.
\end{proof}
In numerical checks over ridge directions, coordinate directions were optimal for both targets, so these values are the floors for arbitrary ridge functions $g(v^\top x)$ in these two cases.

\begin{lemma}[Explicit width-$4$ rank-one DLoR network for $x_1^2+x_2^2$]\label{lem:x2y2_net}
On $[-1,1]^2$, a ReLU network of width $4$ with $4n+2$ stages, each with a weight matrix of the form identity plus a matrix of rank at most one, approximates $x_1^2+x_2^2$ uniformly within $2/n^2$.
\end{lemma}
\begin{proof}
Let $z_i=x_i+1\in[0,2]$ and $h=2/n$. The piecewise-linear interpolant of $z^2$ on the mesh $0,h,\dots,2$ is
\[
p_n(z)=h\,\mathrm{ReLU}(z)+2h\sum_{j=1}^{n-1}\mathrm{ReLU}(z-jh),
\]
because its slope on $[jh,(j+1)h]$ is $(2j+1)h$. On each mesh interval, $p_n(z)-z^2=(z-jh)((j+1)h-z)\le h^2/4$. The network state is $(z_1,z_2,A,T)\ge0$. One identity stage with bias $(1,1,0,0)$ produces $(z_1,z_2,0,0)$ from the zero-padded input. For each $i\in\{1,2\}$ and each hinge $j\in\{0,\dots,n-1\}$, one stage with weight $I+e_4(e_i-e_4)^\top$ and bias $-jh\,e_4$ overwrites $T$ by $\mathrm{ReLU}(z_i-jh)$, and one stage with weight $I+\gamma_je_3e_4^\top$ ($\gamma_0=h$ and $\gamma_j=2h$ for $j\ge1$) adds $\gamma_jT$ to $A$; ReLU leaves the nonnegative coordinates unchanged. A final affine stage with weight $I+e_3(-2e_1-2e_2)^\top$ and bias $2e_3$ makes the third coordinate equal to $p_n(z_1)+p_n(z_2)-2z_1-2z_2+2$, which differs from $x_1^2+x_2^2=(z_1-1)^2+(z_2-1)^2$ by at most $h^2/2=2/n^2$. The count is $1+4n+1$ stages.
\end{proof}

\begin{lemma}[Logit perturbations]\label{lem:logit_pert}
If teacher logits $z$ and replica logits $\hat z$ satisfy $\lVert z-\hat z\rVert_\infty\le\eta$, then every prediction with teacher margin greater than $2\eta$ is preserved, and the cross-entropy changes by at most $2\eta$ per example.
\end{lemma}
\begin{proof}
For the teacher class $y^\star$ and any $i\neq y^\star$, $\hat z_{y^\star}-\hat z_i\ge z_{y^\star}-z_i-2\eta$. The function $z\mapsto\log\sum_ie^{z_i}$ is $1$-Lipschitz in the sup norm, and the cross-entropy with label $y$ is this function minus $z_y$.
\end{proof}

\subsection{Relation to Prior Work}\label{app:prior_work}

This subsection expands the Related Work paragraph of Section~\ref{sec:related_work}.

\paragraph{Structured factorizations.} \citet{liu2024lu} prove universal approximation by
factoring each dense weight into triangular, Toeplitz or Hankel factors and linearizing the
activation between consecutive factors, and they observe that the argument applies to any
factorization into structures preserved under multiplication by a scalar. DLoR components
are such a family, so Theorem~\ref{thm:uat_multiplicative} can also be obtained by combining
that observation with the factorization of Theorem~\ref{thm:A}. The register construction of
\citet{kidger2020universal} is related in a different way: its hidden weights are a scaled
identity plus a correction of rank at most two, and each such matrix is a product of at most
two DLoR components by Theorem~\ref{thm:A}. What is specific to DLoR, and what
Appendix~\ref{app:new_theory} contributes, is the exact cost of the factorization.

\paragraph{Factor counts in the general linear group.} At the matrix level the count of
Theorem~\ref{thm:A} is classical in the case of invertible matrices and rank-one updates. A
factor $\alpha I_N+uv^T$ with $\alpha\neq0$ is a scalar multiple of a reflection, and the
reflection length of $g\in\mathrm{GL}(\mathbb R^N)$ equals $\operatorname{rank}(g-I_N)$
\citep{dieudonne1955generateurs}; this is the absolute order studied by
\citet{huang2017absolute}, and factorizations of singular matrices into idempotents go back
to \citet{erdos1967products} and \citet{dawlings1980semigroups}. We do not claim this count
as new. Theorem~\ref{thm:A} optimizes it over the free scalar $\gamma$, which is what turns
$\operatorname{rank}(W-I_N)$ into $\kappa(W)$, extends it to arbitrary update rank $r$ and to
singular $W$, and Theorem~\ref{thm:B} carries it to the network level, where the activation
and the width constraint make the statement genuinely different.

\paragraph{Architectures with an identity path.} Closest in spirit, \citet{lin2018resnet}
prove that residual blocks with a single hidden neuron, that is an identity path plus a
rank-one update, are universal at width $d$ for ReLU. DLoR differs in placing the identity
\emph{inside} the activation, which is what makes Section~\ref{sec:limits} necessary, and in
covering every activation of Assumption~\ref{ass:main} rather than ReLU alone; it also admits
the exact finite constructions of Theorem~\ref{thm:B} and Proposition~\ref{prop:E}.

\paragraph{Width bottlenecks.} The orthogonal-blindness obstruction of
Section~\ref{sec:limits} is a form of the bottleneck phenomenon studied in the minimum-width
literature \citep{hanin2017approximating, johnson2019deep, park2021minimum}, where a hidden
layer narrower than the input dimension destroys universality. Proposition~\ref{prop:F}
sharpens that qualitative statement into an explicit error floor on an identity-centered
layer, and pairs it with the $O(Nk)$ upper bound that DLoR achieves on the same target.

\section{Additional Experiments}
\label{app:new_experiments}

\subsection{Matrix-Level Cost}
\label{app:matrix_level}

This experiment evaluates the best-approximation formula of Theorem~\ref{thm:C} on $N\times N$ matrices with $N=256$. For a matrix $W$ and a relative tolerance $\theta\in\{0.3,0.1,0.05\}$ we report the smallest total update rank $s=Lr$ such that the relative Frobenius error is at most $\theta$: (i) $s_{\mathrm{SVD}}=\min\{s:\tau_s(W)\le\theta\lVert W\rVert_F\}$ for truncated SVD; (ii) $s_{\mathrm{DLoR}}=\min\{s:\min_\gamma\tau_s(W-\gamma I)\le\theta\lVert W\rVert_F\}$ for scalar DLoR products, where the minimum over $\gamma$ is computed on a grid that contains $0$, $\operatorname{tr}(W)/N$, and the real eigenvalues of $W$, followed by a golden-section refinement; and (iii) $s_{\mathrm{diag}}$, the analogous quantity for the general-diagonal model, where the diagonal baseline is computed by alternating minimization between $D$ and the truncated SVD of $W-D$ and therefore yields an upper bound on the optimal value. In all three cases the parameter count is approximately $2Ns$.

We use one draw from each of the following classes, where $G$ has i.i.d.\ $\mathcal N(0,1/N)$ entries, $A\in\mathbb R^{N\times k}$ and $B\in\mathbb R^{k\times N}$ have i.i.d.\ standard Gaussian entries, and $Q_1,Q_2$ are Haar-distributed orthogonal matrices: Gaussian ($G$); exact rank $16$ ($AB/N$ with $k=16$); rank $16$ plus noise ($X+0.05\,\lVert X\rVert_FN^{-1/2}G$ with $X=AB/N$); polynomial decay ($Q_1\operatorname{diag}(1/i)Q_2^\top$); exponential decay ($Q_1\operatorname{diag}(0.95^{\,i-1})Q_2^\top$); identity plus rank $16$ plus noise ($I+AB/\sqrt N+0.05\,G$); residual ($I+0.3\,G$); and a fine-tuning ratio $W^\star W_0^{-1}$, where $W_0=Q_1\operatorname{diag}(d)Q_2^\top$ with $d_i$ uniform on $[0.5,2]$ and $W^\star=0.9\,W_0+\Delta W$ with $\operatorname{rank}(\Delta W)=8$ and $\lVert\Delta W\rVert_F=0.2\lVert W_0\rVert_F$.

\begin{table}[h]
\centering\color{revcolor}\small
\caption{Total update rank $s=Lr$ required for relative Frobenius error $\theta$ at $N=256$: truncated SVD, scalar DLoR (Theorem~\ref{thm:C}), and the general-diagonal model (``diag''; upper bound from alternating minimization). $\gamma^\star$ is the optimal scalar shift at $\theta=0.1$; $s=0$ means that a multiple of the identity already meets the tolerance.}
\label{tab:matrix_level}
\setlength{\tabcolsep}{3.5pt}
\begin{tabular}{l c|ccc|ccc|ccc}
 & & \multicolumn{3}{c|}{$\theta=0.3$} & \multicolumn{3}{c|}{$\theta=0.1$} & \multicolumn{3}{c}{$\theta=0.05$}\\
Matrix class & $\gamma^\star$ & SVD & DLoR & diag & SVD & DLoR & diag & SVD & DLoR & diag\\
\hline
Gaussian & $-0.03$ & 135 & 135 & 133 & 197 & 197 & 195 & 219 & 219 & 217\\
Exact rank $16$ & $0.00$ & 14 & 14 & 14 & 16 & 16 & 16 & 16 & 16 & 16\\
Rank $16$ + noise & $0.00$ & 14 & 14 & 14 & 16 & 16 & 16 & 16 & 16 & 16\\
Polynomial decay & $0.00$ & 7 & 7 & 7 & 49 & 49 & 49 & 125 & 125 & 123\\
Exponential decay & $0.00$ & 24 & 24 & 24 & 45 & 45 & 45 & 59 & 59 & 59\\
\hline
$I$ + rank $16$ + noise & $1.00$ & 15 & 14 & 14 & 193 & 16 & 16 & 229 & 16 & 16\\
Residual $I+0.3\,G$ & $1.00$ & 208 & 0 & 0 & 249 & 122 & 121 & 255 & 171 & 169\\
Fine-tuning ratio $W^\star W_0^{-1}$ & $0.90$ & 226 & 0 & 0 & 248 & 7 & 7 & 252 & 8 & 8\\
\end{tabular}
\end{table}

Table~\ref{tab:matrix_level} shows the dichotomy predicted by Theorems~\ref{thm:C} and~\ref{thm:D}. Whenever the optimal shift is $\gamma^\star\approx0$ (Gaussian, exactly or approximately low-rank, and decaying spectra), DLoR and truncated SVD need the same total rank, and for Gaussian matrices this rank is a constant fraction of $N$. On identity-centered matrices DLoR needs far less rank: $16$ instead of $193$ for identity plus rank $16$ plus noise and $7$ instead of $248$ for the fine-tuning ratio at $\theta=0.1$. For the fine-tuning example, $\operatorname{rank}(W^\star-W_0)=256$, so an exact additive low-rank update requires full rank, whereas $\kappa(W^\star W_0^{-1})=8$ (Proposition~\ref{prop:G}(d)). Allowing a general diagonal changes the required rank only marginally in these examples.

\subsection{Training from Scratch at Matched Parameter Budgets}
\label{app:scratch}

This experiment compares pure low-rank, DLoR, and dense networks trained from scratch at matched parameter counts on three targets defined on $[-1,1]^d$: $x_1^2+x_2^2$ ($d=2$), $\lVert x\rVert_2^2$ ($d=8$), and $\sum_{i=1}^4x_{2i-1}x_{2i}$ ($d=8$). The first two are the targets for which Lemma~\ref{lem:blindness_quant} gives the floors $1/2$ and $7/8$.

\paragraph{Setup.} Inputs are sampled uniformly from $[-1,1]^d$, with $2048$ training and $4096$ test points; targets are standardized. The structured networks act on a state of width $N=8$ for $d=2$ and $N=16$ for $d=8$: the input is zero-padded to width $N$, followed by stages $z\mapsto\rho(Mz+b)$ with an affine last stage, and the output is the first coordinate. We compare: \emph{pure rank-one} stages $M=uv^\top$; \emph{rank-one DLoR} stages $M=\alpha I+uv^\top$ ($\alpha$ initialized at $1$); \emph{dense} stages with unconstrained $M\in\mathbb R^{N\times N}$; a \emph{dense MLP} with two hidden layers whose width is chosen to match the parameter count of the DLoR network; and a \emph{one-neuron ResNet} with blocks $z\mapsto z+v\,\rho(u^\top z+c)$ and a linear readout, in the spirit of \citet{lin2018resnet}. All models use $\rho=\tanh$ and are trained with Adam on the mean squared error for $3000$ full-batch steps with a cosine learning-rate schedule starting at $3\times10^{-3}$. We report the test mean squared error divided by the variance of the target (so that a constant predictor scores $1$), as the median over $5$ seeds together with the minimum and maximum.

\begin{table}[h]
\centering\color{revcolor}\footnotesize
\caption{Training from scratch at matched parameter counts: normalized test error (median over $5$ seeds, with [min, max]). All models use $\tanh$. Parameter counts are given for $d=2$ / $d=8$.}
\label{tab:scratch}
\setlength{\tabcolsep}{4pt}
\begin{tabular}{l c|c c c}
Architecture & Parameters & $x_1^2+x_2^2$ \ ($d=2$) & $\lVert x\rVert_2^2$ \ ($d=8$) & $\sum_{i}x_{2i-1}x_{2i}$ \ ($d=8$)\\
\hline
Pure rank-one, 8 stages & 192 / 384 & 0.50 {\scriptsize[0.50, 0.69]} & 0.90 {\scriptsize[0.87, 0.96]} & 0.83 {\scriptsize[0.82, 0.84]}\\
DLoR rank-one, 8 stages & 200 / 392 & 0.0008 {\scriptsize[0.0003, 0.0011]} & 0.0045 {\scriptsize[0.0042, 0.0052]} & 0.013 {\scriptsize[0.008, 0.014]}\\
Dense MLP, matched & 205 / 391 & 0.0002 {\scriptsize[0.0001, 0.0004]} & 0.0050 {\scriptsize[0.0047, 0.0058]} & 0.0028 {\scriptsize[0.0020, 0.0042]}\\
\hline
DLoR rank-one, 16 stages & 400 / 784 & 0.0006 {\scriptsize[0.0004, 0.0012]} & 0.0037 {\scriptsize[0.0016, 0.0062]} & 0.0064 {\scriptsize[0.0051, 0.041]}\\
Dense MLP, matched & 415 / 783 & 0.0001 {\scriptsize[0.0001, 0.0002]} & 0.0040 {\scriptsize[0.0036, 0.0045]} & 0.0018 {\scriptsize[0.0016, 0.0020]}\\
\hline
One-neuron ResNet, 8 blocks & 145 / 281 & 0.0007 {\scriptsize[0.0004, 0.0014]} & 0.16 {\scriptsize[0.07, 0.26]} & 0.080 {\scriptsize[0.077, 0.12]}\\
Dense, width $N$, 8 stages & 576 / 2176 & 0.0002 {\scriptsize[0.0001, 0.0004]} & 0.0067 {\scriptsize[0.0064, 0.0071]} & 0.0042 {\scriptsize[0.0034, 0.0047]}\\
\end{tabular}
\end{table}

\paragraph{Results.} Table~\ref{tab:scratch} shows that pure rank-one networks remain at the ridge-function floors implied by orthogonal blindness ($0.50$ and $0.90$, against the floors $1/2$ and $7/8$), at every seed. Adding the scalar identity term removes the floor on all three targets, and rank-one DLoR networks reach errors comparable to those of dense MLPs with the same number of parameters: slightly better on $\lVert x\rVert_2^2$ and worse on the bilinear target. These experiments support the expressivity mechanism; they do not indicate an advantage of DLoR over dense networks on generic targets, in line with Theorem~\ref{thm:D}.

\paragraph{Sensitivity to the centering of the activation.} The results above use a zero-centered activation. With plain Softplus, for which $\rho(0)=\log2\neq0$, and zero-initialized biases, the same rank-one DLoR networks fail to train on the $d=8$ targets: the median normalized error is $0.88$ on $\lVert x\rVert_2^2$ and $0.85$ on the bilinear target with $8$ stages, and $0.96$ on $\lVert x\rVert_2^2$ with $16$ stages ($3$ seeds), for learning rates between $10^{-3}$ and $10^{-2}$. With the centered activation $\mathrm{Softplus}-\log2$, the $16$-stage network reaches $0.0059$ on $\lVert x\rVert_2^2$ ($0.0047$ with $\tanh$; $3$ seeds). A plausible explanation is that, when $\rho(0)\neq0$, the identity path shifts the state by $\rho(0)$ at every stage unless the biases compensate, which is the role of the bias corrections in our constructions (Lemma~\ref{lem:kidger_4_1} and the proof of Theorem~\ref{thm:uat_multiplicative}). A systematic study of the optimization of DLoR networks is left for future work.

\endgroup

\end{document}